\documentclass[journal]{IEEEtran}
\usepackage{epsfig}
\usepackage{amsmath}
\usepackage{amssymb}

\usepackage{color}
\usepackage{bm}
\usepackage{algorithm}
\usepackage{algorithmic}
\usepackage{multirow}
\usepackage{subfigure}
\usepackage{amsthm}
\usepackage{booktabs}
\usepackage{textcomp}
\usepackage{mathrsfs}
\usepackage{makecell}
\usepackage{diagbox}

\usepackage{balance}

\newtheorem{thm}{Theorem}

\newtheorem{assumption}{Assumption}

\newtheorem{remark}{Remark}

\begin{document}
%
\title{Investigating Task-driven Latent Feasibility \\for Nonconvex Image Modeling}

\author{Risheng Liu,~\IEEEmembership{Member,~IEEE,}
	Pan Mu, Jian Chen, 
	Xin Fan,~\IEEEmembership{Senior Member,~IEEE,}
	and~Zhongxuan Luo
	\thanks{
		This work was supported by the National Natural Science Foundation of China (Nos. 61922019, 61672125, 61733002 and 61772105), LiaoNing Revitalization Talents Program (XLYC1807088) and the Fundamental Research Funds for the Central Universities.}
		 	\thanks{Risheng Liu is with the DUT-RU International School of Information Science $\&$ Engineering, Dalian University of Technology, Dalian 116024, China, and also with the Peng Cheng Laboratory, Shenzhen 518055, China. (Corresponding author, e-mail: rsliu@dlut.edu.cn).}
	\thanks{
		Pan Mu is with the School of Mathematical Sciences, Dalian University of Technology, Dalian 116024, China. (e-mail: muyifan11@mail.dlut.edu.cn)}
	\thanks{
		Jian Chen and Zhongxuan Luo are with the School of Software, Dalian University of Technology, Dalian 116024, China. (e-mail:875277780@mail.dlut.edu.cn, zxluo@dlut.edu.cn)}
	\thanks{
		Xin Fan is with the DUT-RU International School of Information Science $\&$ Engineering, Dalian University of Technology,  
		Dalian 116024, China. (e-mail:xin.fan@dlut.edu.cn)}
	\thanks{Risheng Liu, Xin Fan and Zhongxuan Luo are also with the Key Laboratory for Ubiquitous Network and Service Software of Liaoning Province, Dalian 116024, China.}
}

\markboth{Journal of \LaTeX\ Class Files,~Vol.~14, No.~8, August~2015}%
{Shell \MakeLowercase{\textit{et al.}}: Bare Demo of IEEEtran.cls for IEEE Journals}
%



\maketitle

\begin{abstract}
Properly modeling latent image distributions plays an important role in a variety of image-related vision problems. Most exiting approaches aim to formulate this problem as optimization models (e.g., Maximum A Posterior, MAP) with handcrafted priors. In recent years, different CNN modules are also considered as deep priors to regularize the image modeling process. However, these explicit regularization techniques require deep understandings on the problem and elaborately mathematical skills. In this work, we provide a new perspective, named Task-driven Latent Feasibility (TLF), to incorporate specific task information to narrow down the solution space for the optimization-based image modeling problem. Thanks to the flexibility of TLF, both designed and trained constraints can be embedded into the optimization process. By introducing control mechanisms based on the monotonicity and boundedness conditions, we can also strictly prove the convergence of our proposed inference process. We demonstrate that different types of image modeling problems, such as image deblurring and rain streaks removals, can all be appropriately addressed within our TLF framework. Extensive experiments also verify the theoretical results and show the advantages of our method against existing state-of-the-art approaches.
\end{abstract}


%
\IEEEpeerreviewmaketitle

\section{Introduction}

\begin{figure}[t]
	\centering
	\begin{tabular}{c}
		\includegraphics[height=0.210\textwidth]{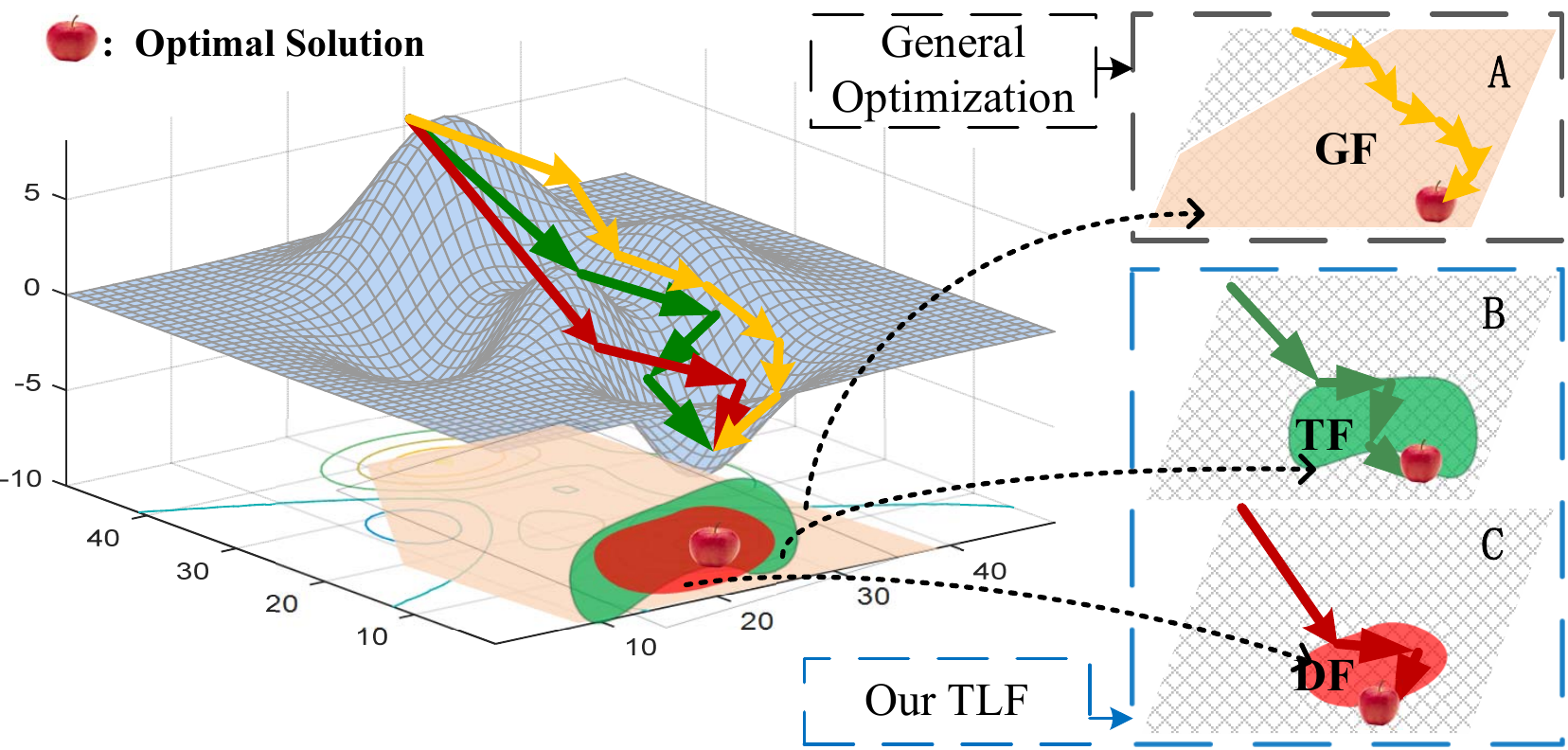}	
	\end{tabular}
	\caption{The illustration of our TLF paradigm. This illustrates and compares the general optimization strategy (i.e., ``A'', with generally feasible regions, abbreviated as GF) and our proposed ones (i.e., TLF in ``B'' and DTLF in ``C'' with task-driven feasible regions (TF for short) and data-driven extension feasible regions (DF for short), respectively). We will demonstrate that our task-driven feasibility paradigm can significantly improve the optimization process. Especially, the data-driven extension scheme can investigate task-driven feasibility form with training data, thus more suitable for real-world applications.}
	\label{fig:demo}
\end{figure}

\IEEEPARstart{M}{any} image-based low-level computer vision problems require to properly formulate the latent image distributions. 
The most widely used paradigm is to formulate this task within the Maximum A Posterior (MAP) estimation framework with conditional probability $p(\mathbf{b}|\mathbf{x})$ and prior distribution $p(\mathbf{x})$, i.e., $p(\mathbf{x}|\mathbf{b})\propto p(\mathbf{b}|\mathbf{x})p(\mathbf{x})$. 
Solving the above MAP problem is equivalent to dealing with the following minimization problem 
\begin{equation}\label{eq:model_orginal}
\begin{array}{c}\min\limits_{\mathbf{x}}F(\mathbf{x})=f(\mathbf{x})+\psi(\mathbf{x}),
\end{array}
\end{equation}
where $f(\mathbf{x})=-\log p(\mathbf{b}|\mathbf{x})$ and $\psi(\mathbf{x})=- \log p(\mathbf{x})$ capture the loss of data fitting and the regularization, respectively. In this work, we assume that the loss $f(\mathbf{x})$ is smooth, while the regularization $\psi(\mathbf{x})$ could be nonconvex and nonsmooth. Over the past decades, both numerical optimization techniques~\cite{beck2009fasta,liu2019surface,xu2012structure,cheng2019image,krishnan2009fast}, and deep learning approaches~\cite{liu2016learning,romano2017little,zhang2017learning,chen2017trainable,gregor2010learning,mu2018learning,liu2019knowledge,Fang2020soft} have been developed to address the MAP inference tasks.

In view of the ill-posed nature of these image modeling tasks, it is necessary to design priors for getting desired solutions. For example, many image restoration tasks utilize a sparsity prior as the regularization term~\cite{beck2009fasta,rudin1992nonlinear}. In particular, Eq.~\eqref{eq:model_orginal} can be minimized by a broad class of general numerical optimization methods, among which Proximal Gradient (PG)~\cite{bertsekas2015convex}, Block Coordinate Descent (BCD)~\cite{tseng2001convergence}, Half Quadratic Splitting (HQS)~\cite{nikolova2005analysis} and Alternating Direction Method of Multipliers (ADMM)~\cite{boyd2011distributed} are proven to be the most reliable methods. Over the past decades, many efforts have been devoted to these schemes. For example, by integrating Nesterov's accelerated gradient method~\cite{nesterov1983method} into the fundamental PG scheme, APG is initially developed for convex models~\cite{beck2009fasta,beck2009fast}. Subsequently, other typical APGs are derived solving problem~\eqref{eq:model_orginal}, including monotone APG (mAPG)~\cite{li2015accelerated}, inexact APG (niAPG)~\cite{yao2016efficient}, and momentum APG for nonconvex problem (APGnc)~\cite{li2017convergence}, etc. Optimization designed priors strategies provide a mathematical understanding of their behaviors with well-defined regularization properties. To construct and solve the flexible, exact proper prior is challenging. However, simple regularizer performs poorly when compared with state-of-the-art methodologies in real-world applications. This is because these methods can not adequately exploit the particular structures of the image processing tasks at hand and the data distributions. These limits make it difficult to solve the problems in a purely optimized manner with designed priors.
  
In recent years, various deep learning strategies~\cite{liu2016learning,chen2017trainable,gregor2010learning} have been proposed to address image modeling problems. These discriminative learning methods aim to combine the classical numerical solvers and the collected training data to obtain some task-specific iterations. Similar to this view, plug-in schemes replace the regularization term by using task-related operator and have been extensively studied with a great empirical success for vision problems~\cite{venkatakrishnan2013plug,zhang2017learning,Wang2017Parameter,chan2017plug,zhang2019deep}. Indeed, these algorithms perform better than some state-of-the-art methods in real-world applications. Unfortunately, these plug-in schemes with implicit regularization term may break the properties and structures of the objective in Eq.~\eqref{eq:model_orginal}. Thus, the existing proofs only demonstrate that the iteration sequences converge to a fixed point without knowing the relationship between it and the optimal solutions. By introducing spectral normalization technique for more accurately constraining deep learning-based denoisers, the fixed-point theory is established in~\cite{ryu2019plug}. Whereas this theory result is effective only under the strongly convex condition of function $f(\mathbf{x})$ which is unsatisfied in plenty of vision problems, for example when the data fitting term $f(\mathbf{x})$ is setting as $\frac{1}{2}\|\mathbf{Ax}-\mathbf{b}\|_2^2$ with non-column full rank matrix $\mathbf{A}$, the strong convex property is unattainable. In contrast to these implicit plug-and-play methods, an explicit plug-in scheme is developed in~\cite{romano2017little}, named Regularization by Denoising (RED), with explicit Laplacian-based regularization functional. While the regularization term is required to be symmetric which is unsatisfied in many state-of-the-art methods, such as NLM~\cite{Kostadin2007Image}, RF~\cite{unser1991recursive} and BM3D~\cite{Kostadin2007Image} etc. Further work has been discussed in~\cite{hong2019acceleration,reehorst2019regularization}. Optimal condition-based approaches, such as~\cite{liu2019deep,liu2019convergence} have also been developed for solving Eq.~\eqref{eq:model_orginal}. However, their conditions often rely on estimating the sub-differential, which can only be implicitly attained. Moreover, there is no mechanism to enforce constraints for their inference processes.

As discussed above, existing approaches often aim to introduce complex priors to Eq.~\eqref{eq:model_orginal} and then solve the optimization model by numerical iterations, deep network architectures or their combinations. However, designing efficient regularization often requires deep understandings on the problem and elaborately mathematical skills. Moreover, it is indeed challenging to strictly investigate the mechanisms and analyze the behaviors for these prior modules.
To address the above issues, we propose a new framework, named Task-driven Latent Feasibility (TLF), to investigate the task related information for image modeling from the perspective of latent constraints. Specifically, we introduce an energy-based model as our problem constraint and build a simple bilevel formulation for MAP inference (please see Fig.~\ref{fig:demo} for illustration). In this way, we can utilize rich task information to help narrow down the solution space for our optimization process. We can also introduce data-driven modules to further integrate training data for latent feasibility investigation. Theoretically, we prove that both the proposed TLF and its data-driven extension can converge to a critical point of the original problem in Eq.~\eqref{eq:model_orginal}. We also demonstrate how to apply TLF to address different real-world image processing applications and extensive experimental results verify the efficiency of our method on all the tested problems. In summary, the contributions of this paper mainly include: 
\begin{itemize}
	\item TLF provides a new perspective to understand and formulate task-driven latent feasibility for nonconvex and nonsmooth MAP-based inference tasks. By further embedding a series of data-driven architectures, TLF can also exploit rich data distribution and specific structure information for image modeling.
	\item By investigating the monotonicity and boundedness properties of the iterations, we establish new control mechanisms to guide our MAP-based image propagation towards the desired solutions. We also strictly prove the convergence of TLF and its data-driven extension.
	\item On the application side, we demonstrate how to utilize TLF as a flexible framework to integrate a variety of knowledge-based and data-driven modules to address real-world vision applications. Extensive experiments show the superiority of TLF on the tested problems.
\end{itemize}

\section{The Proposed Algorithmic Framework}\label{sec-proposed-algorithm}

We first introduce a generic feasibility $\mathcal{G}$ to the nonconvex image modeling problem in Eq.~\eqref{eq:model_orginal}, i.e.,
\begin{equation}\label{eq:model}
\min\limits_{\mathbf{x}} F(\mathbf{x}),\ s.t. \ \mathbf{x} \in\mathcal{G}.
\end{equation}
In conventional approaches, $\mathcal{G}$ is commonly formulated as explicit feasible regions and/or (in)equations~\cite{bao2014l0,li2016rain}. However, it is indeed challenging to utilize these straightforward constraints to characterize exact solution space for complex real-world vision applications. Therefore, in this work, we would like to propose a new viewpoint to understand and formulate latent feasibility for image modeling. 

\subsection{Task-driven Latent Feasibility (TLF)}\label{subsec:knowledge}
Specifically, we first consider the solution space of the following component energy model as our latent feasibility to Eq.~\eqref{eq:model_orginal}, i.e.,
\begin{equation}\label{eq:inner}
\mathcal{G}(\mathbf{x})\in\arg\min\limits_{\mathbf{x}}\{g(\mathbf{x}) + \phi(\mathbf{x})\},
\end{equation}
where $g(\mathbf{x})$ is differentiable and $\phi(\mathbf{x})$ is nonsmooth and possibly nonconvex. We assume that $g(\mathbf{x})$ and $\phi(\mathbf{x})$ are proper and lower semi-continuous. By specifying $\mathcal{G}$ as that in Eq.~\eqref{eq:inner}, we actually obtain a simple bilevel optimization model. Here we would like to utilize the proximal average technique \cite{sabach-2017:first,mu2020image} to solve this problem. Specifically, we first apply the standard PG rule on Eq.~\eqref{eq:model_orginal}, i.e.,
$\mathcal{F}_t(\mathbf{x})\in\mathtt{prox}_{t\psi}(\mathbf{x}-t\nabla f(\mathbf{x})),$
where $\mathtt{prox}_{t\psi}(\cdot)=\arg\min_{\mathbf{x}}\left\{\psi(\mathbf{x})+\frac{1}{2}\|\mathbf{x}-\cdot\|^2\right\}$, and $t$ denotes step size. Then by averaging ${\mathcal{G}(\mathbf{x}^k)}$ and $\mathcal{F}_{t_k}(\mathbf{x}^k)$ with parameter sequence $\{\alpha^k\}$, we have the following aggregated updating scheme with temporal variable $\mathbf{v}^{k}$:
\begin{equation}
	\mathbf{v}^{k} =(1-\alpha^k)\mathcal{F}_{t_k}(\mathbf{x}^k) + \alpha^k \mathcal{G}(\mathbf{x}^k).
\end{equation}
It can be seen that the above formulation actually provides a way to integrate information from both the original objective and the latent feasibility model. Since the constraint module $\mathcal{G}$ may introduce uncertainty to the propagation, we further design a correction step, named Monotone Descent Updating Scheme ($\text{MDUS}$ for short and stated in Alg.~\ref{alg:MDUF}), to guarantee the monotonicity of the image propagation, i.e., $F(\mathbf{x}^{k+1})\leq F(\mathbf{x}^k_{\mathcal{F}})$. So the complete TLF iteration scheme can be summarized in Alg.~\ref{alg:TLF}. 
\begin{algorithm}[t]
	\caption{Task-driven Latent Feasibility (TLF)}\label{alg:TLF}
	\begin{algorithmic}[1]
		\REQUIRE The input $\mathbf{x}^0$, $\alpha^0$, parameters $t_k\in(0,1/L^f)$ and $\gamma\in(0,1)$.
		\WHILE{not converged}
		\STATE $\mathbf{x}^k_{\mathcal{F}}=\mathcal{F}_{t_k}(\mathbf{x}^k)$ 
		and 
		$\mathbf{x}_{\mathcal{G}}^k=\mathcal{G}(\mathbf{x}^k)$.\label{step:x_sub}
		\STATE $\mathbf{v}^{k}=\alpha^{k}\mathbf{x}_{\mathcal{G}}^{k}+(1-\alpha^{k})\mathbf{x}_{\mathcal{F}}^{k}$.\label{z-update-ex}
		\STATE $\mathbf{x}^{k+1}=\text{MDUS}(\mathbf{v}^k,\mathbf{x}^k_{\mathcal{F}};\alpha^k,\gamma)$.
		\ENDWHILE
	\end{algorithmic}
\end{algorithm}
\begin{algorithm}[t]
	\caption{$\mathbf{x}^{k+1}=\text{MDUS}(\mathbf{v}^k,\mathbf{x}^k_{\mathcal{F}};\alpha^k,\gamma)$}\label{alg:MDUF}
	\begin{algorithmic}[1]
		\IF{$F(\mathbf{v}^{k})\leq F(\mathbf{x}_{\mathcal{F}}^{k}) $}
		\STATE  $\mathbf{x}^{k+1}=\mathbf{v}^{k}$.
		\ELSE
		\STATE $\mathbf{x}^{k+1}=\mathbf{x}_{\mathcal{F}}^{k}$.
		\ENDIF
		\STATE $\alpha^{k+1}=\gamma\alpha^k$.
	\end{algorithmic}
\end{algorithm}

\subsection{Data-driven TLF (DTLF)}
As complex data distribution in real-world applications will affect the energy-based optimization problem, a data-driven feasibility module (denoted as $\mathcal{G}^{\mu}$) is introduced to incorporate designed/trained architectures to optimize Eq.~\eqref{eq:model_orginal}. Specifically, the network-based building block at the $k$-th iteration can be denoted as $\mathcal{N}(\cdot;\bm{\mathcal{W}}_T^k)$, where $\bm{\mathcal{W}}_T^k=\{\mathcal{W}_t^k\}_{t=0}^T$ is the set of learnable parameters with $T$-th training stage\footnote{Please refer to the next section for the detailed structures of $\mathcal{N}$ in real-world applications.}. Practically, the network parameters $\bm{\mathcal{W}}_T^k$ are updated to adapt different task (e.g., in image deblurring application, $\mathcal{N}(\cdot;\bm{\mathcal{W}}_T^k)$ can be designed/trained to adapt different noise level) in each iteration step. 
We denote the temporary variable at the $k$-th iteration as $\tilde{\mathbf{x}}^{k}=\mathcal{N}(\mathbf{x}^k,\bm{\mathcal{W}}_T^k)$. By further considering data-driven extension scheme as a proximal approximation of Eq.~\eqref{eq:inner} with parameter $\mu_k>0$ at the $k$-th iterative, $\mathcal{G}^{\mu}$ can be formulated as the following form
\begin{equation}\label{eq:DF}
\mathcal{G}^{\mu}(\mathbf{x},\tilde{\mathbf{x}}^k)\in\arg\min\limits_{\mathbf{x}}\{g(\mathbf{x}) + \phi(\mathbf{x}) + \mu_k/2\|\mathbf{x}-\tilde{\mathbf{x}}^k\|^2\}.
\end{equation}

By embedding designed/trained architectures to task-driven feasibility module, we then develop Data-driven TLF (DTLF) to optimize the minimization problem described in Eq.~\eqref{eq:model} with the extension scheme (i.e., Eq.~\eqref{eq:DF}). Indeed, $\tilde{\mathbf{x}}^k$ is the output of network $\mathcal{N}(\mathbf{x}^k,\mathcal{W}_T^k)$ which is used to approximate task-driven module $\mathcal{G}(\mathbf{x})$ at the $k$-th iteration. In other words, the network-based building blocks aim to learn task behaviors during iterations. Similar to TLF, it is still free for selecting a method to solve the constraint-based subproblem in Eq.~\eqref{eq:DF}. While, to control the network-based iteration sequence, we introduce a relative loose boundness condition about $\mathcal{G}^{\mu}(\mathbf{x},\tilde{\mathbf{x}}^k)$. This strategy prevents improperly designed/trained architectures, which may deflect our iterative trajectory towards unwanted solutions. The monitor is obtained by checking the boundedness of $\|\mathbf{x}^k_{\mathcal{G}_{\mu}}-\mathbf{x}^k\|$. We summarize the complete DTLF iterations and Boundedness-based Updating Scheme (BUS) in Alg.~\ref{alg:DTLF} and Alg.~\ref{alg:bound}, respectively. The convergence behaviors will be analyzed in the following section.

\begin{remark}
In fact, both MDUS and BUS can help us automatically recognize proper modules to generate our TLF iterations. Specifically, in Alg.~\ref{alg:TLF}, we introduce MDUS (i.e., Alg.~\ref{alg:MDUF}) to guarantee that the proximal averaged propagation can always decrease the objective. If the monotonicity of the objective cannot be guaranteed, TLF will reject this module and perform standard numerical updating to correct the iterations. 	
Moreover, in Alg.~\ref{alg:DTLF}, we further design BUS (i.e., Alg.~\ref{alg:bound}) to introduce a boundedness criterion to prevent improper data-driven architectures. That is, these data-driven modules will be rejected if they do not satisfy our boundedness condition stated in Alg.~\ref{alg:bound}. Please notice that we still perform MDUS to guide the overall iteration for DTLF.
\end{remark}

\begin{algorithm}[t]
	\caption{Data-driven TLF (DTLF)}\label{alg:DTLF}
	\begin{algorithmic}[1]
		\REQUIRE The input $\mathbf{x}^0$, $\mu^0$, $\alpha^0$, parameters $t_k\in(0,1/L^f)$,  $\beta,\ \gamma\in(0,1)$ and $C>0$.
		\WHILE{not converged}
		\STATE $\mathbf{x}^k_{\mathcal{F}}=\mathcal{F}_t(\mathbf{x}^k)$, $\tilde{\mathbf{x}}^k=\mathcal{N}(\mathbf{x}^k,\bm{\mathcal{W}}^k_T)$.
		\STATE
		$\mathbf{x}_{\mathcal{G}_{\mu}}^k=\mathcal{G}_{\text{DF}}^{\mu}(\mathbf{x},\tilde{\mathbf{x}}^k)$.
		\STATE $\mathbf{z}^{k}=\alpha^{k}\mathbf{x}_{\mathcal{G}_{\mu}}^{k}+(1-\alpha^{k})\mathbf{x}_{\mathcal{F}}^{k}$.
		\STATE $\mathbf{u}^k=\text{BUS}(\mathbf{x}^k,\mathbf{x}^k_{\mathcal{G}},\mathbf{x}^k_{\mathcal{G}_{\mu}},\mathbf{z}^k;\mu^k,\beta)$.
		\STATE $\mathbf{x}^{k+1}=\text{MDUS}(\mathbf{u}^k,\mathbf{x}^k_{\mathcal{F}};\alpha^k,\gamma)$.\label{eq:Alg2Condition1}
		\ENDWHILE
	\end{algorithmic}
\end{algorithm}

\begin{algorithm}[t]
	\caption{$\mathbf{x}^{k+1}=\text{BUS}(\mathbf{x}^k,\mathbf{x}^k_{\mathcal{G}},\mathbf{x}^k_{\mathcal{G}_{\mu}},\mathbf{z}^k;\mu^k,\beta)$}\label{alg:bound}
	\begin{algorithmic}[1]
		\IF{$\|\mathbf{x}_{\mathcal{G}_{\mu}}^{k}-\mathbf{x}^{k}\|\leq C \|\mathbf{x}_{\mathcal{G}}^{k}-\mathbf{x}^{k}\|$}\label{eq:AlgNetCondition}
		\STATE $\mathbf{u}^k = \mathbf{z}^k$.
		\ELSE  
		\STATE $\mathbf{u}^k = \alpha^{k}\mathbf{x}_{\mathcal{G}}^{k}+(1-\alpha^{k})\mathbf{x}_{\mathcal{F}}^{k}$, $\mu^{k+1}=\beta\mu^k$.
		\ENDIF
	\end{algorithmic}
\end{algorithm}

\section{Convergence Analysis}\label{subsec:convergence}
In this part, we would like to discuss the convergent behaviors for the proposed TLF and its extension form (i.e., DTLF) under some loose conditions. We suggest readers to refer to~\cite{rockafellar2009variational} for some definitions in variational analysis, such as proper, lower-semicontinuous, coercive and the limiting subdifferential which will be useful in the following analysis. Our convergence results are also based on the following fairly loose assumptions.
 
\begin{assumption}\label{assume}
	The objective function $F(\mathbf{x})$ in Eq.~\eqref{eq:model_orginal} is proper, lower-semicontinuous and coercive. Function $f(\mathbf{x})$ is convex and Lipschitz smooth, i.e., $\forall \mathbf{x},\mathbf{y}\in\mathbb{R}^{D}$, we have $\|\nabla f(\mathbf{x}) - \nabla f(\mathbf{y})\|\leq L^f\|\mathbf{x}-\mathbf{y}\|$, where $L^f$ is the Lipschitz constant for $\nabla f$.
\end{assumption}

\begin{thm}\label{thm:convergence}
	Let $\{\mathbf{x}^k,\mathbf{x}^k_{\mathcal{F}},\mathbf{x}^{k+1}\}_{k\in\mathbb{N}}$ be the iteration sequences generated by Alg.~\ref{alg:TLF}. Then the theoretical results are summarized in the following.
	\begin{itemize}
		\item There exists a constant $\sigma>0$ satisfying that 
		$
		F(\mathbf{x}^{k+1})\leq F(\mathbf{x}^k)-\sigma\|\mathbf{x}_{\mathcal{F}}^{k+1}-\mathbf{x}^k\|^2.
		$
		\item Let $\mathbf{x}^{\ast}$ be any accumulation of the sequence $\{\mathbf{x}^{k+1}\}$ which implies that $\mathbf{x}^{\ast}$ is a critical point of the minimization problem in Eq.~\eqref{eq:model_orginal}, i.e., $0\in\partial F(\mathbf{x}^{\ast})$.
	\end{itemize}
\end{thm}
\begin{proof}
	We first show the sufficient descent property about the objective function $F(\mathbf{x})$. By using the proximal gradient scheme in Step \ref{step:x_sub} of Alg.~\ref{alg:TLF} and the Lipschitz property of $f(\mathbf{x})$, i.e.,  
	\begin{equation*}\label{psi-ineq}
	\begin{array}{l}
	\psi(\mathbf{x}_{\mathcal{F}}^k)\!+\!\big<\mathbf{x}_{\mathcal{F}}^k-\mathbf{x}^k,\nabla f(\mathbf{x}^k)\big>\!+\!\frac{1}{2s}\|\mathbf{x}_{\mathcal{F}}^k-\mathbf{x}^k\|^2\leq\psi(\mathbf{x}^k)\ \text{and}\\
	f(\mathbf{x}_{\mathcal{F}}^k)\leq f(\mathbf{x}^k)+\big<\mathbf{x}_{\mathcal{F}}^k-\mathbf{x}^k,\nabla f(\mathbf{x}^k)\big>+\frac{L^f}{2}\|\mathbf{x}_{\mathcal{F}}^k-\mathbf{x}^k\|^2,
	\end{array}
	\end{equation*}
	we conclud that 
	$F(\mathbf{x}_{\mathcal{F}}^k)\leq F(\mathbf{x}^k)-(\frac{1}{2s}-\frac{L^f}{2})\|\mathbf{x}_{\mathcal{F}}^k-\mathbf{x}^k\|^2.$ From the MDUS correction scheme, the sufficient descent property is obtained with $\sigma = \frac{1}{2s}-\frac{L^f}{2}$. Then, we will show the second item. The descent inequality in the first item implies 	$\sum_{k=1}^{\infty} \|\mathbf{x}_{\mathcal{F}}^k-\mathbf{x}^k\|^2<\infty$, 	which means that $\|\mathbf{x}_{\mathcal{F}}^k-\mathbf{x}^k\|\rightarrow 0,\ k\rightarrow\infty$. Thus, there exists subsequence $\{\mathbf{x}_{\mathcal{F}}^{k_p}\}$ and $\{\mathbf{x}^{k_p}\}$ converge to a same point $\mathbf{x}^{\ast}$ as $p\rightarrow\infty$. Incorporating the lower-semicontinuous of $F(\mathbf{x})$ and the supreme principle, we obtain that 	$\lim_{p\rightarrow\infty}F(\mathbf{x}^{k_p})=F(\mathbf{x}^{\ast})$. With the optimal condition $	0\in\partial\psi(\mathbf{x}_{\mathcal{F}}^k)+\nabla f(\mathbf{x}^k)+\frac{1}{s}(\mathbf{x}_{\mathcal{F}}^k-\mathbf{x}^k)$, we know that $-\frac{1}{s}(\mathbf{x}_{\mathcal{F}}^k-\mathbf{x}^k)-\nabla f(\mathbf{x}^k)+\nabla f(\mathbf{x}_{\mathcal{F}}^k)\in\partial F(\mathbf{x}_{\mathcal{F}}^k)$. Actually, for $ k\rightarrow\infty$, we have $	\|\frac{1}{s}(\mathbf{x}_{\mathcal{F}}^k-\mathbf{x}^k)+\nabla f(\mathbf{x}^k)-\nabla f(\mathbf{x}_{\mathcal{F}}^k)\| \leq	(\frac{1}{s}+L^f)\|\mathbf{x}_{\mathcal{F}}^k-\mathbf{x}^k\|\rightarrow 0$. 
	The above implies that $\mathbf{x}^{\ast}$ is a critical point. This complete the proof.
\end{proof}

\begin{figure*}[t]
	\centering
	\begin{tabular}{c@{\extracolsep{0.2em}}c@{\extracolsep{0.2em}}c@{\extracolsep{0.2em}}c}
		\includegraphics[height=0.145\textwidth]{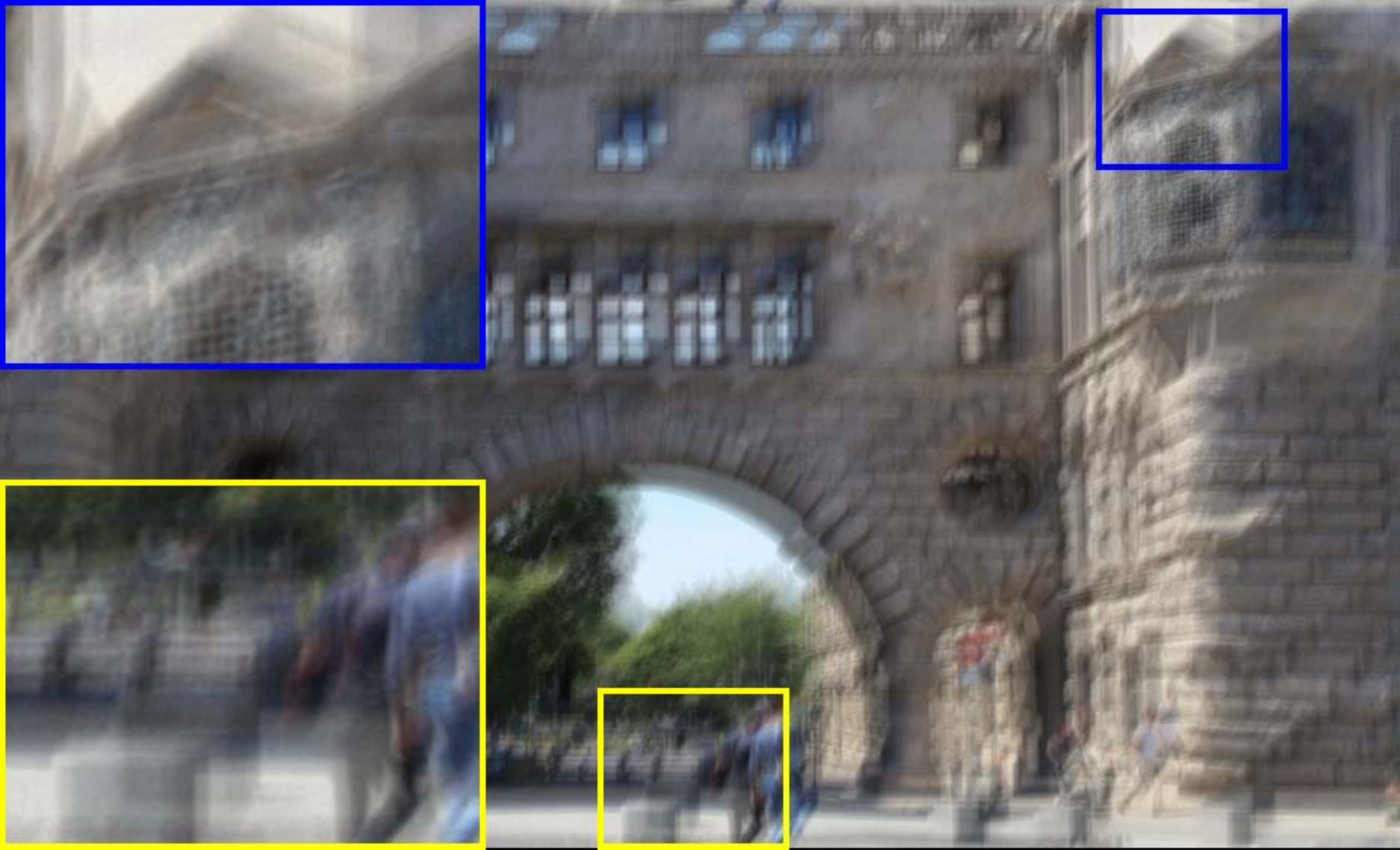}&
		\includegraphics[height=0.145\textwidth]{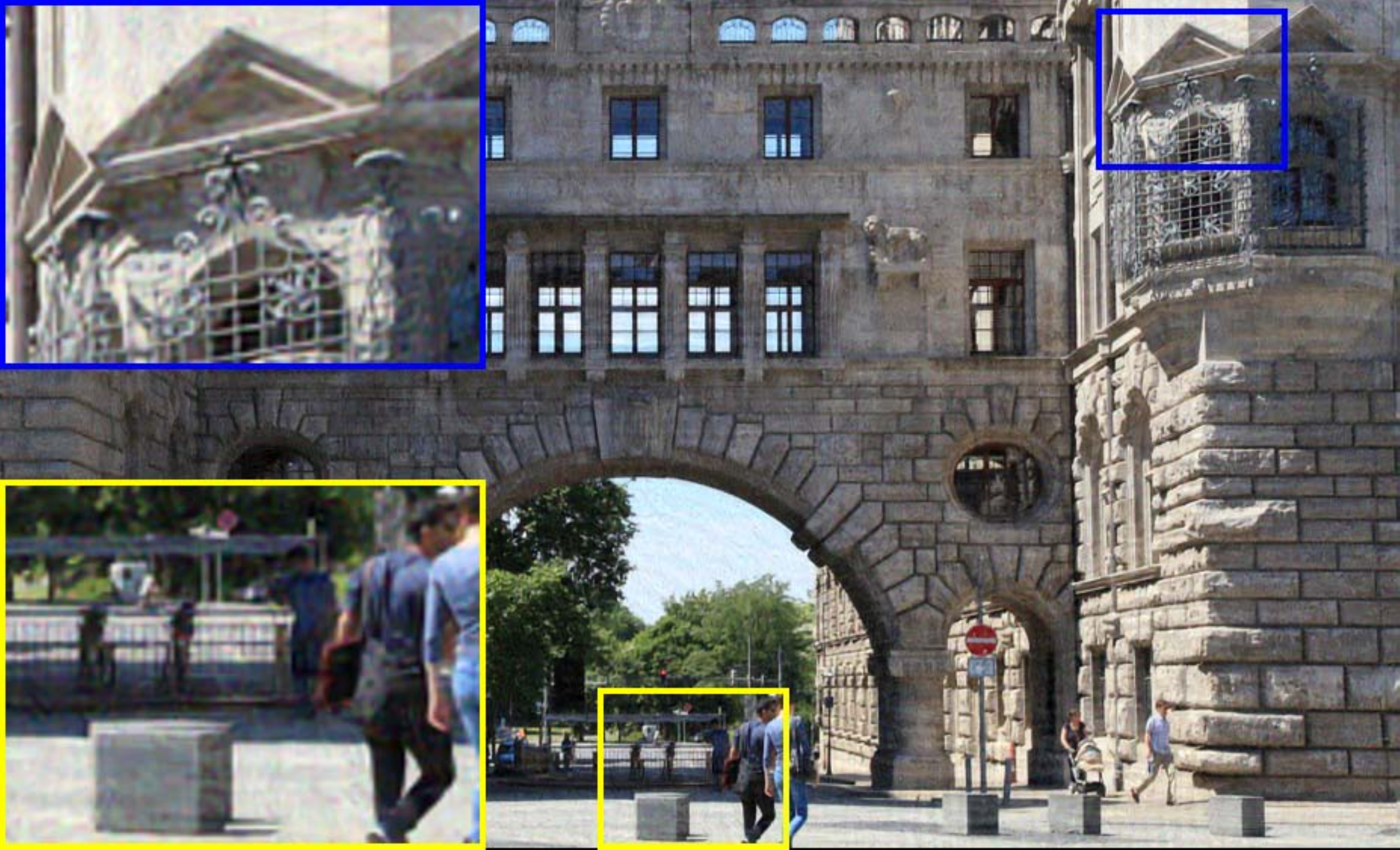}&
		\includegraphics[height=0.145\textwidth]{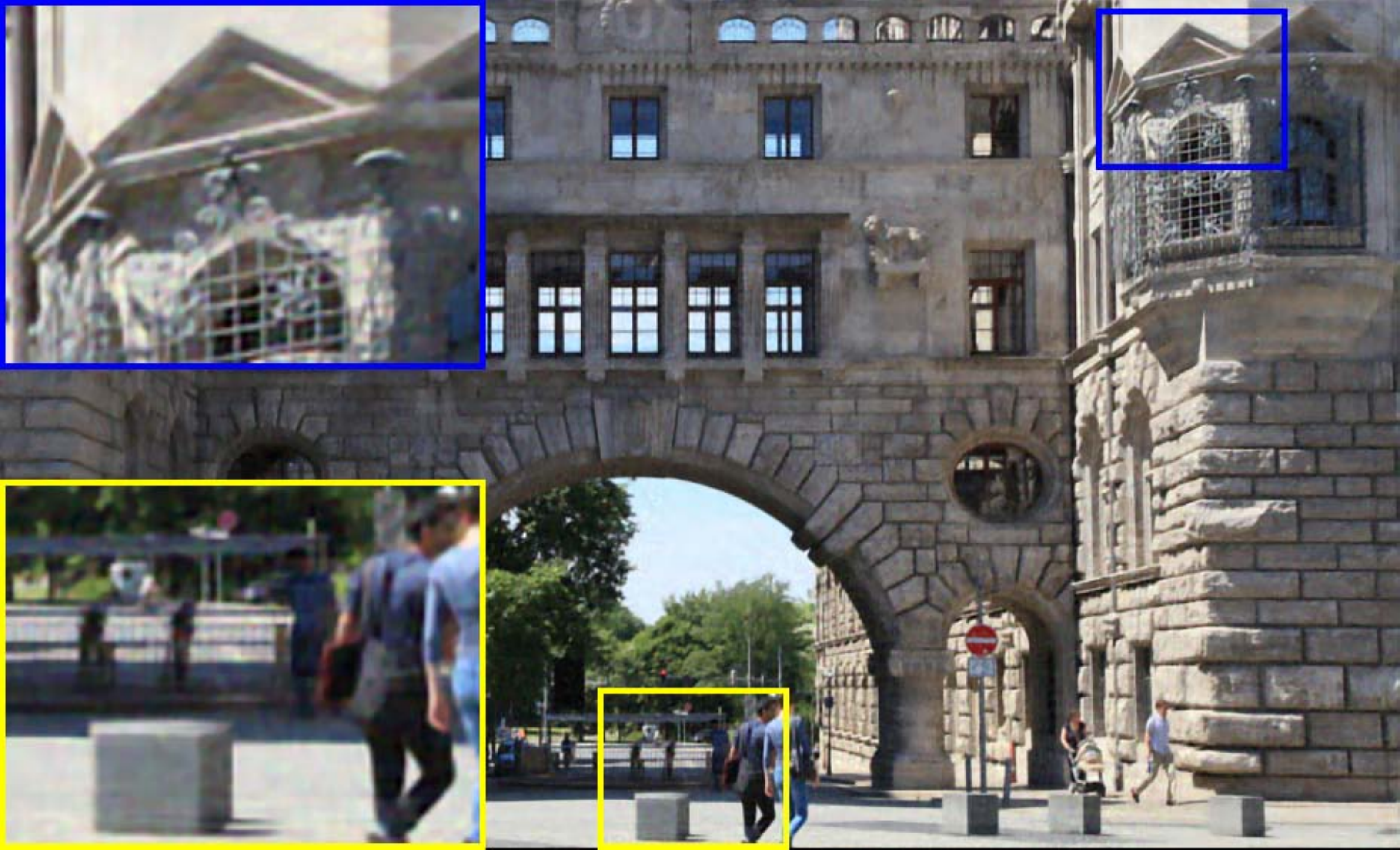}&
		\includegraphics[height=0.145\textwidth]{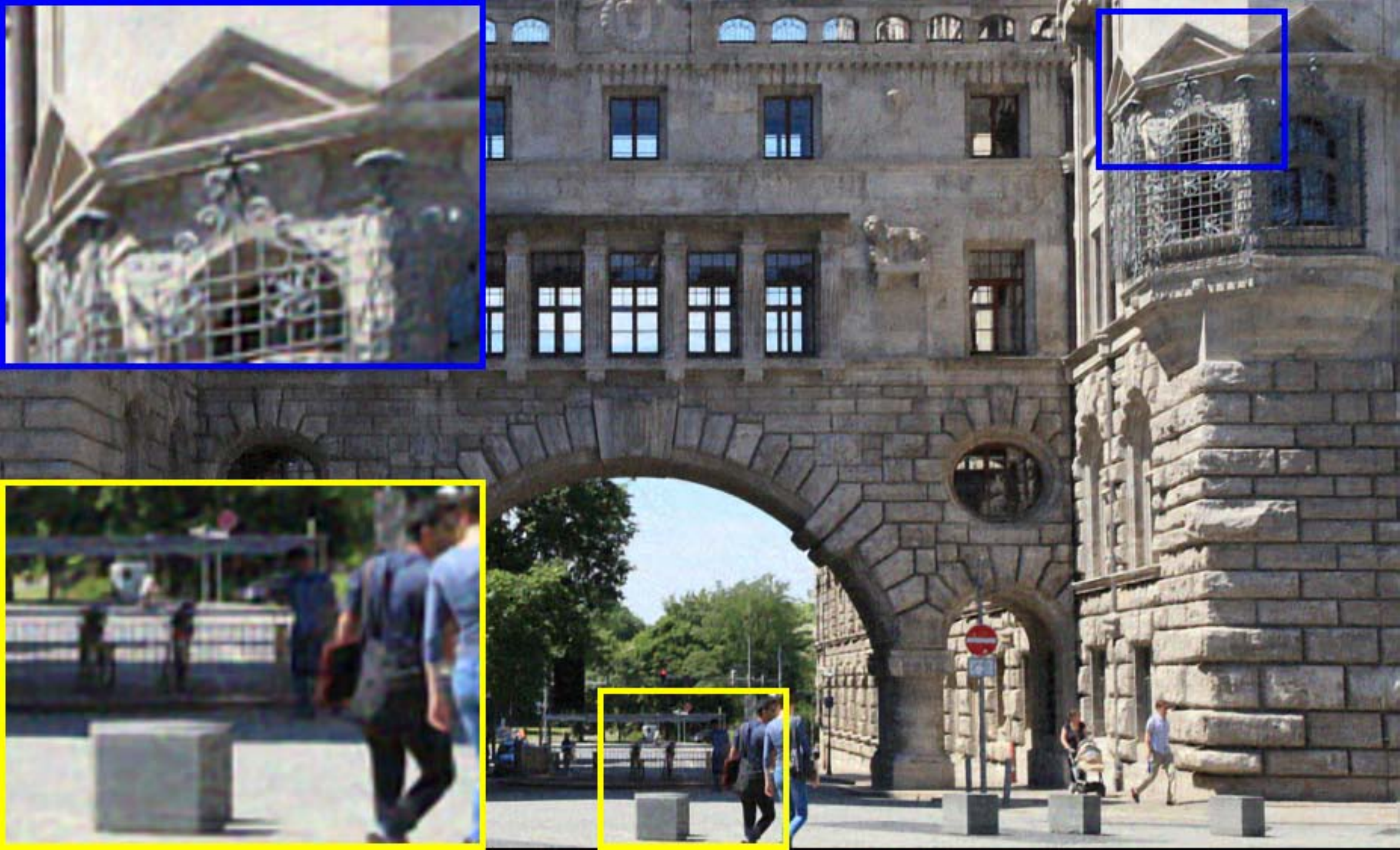}\\
		\footnotesize (a) PSNR / SSIM &\footnotesize (b) 28.1463 / 0.8064 &\footnotesize (c) 28.6781 / 0.8436 &\footnotesize (d) 29.2157 / 0.8447\\
		\footnotesize Input &\footnotesize Eq.~\eqref{eq:model_orginal} &\footnotesize Eq.~\eqref{eq:DF} &\footnotesize Eq.~\eqref{eq:model}\\
	\end{tabular}
	\caption{Illustrating the effectiveness of proposed method for minimize Eq.~\eqref{eq:model_orginal} with the comparison results (i.e., PSNR and SSIM scores). Subfigure (a) is input image. (b) and (c) are the optimization results of Eq.~\eqref{eq:model_orginal} and data-driven scheme (i.e., Eq.~\eqref{eq:DF}) respectively. Subfigure (d) is the result by optimizing Eq.~\eqref{eq:model} with DTLF.}
	\label{fig:CompOur_deblur}. 
\end{figure*}

\begin{thm}\label{cor:global}
	Let $F(\mathbf{x})$ be a semi-algebraic function\footnote{Please refer to \cite{bolte2014proximal} for the formal definition of semi-algebraic function. Actually, many functions arising in learning and vision areas, including $\ell_0$ norm, rational $\ell_p$ norms (i.e., $p=p_1/p_2$ with positive integers $p_1$ and $p_2$) and their finite sums or products, are all semi-algebraic.}. Then we can further have that the sequence $\{\mathbf{x}^k\}_{k\in\mathbb{N}}$ in Alg.~\ref{alg:TLF} has finite length, i.e.,	$\sum_{k=1}^{\infty}\|\mathbf{x}^{k+1}-\mathbf{x}^k\|<\infty$. 
\end{thm}
\begin{proof}
	With the KL property (see \cite{bolte2014proximal}) and the definition of sub-differential, we have 	$\varphi'(F(\mathbf{x}_{\mathcal{F}}^k)-F(\mathbf{x}^{\ast}))\geq \frac{s}{1+sL^f}\|\mathbf{x}_{\mathcal{F}}^k-\mathbf{x}^k\|^{-1}$, where $\varphi$ is the desingularizing function. From the concavity of $\varphi$, we obtain
	\begin{equation*}
	\begin{array}{l}
	\varphi\big(F(\mathbf{x}^{k+1})-F(\mathbf{x}^{\ast})\big)-\varphi\big(F(\mathbf{x}^{k+2})-F(\mathbf{x}^{\ast})\big)\\
	\geq\varphi'(F(\mathbf{x}^{k+1})-F(\mathbf{x}^{\ast}))\big( F(\mathbf{x}^{k+1}-F(\mathbf{x}^{k+2})) \big)\\
	\geq\frac{s}{1+sL^f}\frac{1}{\|\mathbf{x}_{\mathcal{F}}^k-\mathbf{x}^k\|}\cdot\frac{1-sL^f}{2s}\|\mathbf{x}_{\mathcal{F}}^{k+1}-\mathbf{x}^{k+1}\|^2.
	\end{array}
	\end{equation*}
	If we denote $\hat{\vartriangle}_{k,k+1}:=\varphi(F(\mathbf{x}^{k})-\varphi(F(\mathbf{x}^{\ast}))-\varphi(F(\mathbf{x}^{k+1})-F(\mathbf{x}^{\ast}))$ and $\hat{C}=\frac{1-sL^f}{2(1+sL^f)}$, the inequality 	$\|\mathbf{x}_{\mathcal{F}}^{k+1}-\mathbf{x}^{k+1}\|^2\leq\hat{C}\hat{\vartriangle}_{k+1,k+2}\|\mathbf{x}_{\mathcal{F}}^k-\mathbf{x}^k\|
	$ holds which implies $	2\|\mathbf{x}_{\mathcal{F}}^{k+1}-\mathbf{x}^{k+1}\|\leq\|\mathbf{x}_{\mathcal{F}}^k-\mathbf{x}^k\| + \hat{C}\hat{\vartriangle}_{k+1,k+2}$. Subsequently, we have
	\begin{equation*}
	\begin{array}{l}
	2\sum\limits_{i=l+1}^{k}\|\mathbf{x}_{\mathcal{F}}^{i+1}-\mathbf{x}^{i+1}\|
	\leq\sum\limits_{i=l+1}^{k}\|\mathbf{x}_{\mathcal{F}}^{i+1}-\mathbf{x}^{i+1}\|\\
	+\|\mathbf{x}_{\mathcal{F}}^{i+1}-\mathbf{x}^{i+1}\|
	+\hat{C}\hat{\vartriangle}_{l+1,k+2}.
	\end{array}
	\end{equation*}
	Obviously, the above inequality implies the finite length of sequence $\{\mathbf{x}_{\mathcal{F}}^k-\mathbf{x}^k\}$. If $\mathbf{x}^{k+1}=\mathbf{x}_{\mathcal{F}}^k$, the inequality  $\sum_{k=1}^{\infty}\|\mathbf{x}^{k+1}-\mathbf{x}^k\|<\infty$ holds. If $\mathbf{x}^{k+1}=\alpha^k\mathbf{x}_{\mathcal{G}}^{k}+(1-\alpha^k)\mathbf{x}_{\mathcal{F}}^k$, with the proper, lower-semicontinuous and coercive property of $F$, the sequence $\{\mathbf{x}_{\mathcal{G}}^{k+1}\}$ is bounded. Then with the update scheme about $\alpha^k$ and the finite length of $\{\mathbf{x}_{\mathcal{F}}^k-\mathbf{x}^k\}$, we have
	\begin{equation*}\label{eq:cauchy}
	\begin{array}{l}
	\!\sum\limits_{k=1}^{\infty}\!\|\mathbf{x}^{k\!+\!1}\!-\!\mathbf{x}^k\|
	\!\leq\!\sum\limits_{k=1}^{\infty}\!\left(\alpha^k\|\mathbf{x}_{\mathcal{G}}^{k}\!-\!\mathbf{x}^k\|
	\!+\!(1\!-\!\alpha^k)\|\mathbf{x}_{\mathcal{F}}^{k}\!\!-\!\mathbf{x}^k\|\right)
	\!<\!\infty.
	\end{array}
	\end{equation*}	
	This completes the proof.
\end{proof}

Indeed, the summable sequence $\{\|\mathbf{x}^{k+1}-\mathbf{x}^k\|\}_{k\in\mathbb{N}}$ as stated in Theorem~\ref{cor:global} implies that there exists $m>n>l$ satisfying   $\|\mathbf{x}^m-\mathbf{x}^n\|\leq\sum_{k=n}^{m-1}\|\mathbf{x}^{k+1}-\mathbf{x}^k\|\rightarrow 0$, as $l\rightarrow\infty$. Subsequently, it follows that the iteration  $\{\mathbf{x}^k\}_{k\in\mathbb{N}}$ is a Cauchy sequence and hence is a globally convergent sequence which is also defined as sequence convergent.

\begin{remark}
As described in~\ref{subsec:knowledge}, the objective function $F(\mathbf{x})$ is sufficiently descent in Alg.~\ref{alg:DTLF} and it is easy to check that the convergence results of Alg.~\ref{alg:DTLF} can be obtained in the same manner as stated in Theorem~\ref{thm:convergence}. The temporary iteration $\mathbf{u}^k$ in Alg.~\ref{alg:DTLF} is bonded under the checking condition. This implies the boundness of  $\|\mathbf{x}_{\mathcal{G}_{\mu}}^{k}-\mathbf{x}^k\|$. It can be concludes that, in Alg.~\ref{alg:DTLF}, sequence convergent property of $\{\mathbf{x}^k\}$ is attained.
\end{remark}

\section{Applications for Image Modeling}\label{sec:app}

We emphasize that different from these existing image modeling approaches, the proposed TLF allows us to introduce an energy-based feasibility module when solving the optimization model in Eq.~\eqref{eq:model_orginal}. This section first considers non-blind deblurring and image inpainting applications. We take non-blind deblurring as an illustrative example for establishing TLF. Then, we extend TLF to even more challenging single image rain streaks removal task.

\subsection{Image Deblurring}
Here we consider a particular non-blind deblurring task to recover latent image $\mathbf{x}$ from blurred observation $\mathbf{b}$. By formulating this problem as a sparse coding model $\mathbf{b}=\mathbf{D}\bm{\beta} +\mathbf{n}$, where $\bm{\beta}$ denotes the sparse code, $\mathbf{D}$ is a given dictionary and $\mathbf{n}$ is unknown noise, we derive a specific case of Eq.~\eqref{eq:model_orginal}, that is $\min_{\bm{\beta}} \frac{1}{2}\|\mathbf{D}\bm{\beta}-\mathbf{b}\|^2+\lambda_1\|\bm{\beta}\|_p,$ where $p\in[0,1]$, $\lambda_1>0$. Here we follow standard settings in \cite{beck2009fast} to define $\mathbf{D}$ as the inverse wavelet transform. Indeed, $\mathbf{D}$ can be denoted as $\mathbf{D} = \mathbf{K}\mathbf{W}^{\top}$, where $\mathbf{K}$ is the matrix of the blur kernel $\mathbf{k}$, and $\mathbf{W}^{\top}$ is the inverse of the wavelet transform of $\mathbf{W}$. Subsequently, it can be equivalently described as the following intuitive form, i.e., 
\begin{equation}\label{eq:ISTA}
\min\limits_{\mathbf{x}} \frac{1}{2}\|\mathbf{K}\mathbf{x}-\mathbf{b}\|^2+\lambda_1\|\mathbf{W}\mathbf{x}\|_p,
\end{equation}
where we actually have the latent image with the form $\mathbf{x}=\mathbf{W}^{\top}\bm{\beta}$. In the following, we demonstrate how to utilize TLF and DTLF to solve the above nonconvex image modeling problem.

\subsubsection{TLF Strategy} 
As for the module $\mathcal{G}$, we tend to design a relatively simple and task-related model to enforce our constraints on the latent image. Specifically, we consider the following widely used Total Variation (TV) model~\cite{osher2005iterative} in image domain i.e.,
\begin{equation}\label{eq:tv}
	\begin{array}{c}
		\mathcal{G}(\mathbf{x})\in\arg\min\limits_{\mathbf{x}}\frac{1}{2}\|\mathbf{K}\mathbf{x}-\mathbf{b}\|^2+\lambda_2\!\sum\limits_{j\in\{h,v\}}\|\nabla_j\mathbf{x}\|_{q},
	\end{array}
\end{equation}
where $\lambda_2$ is threshold parameter, and $q\in [0,1]$. $\nabla_h$ and $\nabla_v$ respectively denote the gradient on the horizontal and vertical directions. As it is flexible to select operator for solving Eq.~\eqref{eq:tv}, here we indeed apply splitting scheme to update $\mathbf{x}_{\mathcal{G}}^k$. By introducing two auxiliary variables (named as $\mathbf{z}_h$ and $\mathbf{z}_v$), the variable $\mathbf{x}^k_{\mathcal{G}}$ can be updated by 
\begin{equation*}
	\mathbf{x}_{\mathcal{G}}^k=\arg\min\limits_{\mathbf{x}}\frac{1}{2}\|\mathbf{K}\mathbf{x}-\mathbf{b}\|^2+\sum\limits_{j\in\{h,v\}}\rho_j\|\mathbf{z}^k_j-\nabla_j\mathbf{x}\|^2,
\end{equation*}
where $\rho_h>0$ and $\rho_v>0$ are two constant parameters. $\mathbf{z}^k_h$ and $\mathbf{z}^k_v$ are updated by two proximal gradient operators respectively and we omit them here. In addition, applying proximal gradient approach to update $\mathbf{x}_{\mathcal{F}}^k$, which can be transformed as $\mathbf{x}_{\mathcal{F}}^k=\mathbf{W}^{\top}\bm{\beta}^k_{\mathcal{F}}$, yields the following form
\begin{equation*}
\bm{\beta}^k_{\mathcal{F}}\in\mathtt{prox}_{s_k \|\bm{\beta}\|_p}\left(\bm{\beta}^k\!-\!s_k(\mathbf{W}\mathbf{K}^{\top}\mathbf{K}\mathbf{W}^{\top}\bm{\beta}^k\!-\!\mathbf{W}\mathbf{K}^{\top}\mathbf{b})\right),
\end{equation*}
where $\bm{\beta}^k = \mathbf{W}\mathbf{x}^k$. It is clearly to obtain the detailed updating steps following Alg.~\ref{alg:TLF}.

\begin{figure}[t]
	\centering
	\begin{tabular}{c}
		\includegraphics[height=0.185\textwidth]{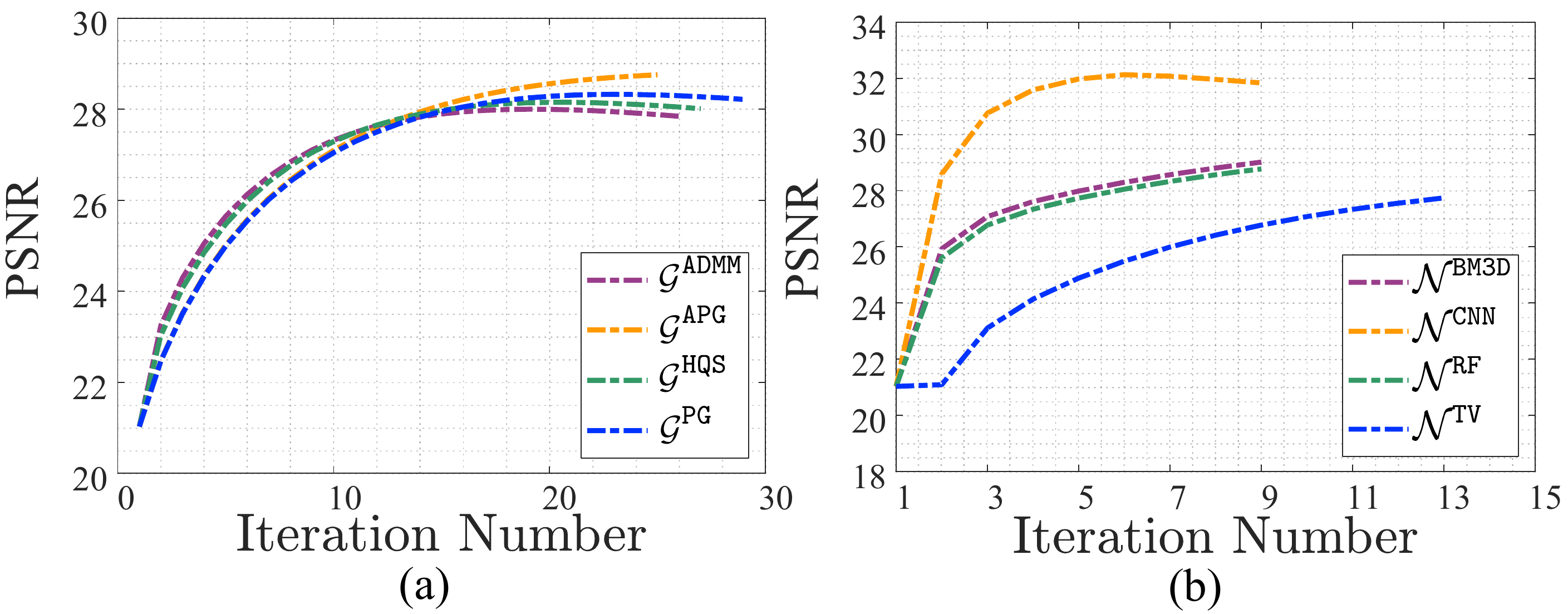}\\
	\end{tabular}
	\caption{Comparing different modularization settings of TLF with additional Gaussian noise level $1\%$. Subfigure (a) plots PSNR of TLF with different first-order methods when updating $\mathbf{x}_{\mathcal{G}}^k$. Subfigure (b) shows the PSNR results of DTLF with different data-ensemble structures, denoted by superscripts BM3D, CNN, RF and TV.} 	
	\label{fig:DiffModularization}
\end{figure}

\begin{figure}[t]
	\centering
	\begin{tabular}{c@{\extracolsep{0.01em}}c}
		\includegraphics[height=0.19\textwidth]{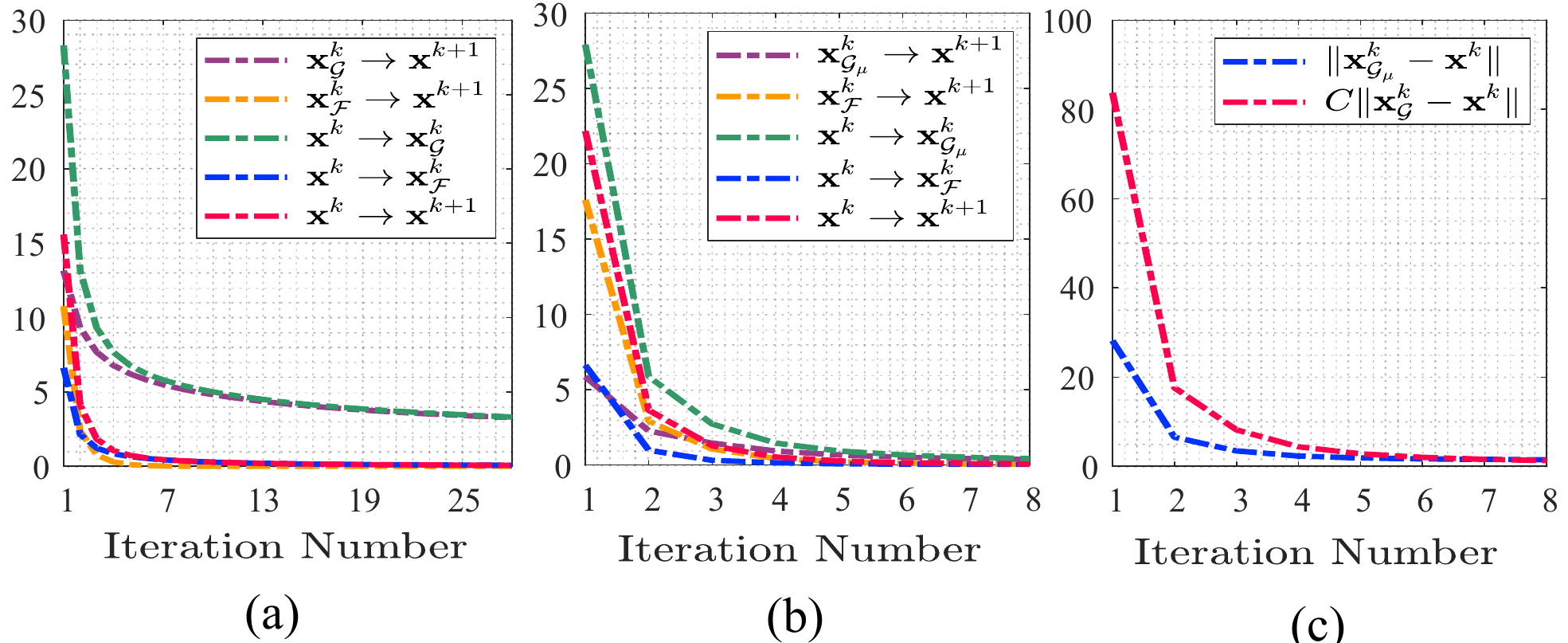}&
	\end{tabular}
	\caption{Illustration the iteration behaviors of TLF in different settings with additional Gaussian noise level 1\textperthousand. Subfigures (a) and (b) plot the variation about each updating of TLF and DTLF respectively. The legend in subfigures (a) and (b), i.e., ``$\text{left}\rightarrow\text{right}$'', means $\|\text{left}-\text{right}\|$. Subfigure (c) plots the error control condition of DTLF.}\label{fig:DiffError}
\end{figure}

\begin{figure}[t]
	\centering
	\begin{tabular}{c@{\extracolsep{0.01em}}}
		\includegraphics[height=0.36\textwidth]{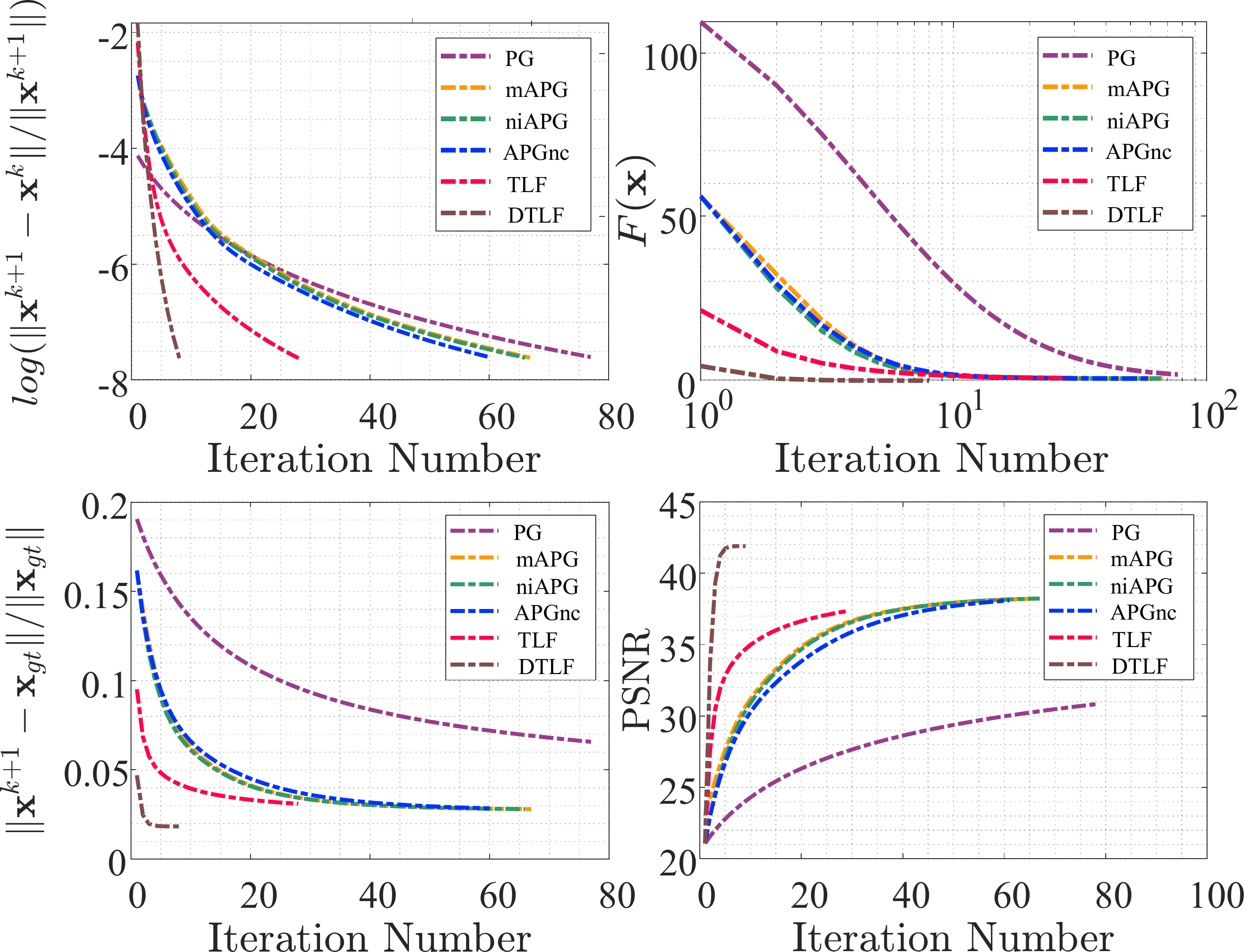}
	\end{tabular}
	\caption{Comparing iteration behaviors of TLF and DTLF to classical first-order methods, including exact ones (PG, mAPG), inexact APG (niAPG) and APGnc with additional Gaussian noise level 1\textperthousand.}
	\label{fig:CurvesCompPGs}
\end{figure}

\subsubsection{DTLF Strategy}
Specifically, in this work, we introduce a residual network $\mathcal{N}(\cdot)$ to our task-driven module with $T$-th training stage. For training strategy, we just adopt a similar method~\cite{zhang2017learning} to train the network. The detailed information about $\mathcal{N}$ can be found in the following section. Notice that standard training strategies can be directly adopted to optimize our basic architecture parameters. If necessary, one may further jointly fine-tune parameters of the whole network after the design phase. By setting $\tilde{\mathbf{x}}^k=\mathcal{N}(\mathbf{x}^k;\bm{\mathcal{W}}_T^k)$, the data-driven scheme, i.e., $\mathcal{G}^{\mu}(\mathbf{x},\tilde{\mathbf{x}}^k)$, can be defined as
\begin{equation*}
\begin{array}{l}
\mathcal{G}^{\mu}(\mathbf{x},\tilde{\mathbf{x}}^k)\in\\
\arg\min\limits_{\mathbf{x}}\frac{1}{2}\|\mathbf{K}\mathbf{x}\!-\!\mathbf{b}\|^2\!+\!\lambda_2\!\sum\limits_{j\in\{h,v\}}\|\nabla_j\mathbf{x}\|_{p}
+\frac{\mu^k}{2}\|\mathbf{x}\!-\!\tilde{\mathbf{x}}^k\|^2.
\end{array}
\end{equation*} 
Hence, following the iteration form of $\mathbf{x}^k_{\mathcal{G}}$ and the above data-driven scheme, we obtain that 
\begin{equation*}
\mathbf{x}^k_{\mathcal{G}_{\mu}} \!=\! \arg\min\limits_{\mathbf{x}}\frac{1}{2}\|\mathbf{K}\mathbf{x}\!-\!\mathbf{b}\|^2\!+\!\sum\limits_{j\!\in\!\{h,v\}}\rho_j\|\mathbf{z}^k_j\!-\!\nabla_j\mathbf{x}\|^2\!+\!\frac{\mu^k}{2}\|\mathbf{x}\!-\!\tilde{\mathbf{x}}^k\|^2.
\end{equation*}
Subsequently, it is easy to obtain the detailed iteration steps following Alg.~\ref{alg:DTLF}.

\subsection{Rain Streaks Removal}
This subsection focuses on single image rain streaks removal task, which is a challenging real-world computer vision problem. A rainy image $\mathbf{y}$ is often considered as linear combination of rain-free background $\mathbf{x}_{b}$ and rain streaks layer $\mathbf{x}_r$, i.e., $\mathbf{y}=\mathbf{x}_{b}+\mathbf{x}_r$. We set $\mathbf{x}:=[\mathbf{x}_b;\mathbf{x}_r]$. As for designing the minimized optimization model, we tend to perform fundamental energy-based sparsity of the observation with a certain transformed domain in objective function which can be formulated as 
\begin{equation*}\label{eq:derain}
\min\limits_{\mathbf{x},\bm{\beta},\bm{\gamma}}f(\mathbf{x}_b,\mathbf{x}_r;\bm{\beta},\bm{\gamma})+\psi_{\bm{\beta}}(\bm{\beta})+\psi_{\bm{\gamma}}(\bm{\gamma})+\chi_{[0,1]}(\mathbf{x}_b,\mathbf{x}_r),
\end{equation*}
where $f(\mathbf{x}_b,\mathbf{x}_r;\bm{\beta},\bm{\gamma})=\frac{1}{2}\|\mathbf{x}_{b}-\mathbf{D}\bm{\beta}\|^2+\frac{1}{2}\|\mathbf{x}_r-\mathbf{D}\bm{\gamma}\|^2$, $\psi_{\bm{\beta}}(\bm{\beta})=\nu_1\|\bm{\beta}\|_{p_1}$,  $\psi_{\bm{\gamma}}(\bm{\gamma})=\nu_2\|\bm{\gamma}\|_{p_2}$, $p_1,\ p_2\in[0,1]$ and $\nu_1,\ \nu_2$ are two positive constants. $\chi_{[0,1]}$ denotes indicator function, i.e., if $\mathbf{x}_b$, $\mathbf{x}_r\in[0,1]$ then $\chi_{[0,1]}(\mathbf{x}_b,\mathbf{x}_r)=0$, otherwise $\chi_{[0,1]}(\mathbf{x}_b,\mathbf{x}_r)=\infty$. $\bm{\beta}$ and $\bm{\gamma}$ are two auxiliary variables serving for this objective subproblem, and respectively denote the sparse codes of $\mathbf{x}_b$, $\mathbf{x}_r$ on $\mathbf{D}$. As for $\mathcal{G}$ stated in Eq.~\eqref{eq:model}, we consider the general TV regularization as the following form
\begin{equation*}
\mathcal{G}\in\arg\min\limits_{\mathbf{x}_b,\mathbf{x}_r}\frac{1}{2}\|\mathbf{y}\!-\!\mathbf{x}_b\!-\!\mathbf{x}_r\|^2\!+\!\rho_1\!\sum_{j\in\{h,v\}}\!\|\nabla\mathbf{x}_b\|_{p_1}\!+\!\rho_2\|\mathbf{x}_r\|_{p_2}. 
\end{equation*} 
In this part, we first introduce two residual network $\mathcal{N}_b$ and $\mathcal{N}_r$. Then we denote two temporary variable $\tilde{\mathbf{x}}_b^k=\mathcal{N}_b(\mathbf{x}_b^k;\bm{\mathcal{W}}_{T,b}^k)$  and $\tilde{\mathbf{x}}_r^k=\mathcal{N}_r(\mathbf{x}_r^k;\bm{\mathcal{W}}_{T,r}^k)$ respectively for background and rain streaks layers. For background layer network $\mathcal{N}_b$, we just follow the aforementioned strategy to build a series of denoising CNNs to extract natural image well. For rain streaks layer, $\mathcal{N}_r$ learns rain streaks behavior from rainy images by training rainy image and synthetic rain streaks layer as degraded and clean image pair. We update variables $\mathbf{x}^k_{\mathcal{F},b}$ and $\mathbf{x}^k_{\mathcal{F},r}$ synchronously as

\begin{equation*}
\left[\!\begin{array}{c}
\mathbf{x}_{\mathcal{F},b}^k\\
\mathbf{x}_{\mathcal{F},r}^k
\end{array}\!\right]
\!=\!\left[\!
\begin{array}{c}
\arg\min_{\mathbf{x}_b} \frac{1}{2}\|\mathbf{x}_b\!-\!\mathbf{D}\bm{\beta}^k\|^2\!+\!\chi_{[0,1]}(\mathbf{x}_b^k,\mathbf{x}_r^k)\\
\arg\min_{\mathbf{x}_r} \frac{1}{2}\|\mathbf{x}_r\!-\!\mathbf{D}\bm{\gamma}^k\|^2\!+\!\chi_{[0,1]}(\mathbf{x}_b^k,\mathbf{x}_r^k)
\end{array}
\!\right]\!,
\end{equation*}
where the auxiliary variables $\bm{\beta}^k$ and $\bm{\gamma}^k$ are updated by proximal gradient operator. Similarly, we have that
\begin{equation*}
\left[\!\begin{array}{c}
\mathbf{x}_{\mathcal{G_{\mu}},b}^k\\
\mathbf{x}_{\mathcal{G_{\mu}},r}^k
\end{array}\!\right]
\!=\!\left[\!
\begin{array}{c}
\arg\min_{\mathbf{x}_b} G(\mathbf{x}_b,\mathbf{x}_r^k)+\frac{\eta_1^k}{2}\|\mathbf{x}_b-\tilde{\mathbf{x}}_b^k\|^2\\
\arg\min_{\mathbf{x}_r} G(\mathbf{x}_b^k,\mathbf{x}_r)+\frac{\eta_2^k}{2}\|\mathbf{x}_r-\tilde{\mathbf{x}}_r^k\|^2
\end{array}
\!\right]\!.
\end{equation*}
Then, following DTLF iterations, we could obtain detailed updating scheme.

Typically, TLF could integrate different domain knowledge to address a broad variety of vision applications, including deblurring, inpainting and rain streaks removal, etc. Here, the matrix $\mathbf{K}$ actually formulates the observation forward model for particular image processing paradigm. Possible choices of $\mathbf{K}$ include an identity operator for denoising, convolution operators for deblurring, filtered subsampling operators for superresolution, the Fourier $k$-domain subsampling operator for magnetic resonance imaging (MRI) reconstruction or mask for image inpainting. We incorporate experimentally designed and trained network architectures into TLF to solve these problems. In summary, the proposed TLF indeed could integrate advantages from different domain knowledge. 

\begin{figure}[t]
	\begin{tabular}{c@{\extracolsep{0.04em}}c@{\extracolsep{0.04em}}c}
		\includegraphics[width=0.16\textwidth]{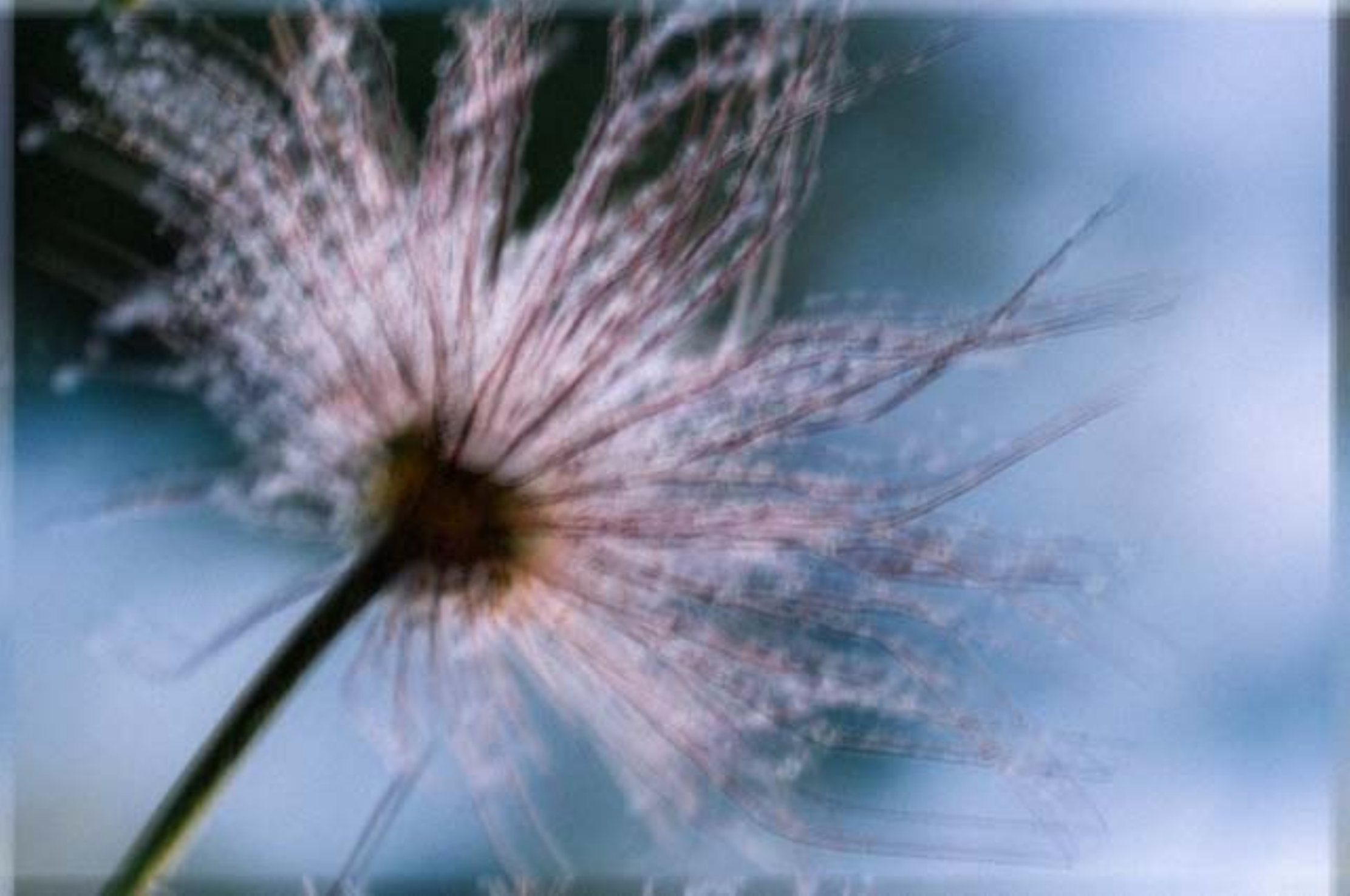}&
		\includegraphics[width=0.16\textwidth]{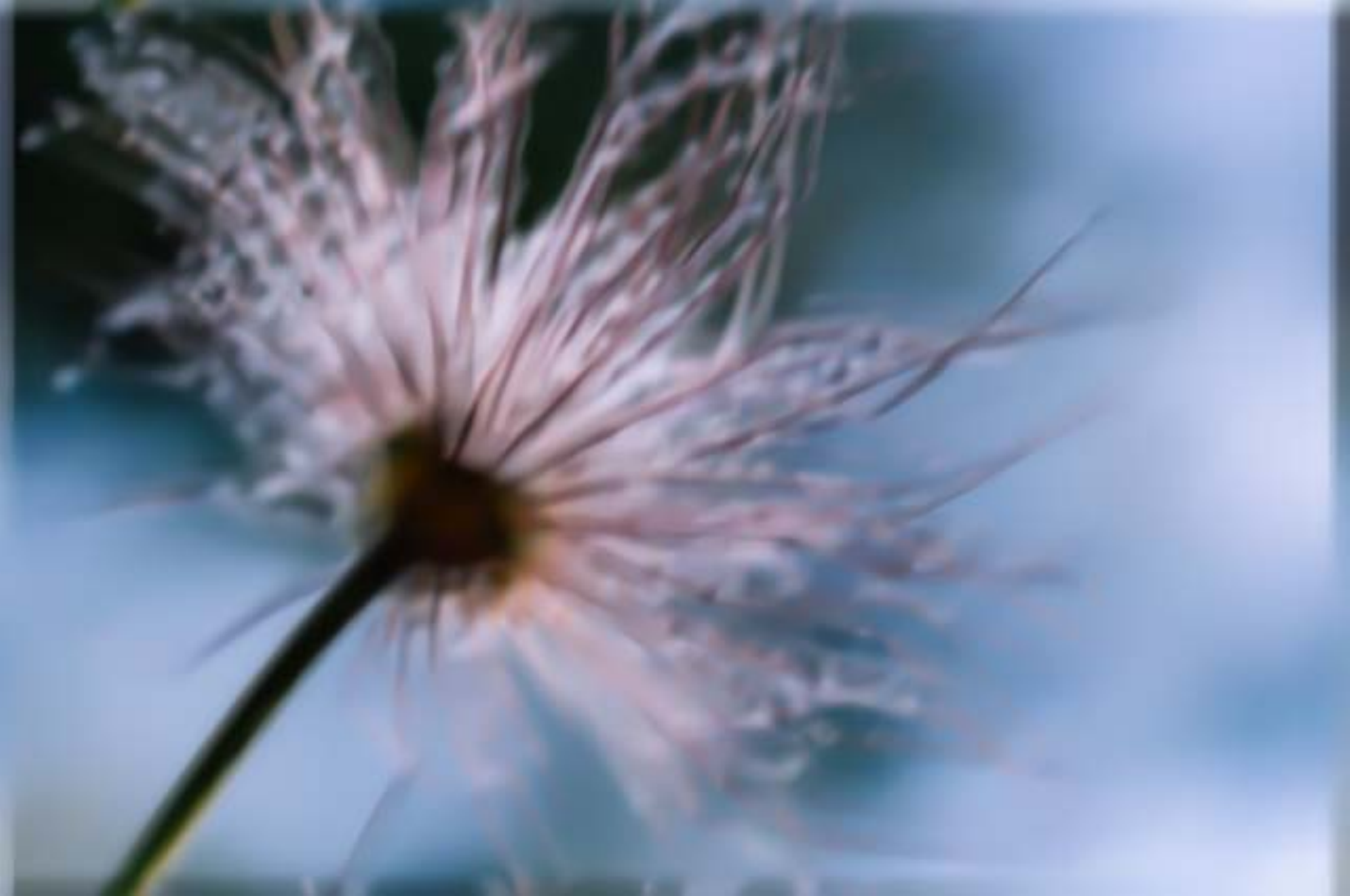}&
		\includegraphics[width=0.16\textwidth]{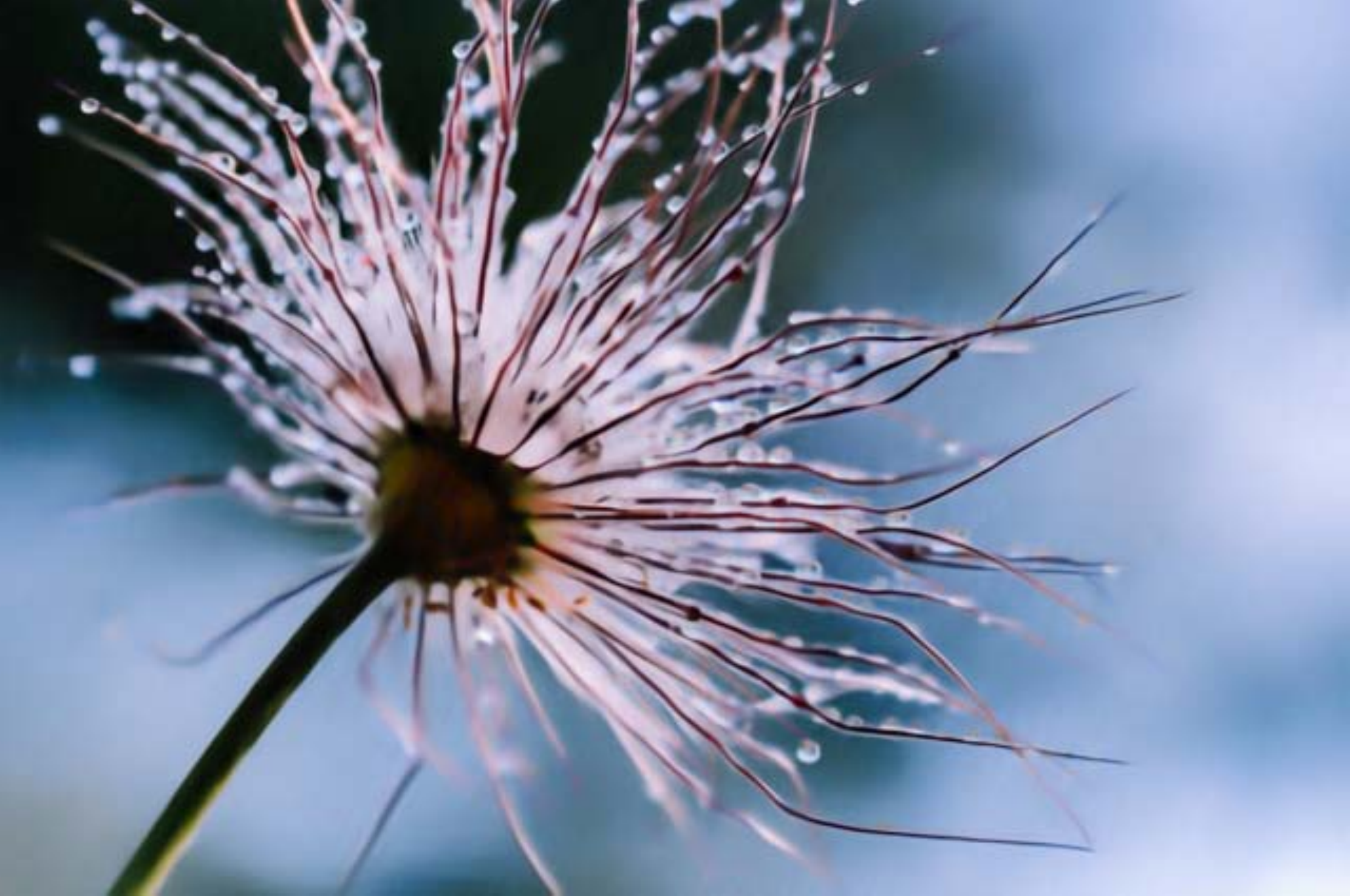}\\
	\footnotesize -- &\footnotesize 17.75 / 0.5449 &\footnotesize 27.47 / 0.8944\\
		\includegraphics[width=0.16\textwidth]{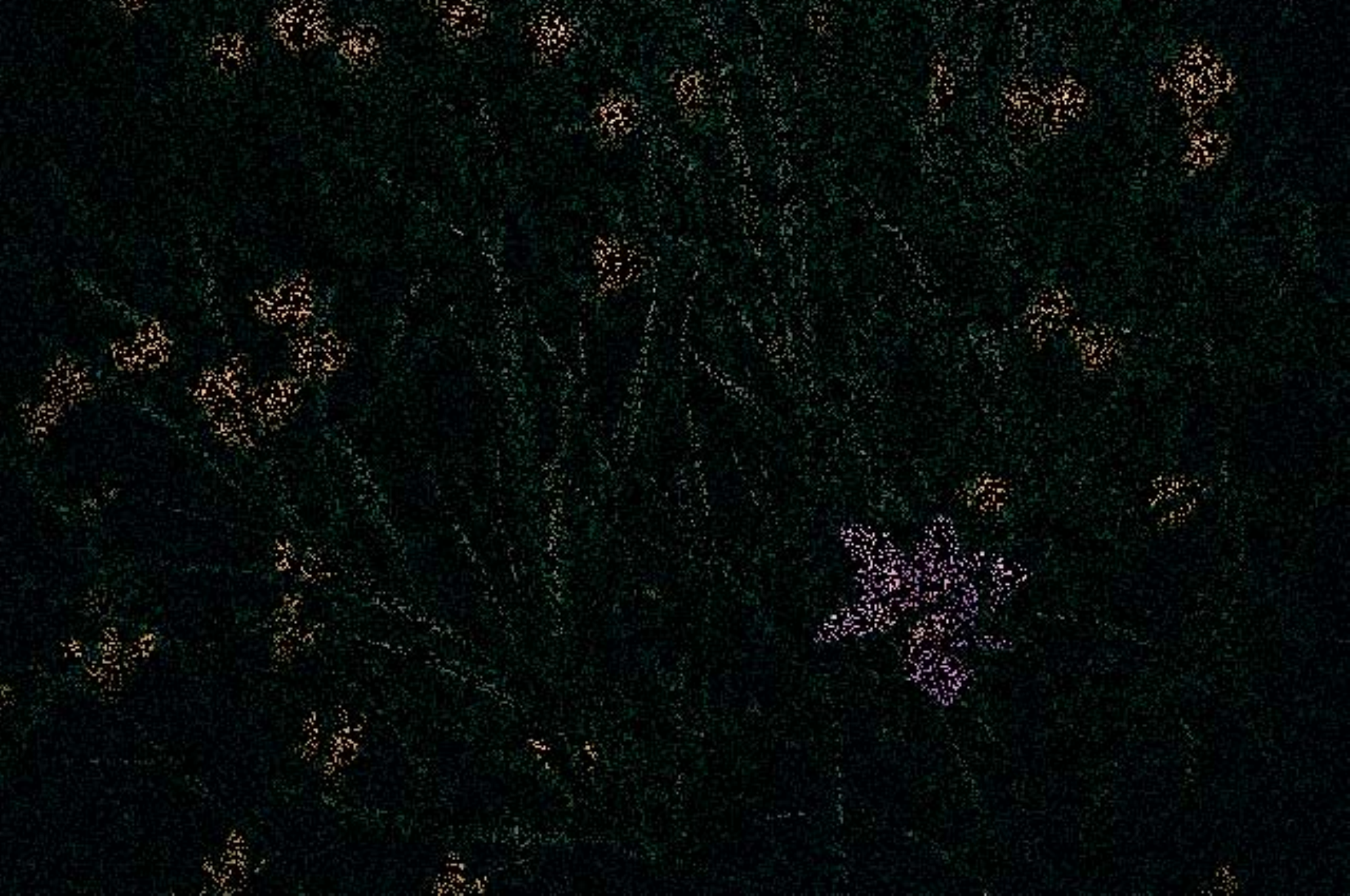}&
		\includegraphics[width=0.16\textwidth]{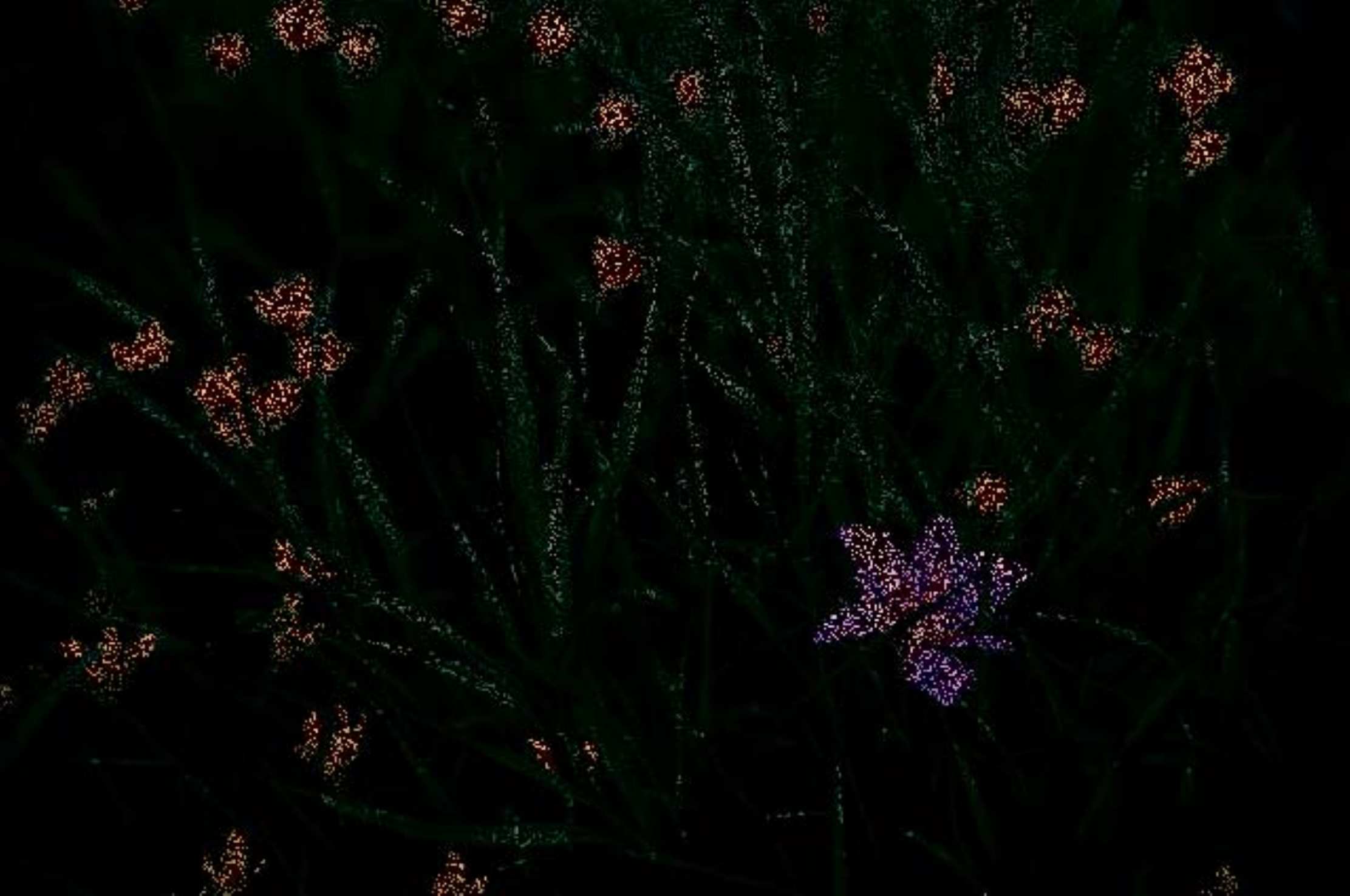}&
		\includegraphics[width=0.16\textwidth]{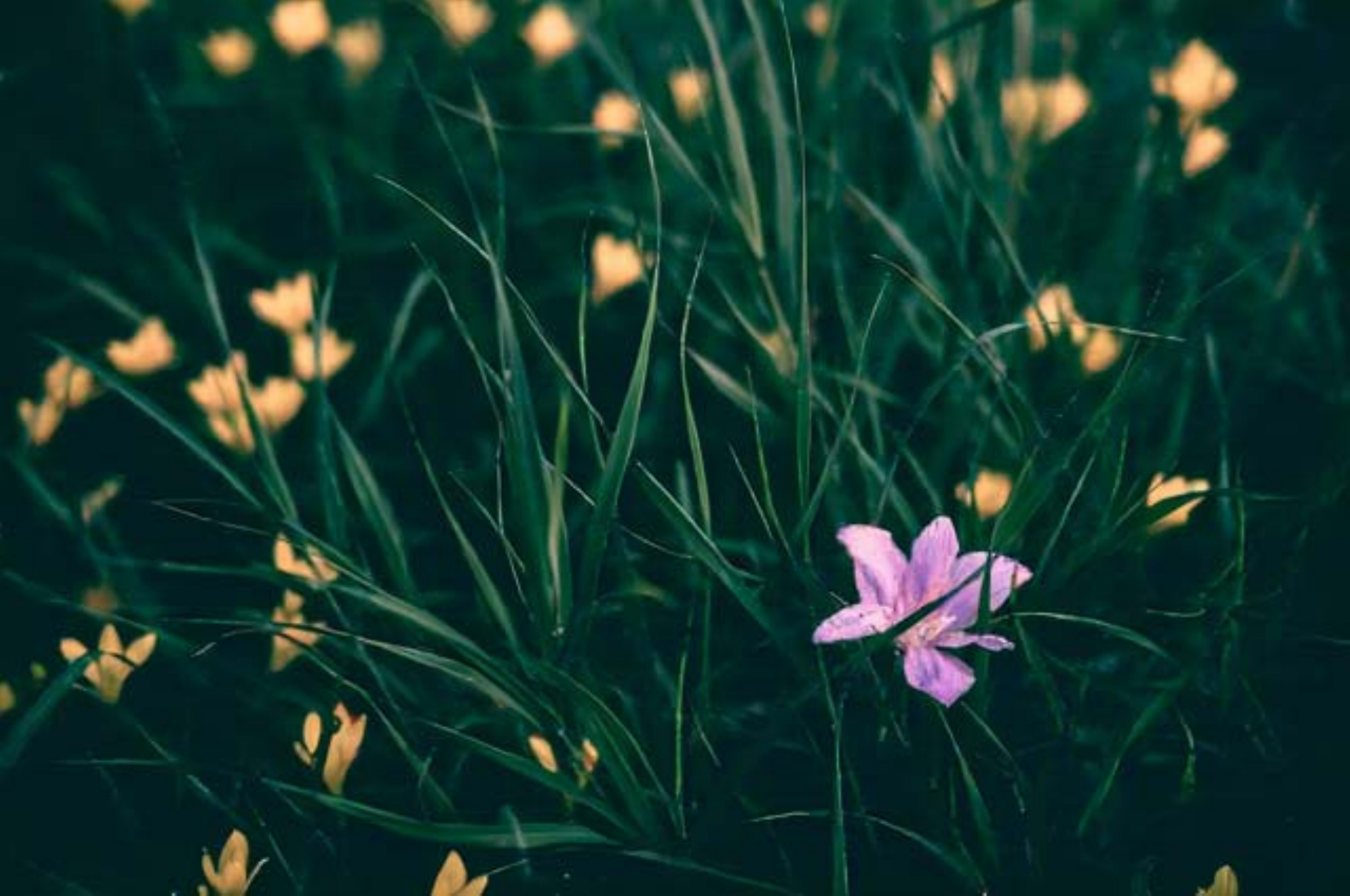}\\
		\footnotesize -- & \footnotesize 15.95 / 0.5497
		&\footnotesize 35.43 / 0.9591\\
		\footnotesize Input & \footnotesize NC-CNNs & \footnotesize Ours
	\end{tabular}
	\caption{Visual comparison between NC-CNNs and our TLF on image deblurring (the top row) and inpainting (the bottom row).} 
	\label{fig:comp_cnns}
\end{figure}

\begin{table}[t]
	\renewcommand{\arraystretch}{1.1}
	\centering
	\caption{Averaged PSNR and SSIM on the Benchmark Image Set \cite{schmidt2014shrinkage}. The first row $\sigma$ represents the Gaussian noise level. The first column is the comparison with traditional methods.}
	\setlength{\tabcolsep}{1.8mm}{
		\begin{tabular}{c|cc|cc|cc}
			\hline
			\multirow{2}*{Methods}& \multicolumn{2}{c|}{1$\%$}&\multicolumn{2}{c|}{2$\%$} &\multicolumn{2}{c}{3$\%$}\\
			\cline{2-7}
			& PSNR & SSIM & PSNR & SSIM & PSNR & SSIM\\
			\hline
			APG & 27.32 & 0.71 & 25.61 & 0.63 & 24.63 & 0.57\\
			mAPG & 26.68 & 0.67 & 25.20 & 0.60 & 24.39 & 0.55\\
			niAPG & 27.24 & 0.73 & 25.63 & 0.64 & 24.76 & 0.61\\
			FTVD & 27.56 & 0.77 & 26.63 & 0.73 & 24.88 & 0.62\\
			Ours & \textbf{28.48} & \textbf{0.81} & \textbf{27.06} & \textbf{0.75} &\textbf{26.13} & \textbf{0.71}\\
			\hline
	\end{tabular}}
	\label{tab:deblur_tradional}
\end{table}

\begin{figure*}[t]
	\begin{tabular}{c@{\extracolsep{0.2em}}c@{\extracolsep{0.2em}}c@{\extracolsep{0.2em}}c@{\extracolsep{0.2em}}c@{\extracolsep{0.2em}}c}
		\includegraphics[width=0.16\textwidth]{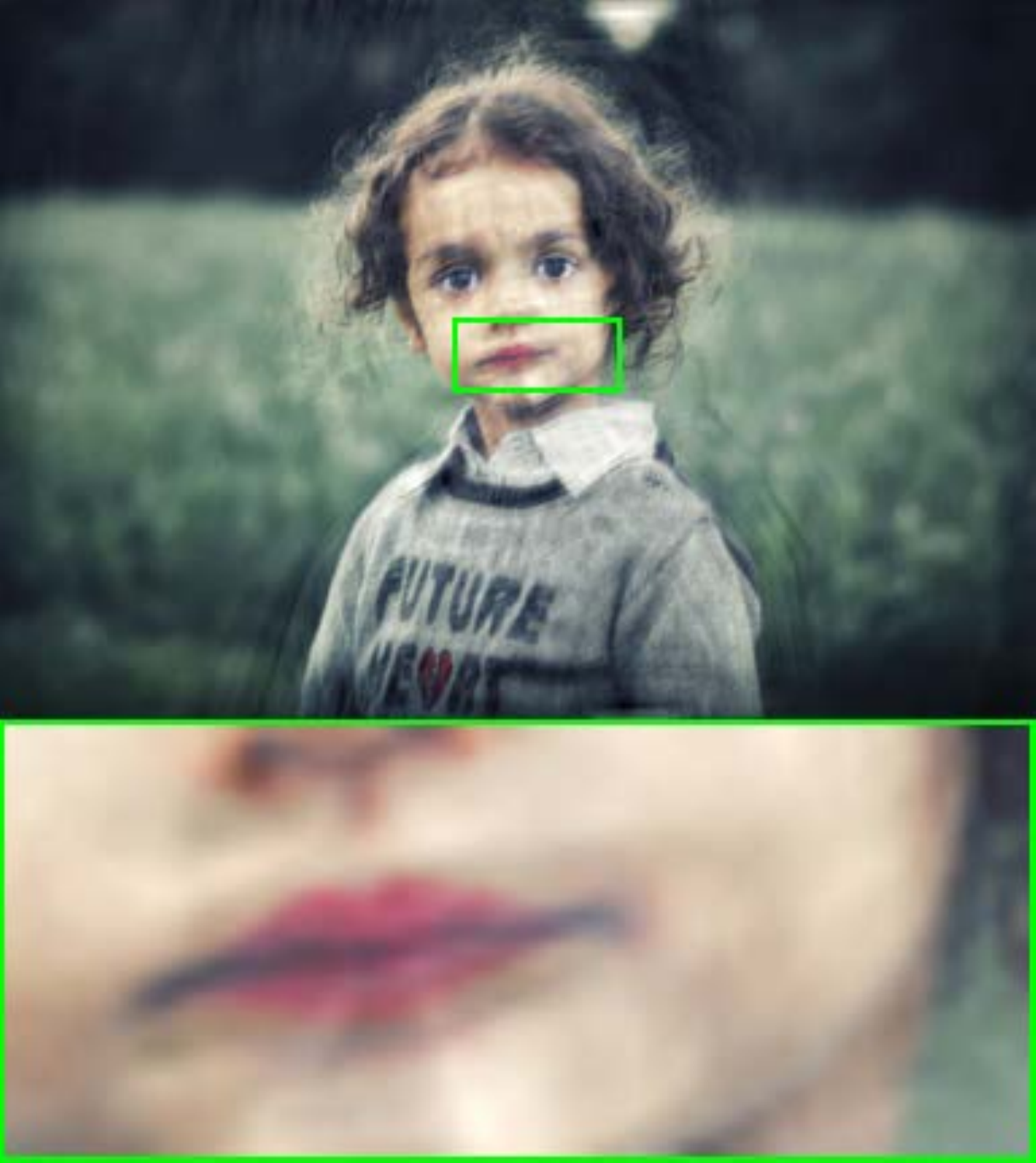}&
		\includegraphics[width=0.16\textwidth]{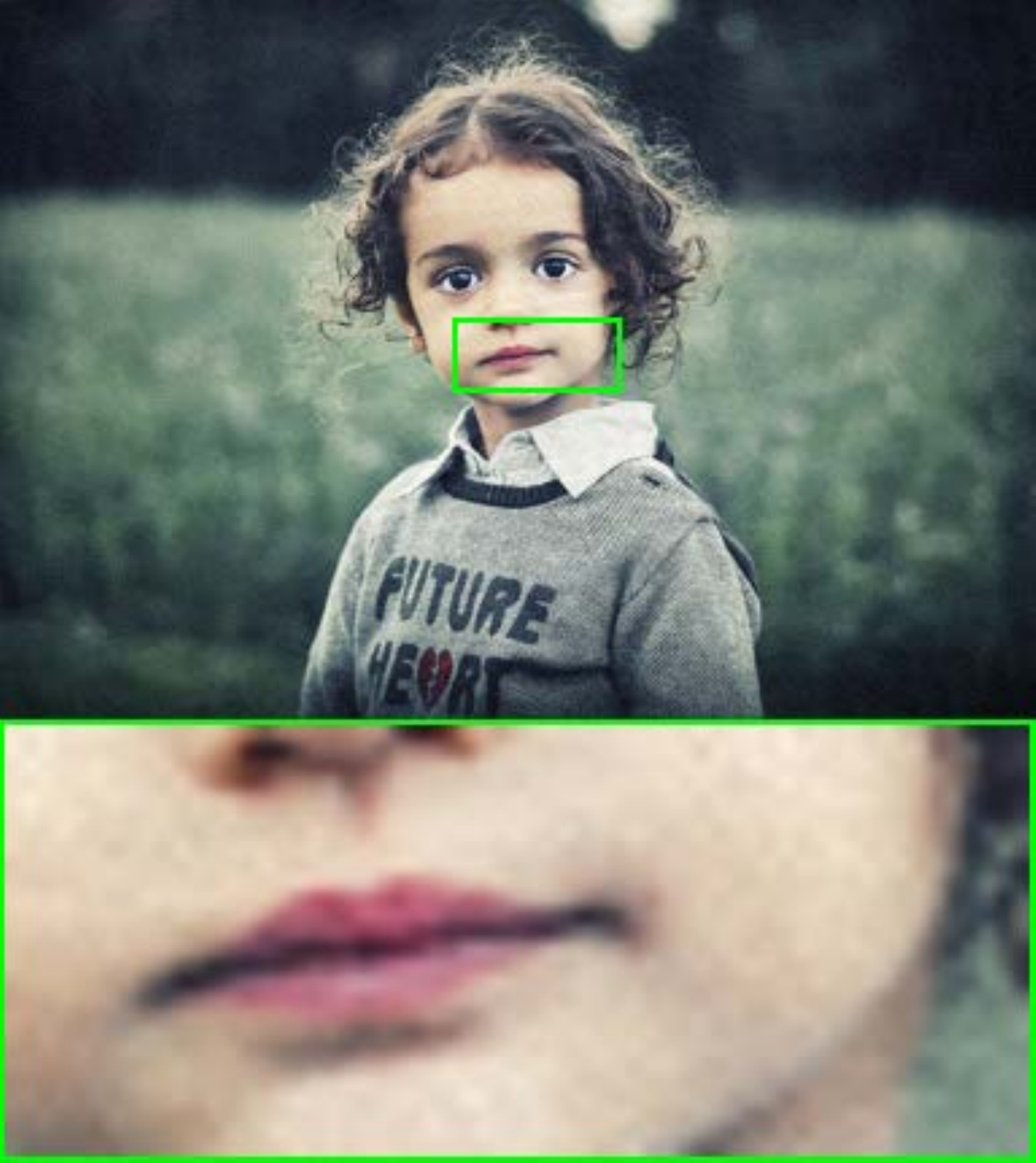}&
		\includegraphics[width=0.16\textwidth]{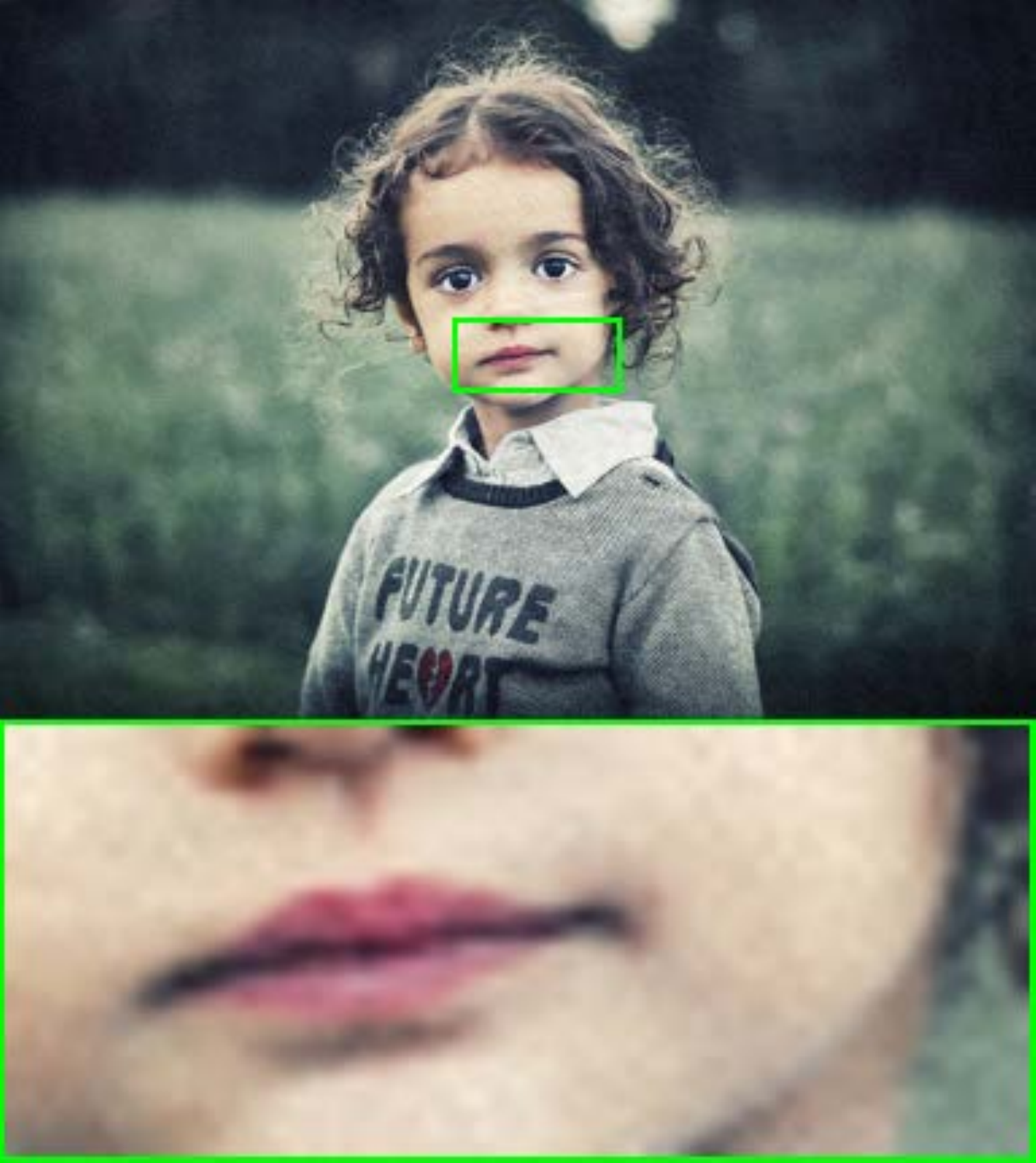}&
		\includegraphics[width=0.16\textwidth]{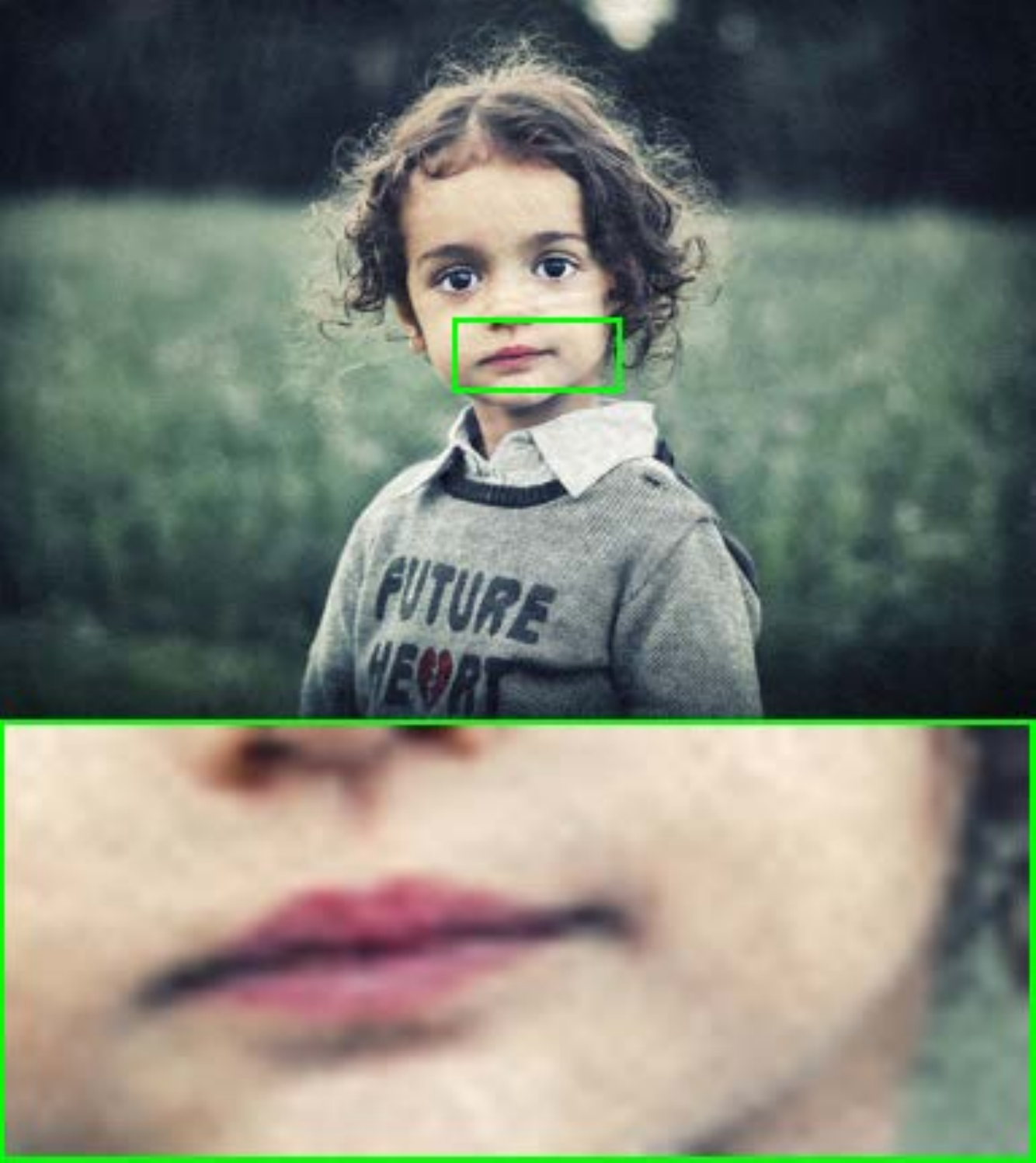}&
		\includegraphics[width=0.16\textwidth]{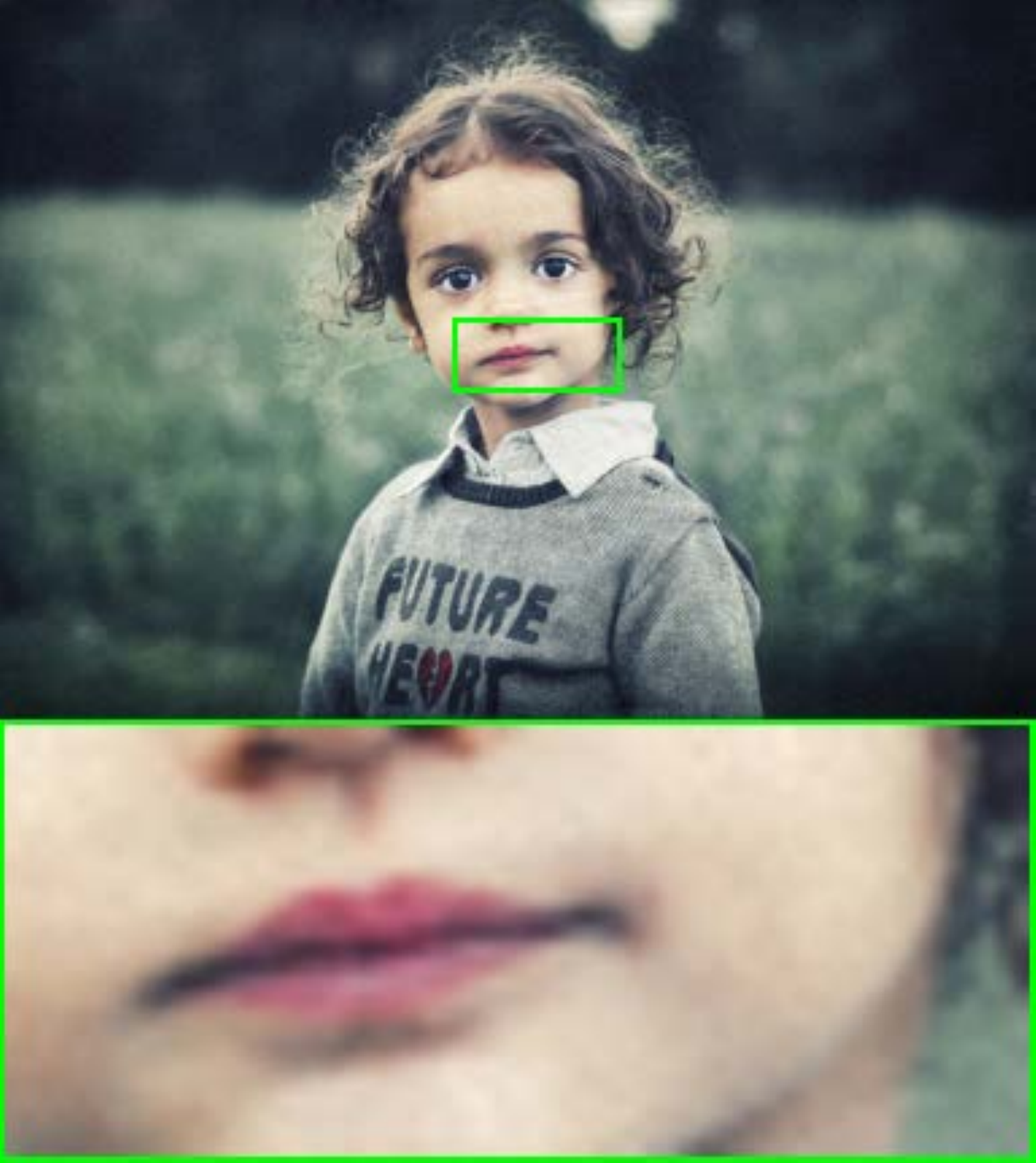}&
		\includegraphics[width=0.16\textwidth]{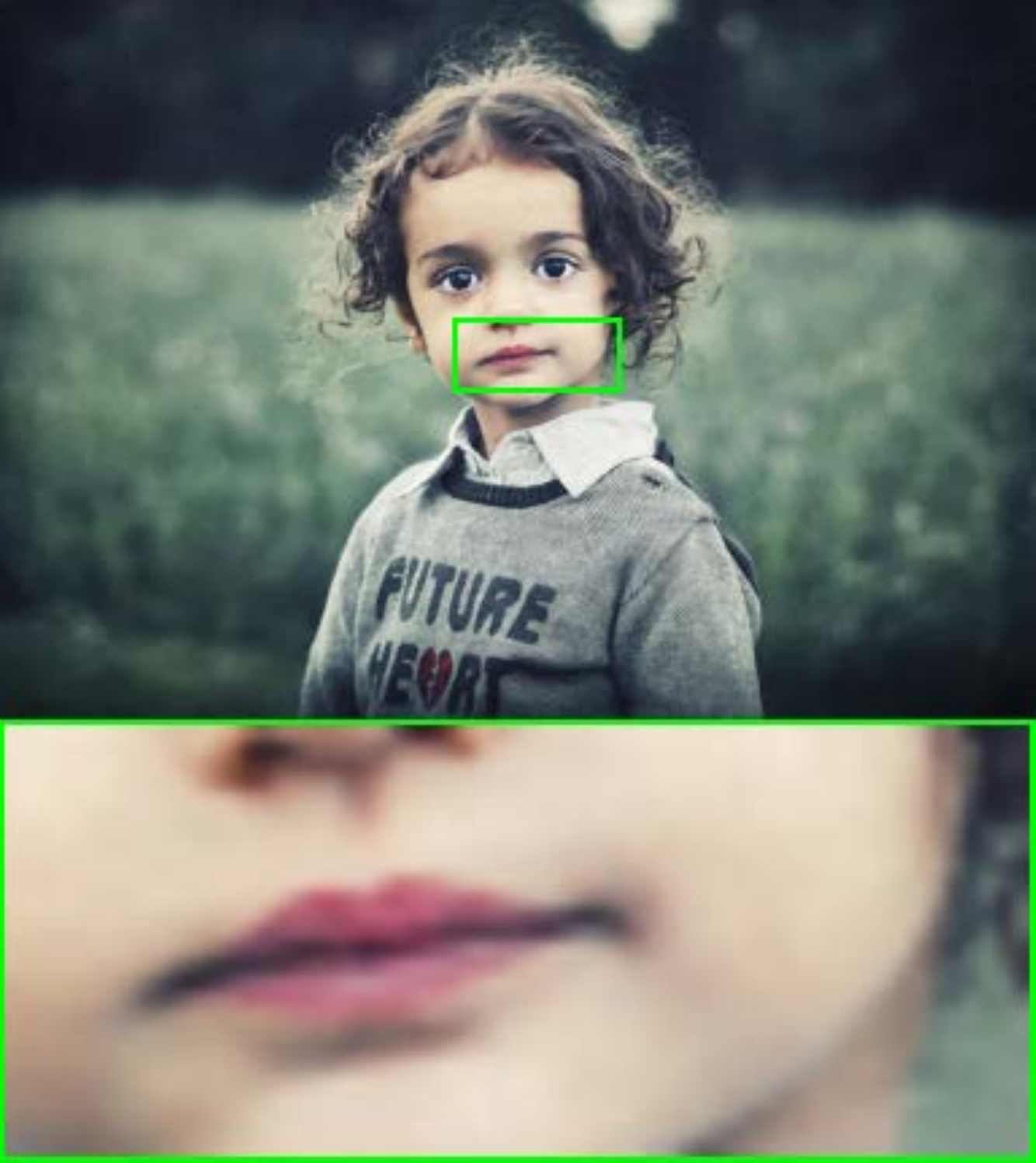}\\
		\footnotesize 30.8229 / 0.9108 &\footnotesize 38.2246 / 0.9247  &\footnotesize 38.2269 / 0.9251 &\footnotesize 38.0961 / 0.9324 & \footnotesize 38.3633 / 0.9581 & \footnotesize 40.4552 / 0.9802\\
		\footnotesize PG &\footnotesize mAPG  &\footnotesize niAPG &\footnotesize APGnc & \footnotesize TLF (Ours) & \footnotesize DTLF (Ours)
	\end{tabular}
	\caption{The non-blind image deconvolution performances of the proposed TLF and DTLF scheme with compared proximal-based first-order methods (PG, mAPG, niAPG and APGnc). The quantitative scores (PSNR / SSIM) are marked blow each image.}
	\label{fig:VisualDeblurPGs}
\end{figure*}

\section{Experimental Results}\label{sec:app-result}

In this section, we first verify our theoretical results by investigating the iteration behaviors of the proposed TLF and DTLF on standard non-blind deblurring formulation with Eq.~\eqref{eq:ISTA}. We then compared the performance of DTLF with state-of-the-art methods (both general and learning-based approaches) on different vision applications. We conducted these experiments on a computer with Intel Core i7-7700 CPU (3.6 GHz), 32GB RAM, and an NVIDIA GeForce GTX 1060 6GB GPU. All the comparisons shown in this paper are conducted under the same hardware configuration. 

\subsection{Theoretical Verifications}\label{sec:theore_verifi}
To verify our theoretical investigations, we performed experiments on non-blind deblurring. Notice that this problem can be directly addressed by our TLF and DTLF.

\subsubsection{Modularization Settings}  We first provided a comparison among different optimization models in image deblurring application, and the corresponding results are shown in Fig.~\ref{fig:CompOur_deblur}. Observed that our TLF with a data-driven module performs the best when comparing with single modeling schemes described in Eq.~\eqref{eq:model_orginal} (i.e., only with the objective subproblem) and Eq.~\eqref{eq:DF} (i.e., only with the data-driven module). The experimental results illustrate the effectiveness of our TLF.

To analyze the flexibility of $\mathcal{G}$, we then investigated the performance of TLF with different operator settings and the corresponding PSNR (peak signal-to-noise ratio) results with $1\%$ noise level are plotted in Fig.~\ref{fig:DiffModularization} (a). As for solving $\mathcal{G}$ module specified in Eq.~\eqref{eq:tv}, four different first-order methods, such as PG ($\mathcal{G}^{\mathtt{PG}}$), APG ($\mathcal{G}^{\mathtt{APG}}$), HQS ($\mathcal{G}^{\mathtt{HQS}}$) and ADMM ($\mathcal{G}^{\mathtt{ADMM}}$) are considered (as mentioned in subsection~\ref{subsec:knowledge}). As can be seen in Fig.~\ref{fig:DiffModularization} (a) that various methods when obtaining $\mathbf{x}_{\mathcal{G}}^k$ have a slight influence on the performance of our TLF scheme. We adopt HQS as the approach to obtain the iteration steps of $\mathbf{x}_{\mathcal{G}}^k$ in TLF and $\mathbf{x}_{\mathcal{G}_{\mu}}^k$ in DTLF hereafter. In fact, to provide a relatively fair comparison, we keep parameters the same under four different circumstances mentioned above. Hereafter, we select relative error (i.e., $\|\mathbf{x}^{k+1}-\mathbf{x}^k\|/\|\mathbf{x}^{k+1}\|$) as a stop criterion.

To further explore the effectiveness of network-based block $\mathcal{N}$, four different task-specific structures, i.e., TV~\cite{osher2005iterative}, RF~\cite{unser1991recursive}, CNNs and BM3D~\cite{dabov2007image} (named as $\mathcal{N}^{\mathtt{TV}}$, $\mathcal{N}^{\mathtt{RF}}$, $\mathcal{N}^{\mathtt{CNN}}$ and $\mathcal{N}^{\mathtt{BM3D}}$, respectively) are adopted under DTLF scheme. 
For CNNs architecture, the introduced residual network consists of nineteen layers as described in the paper, i.e., seven dilated convolutions with $3 \times 3$ filter size, six ReLu operations (plugged between each two convolution layers) and five batch normalizations (plugged between convolution and ReLU, expect the first convolution layer). In training strategy, similar to~\cite{zhang2017learning}, we randomly select 800 natural images from ImageNet database~\cite{russakovsky2015imagenet} to train different Gaussian noise levels with a standard deviation of 0.1 that meet the condition in each iteration. The learning rate is started with 0.001 and decayed by multiplying 0.1 at 30, 60 and 80 epochs. We use ADAM with a weight decay of 0.0001 to train the network with a MSE loss. Fig.~\ref{fig:DiffModularization} (b) plotted the PSNR with $\mathcal{N}^{\mathtt{TV}}$, $\mathcal{N}^{\mathtt{RF}}$, $\mathcal{N}^{\mathtt{CNN}}$ and $\mathcal{N}^{\mathtt{BM3D}}$. As can be seen, DTLF performs better and faster with $\mathcal{N}^{\mathtt{CNN}}$ than others. Hence, we set network-based building block $\mathcal{N}$ as CNNs hereafter.

To compare our TLF with classical first-order methods, we evaluated the performance of the proposed method and APG \cite{li2015accelerated}, monotone APG (mAPG) \cite{li2015accelerated}, inexact APG (niAPG)~\cite{yao2016efficient} and FTVD~\cite{wang2008new}, under three different additional Gaussian noise levels (i.e., $1\%$, $2\%$ and $3\%$) on the image set collected by~\cite{schmidt2014shrinkage}. The corresponding results are shown in Tab.~\ref{tab:deblur_tradional} with quantitative performance. It can be seen that our TLF outperforms classical numerical solvers by a large margin in terms of the performance.

We further conducted an ablation experiment to compare with these pre-trained CNNs. We must clarify that in most of our considered applications (e.g., deblurring and inpainting), naively cascading these pre-trained CNNs do not work well. As can be seen in Fig.~\ref{fig:comp_cnns} that Naively Cascaded CNNs (NC-CNNs) cannot properly fit the degeneration of image blurs/masks. Notice that both the proposed method and NC-CNNs share the same pre-trained architectures. While these two strategies have completely different final results. This is mainly because in our framework, energy-model based calculations can roughly remove the degeneration while these CNN architectures are used to refine image details and remove artifacts.

\begin{figure*}[t]
	\begin{tabular}{c@{\extracolsep{0.2em}}c@{\extracolsep{0.2em}}c@{\extracolsep{0.2em}}c@{\extracolsep{0.2em}}c@{\extracolsep{0.2em}}c}
		\includegraphics[width=0.16\textwidth]{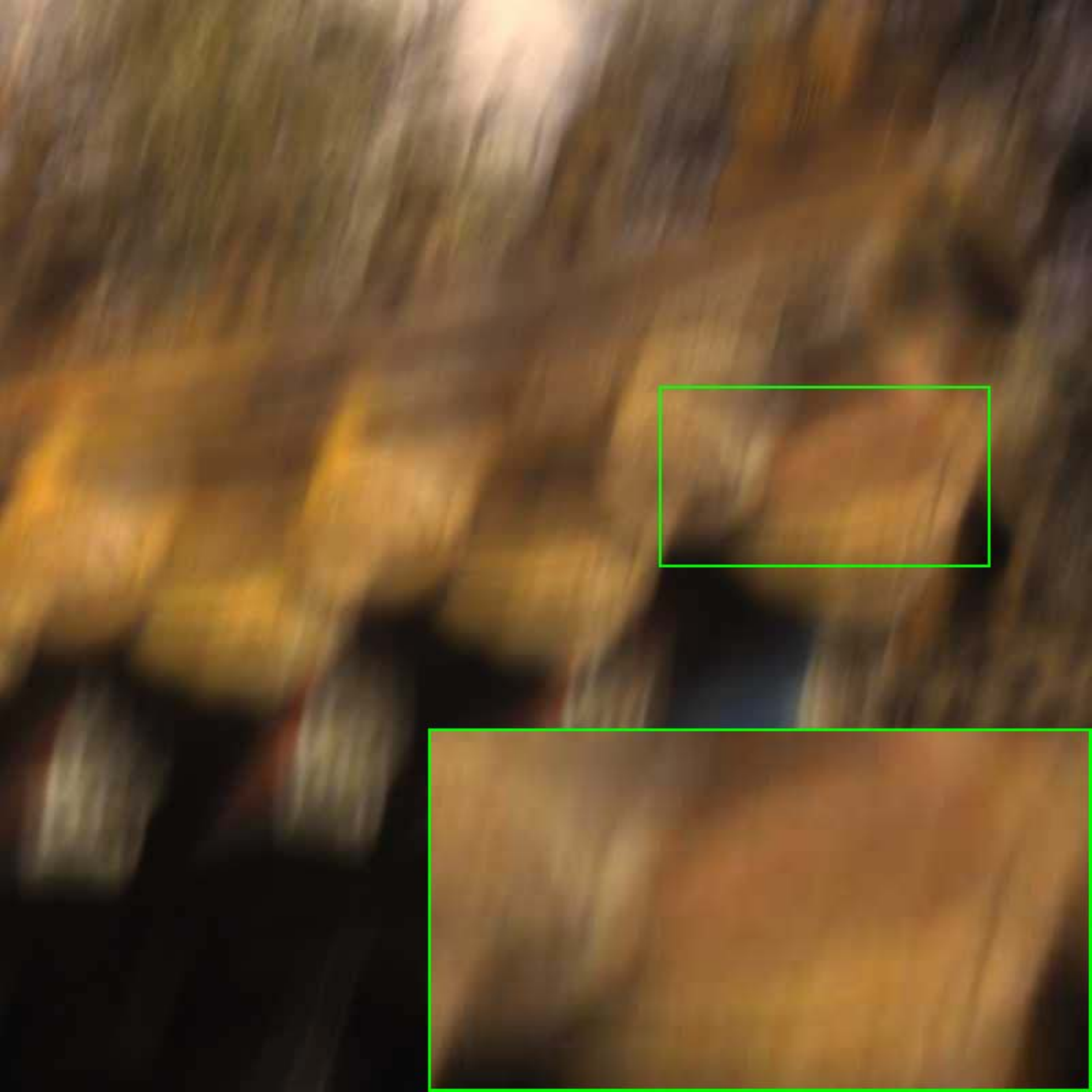}&
		\includegraphics[width=0.16\textwidth]{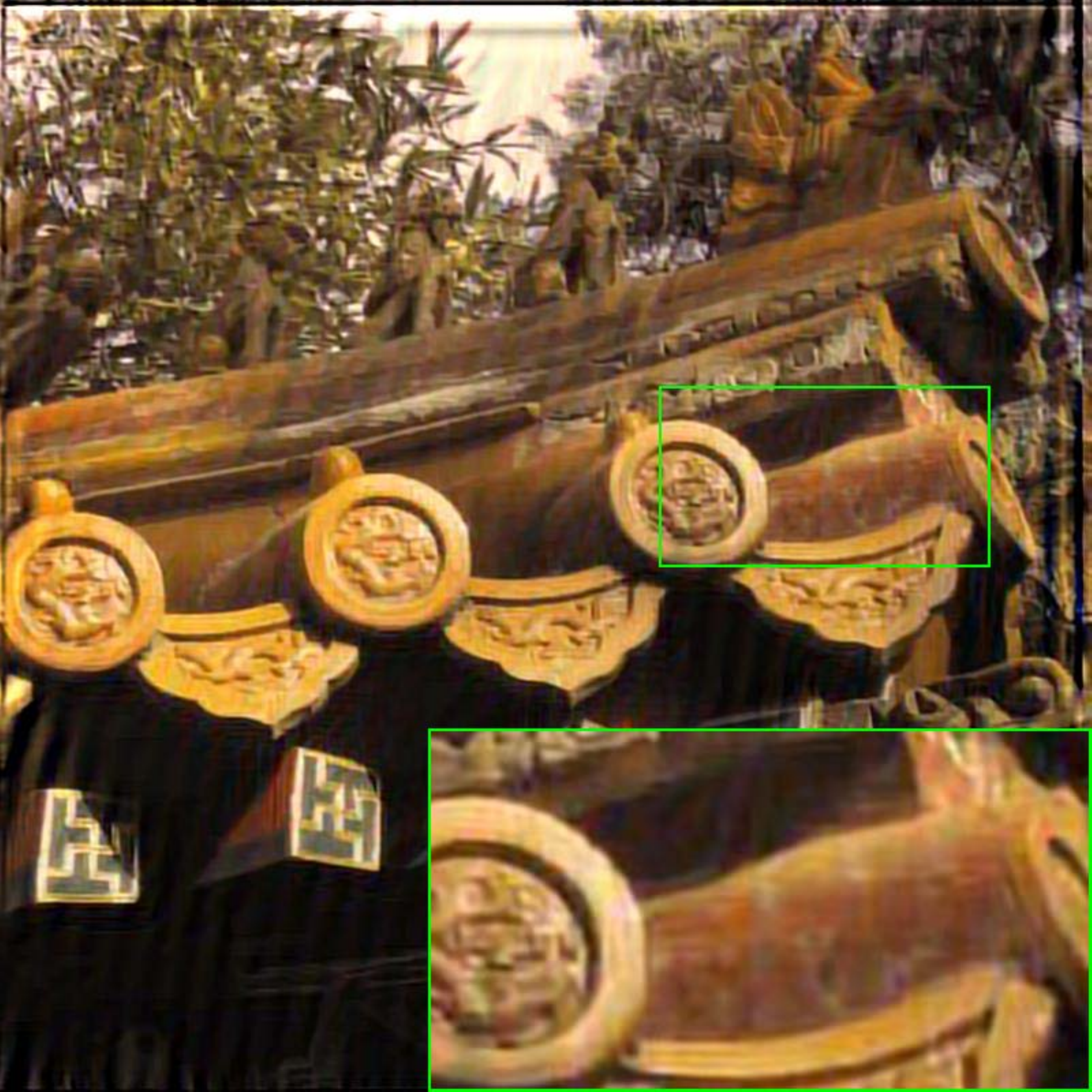}&
		\includegraphics[width=0.16\textwidth]{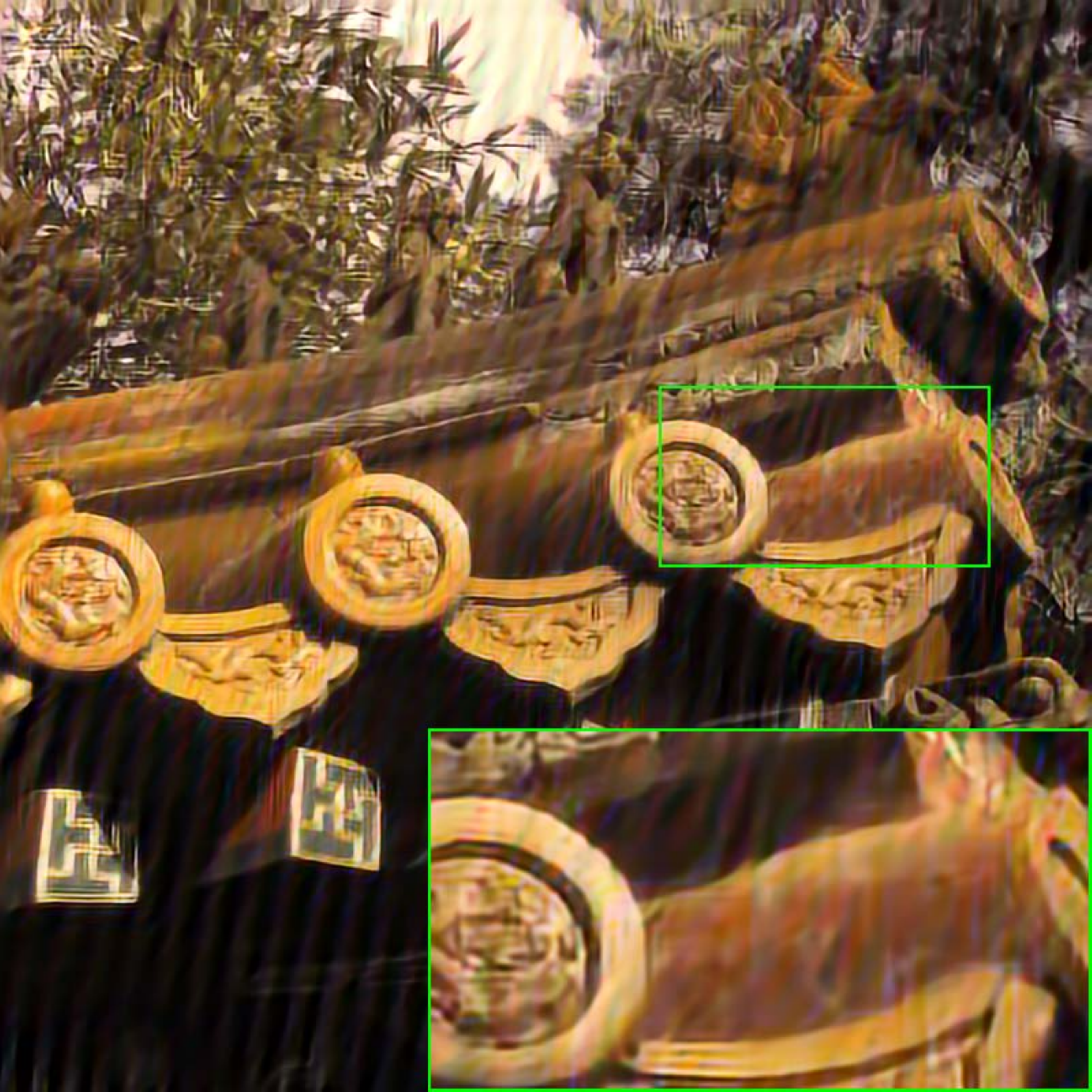}&
		\includegraphics[width=0.16\textwidth]{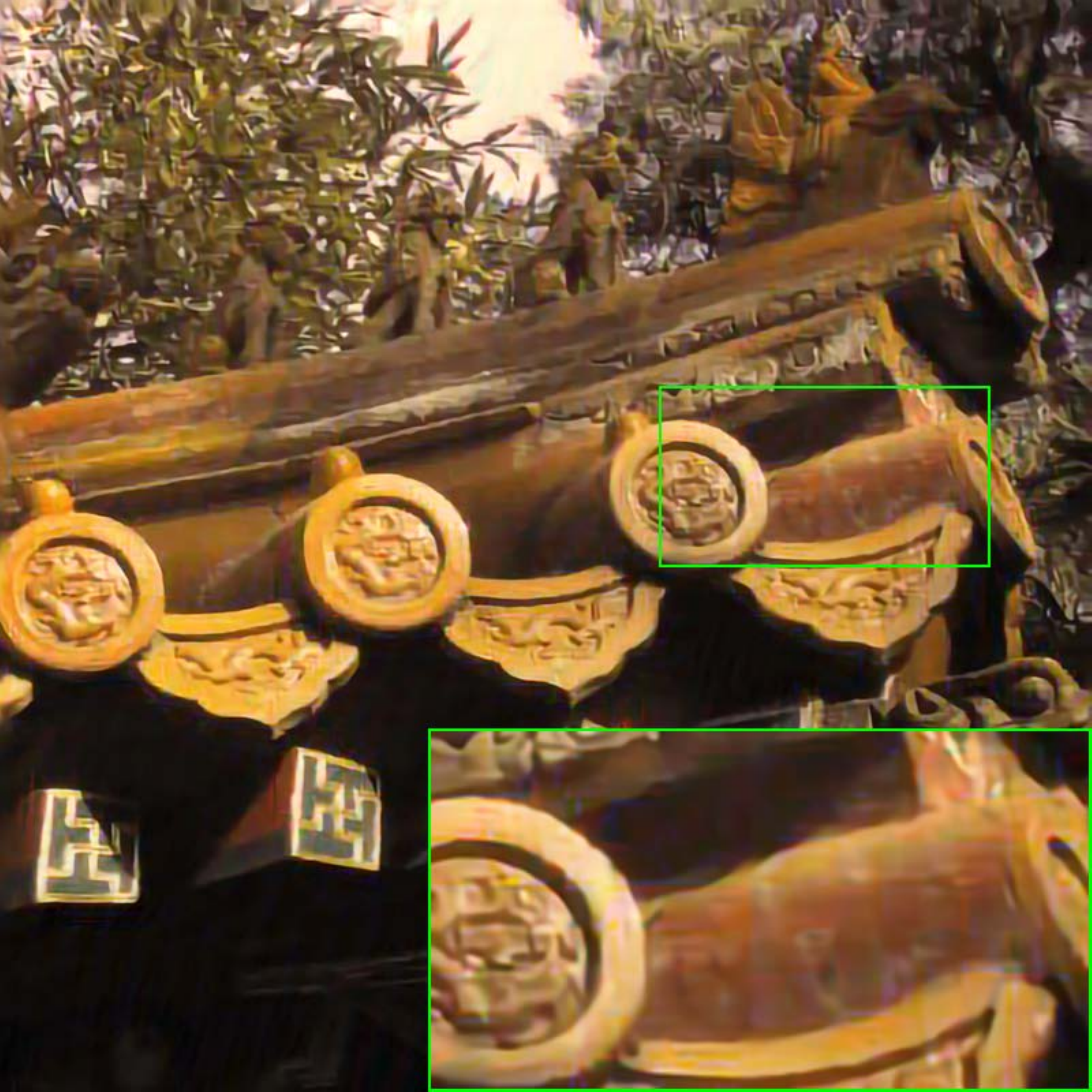}&
		\includegraphics[width=0.16\textwidth]{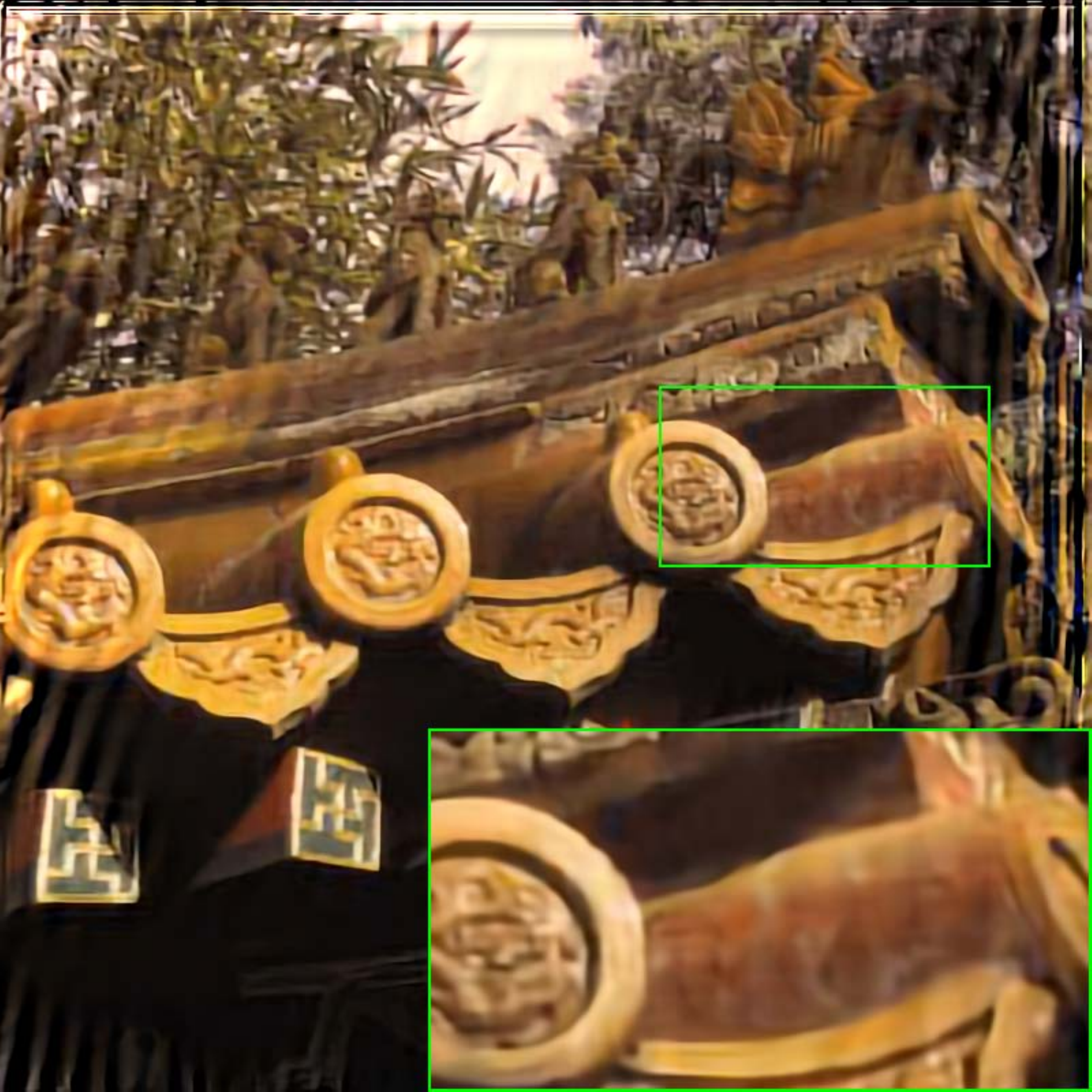}&
		\includegraphics[width=0.16\textwidth]{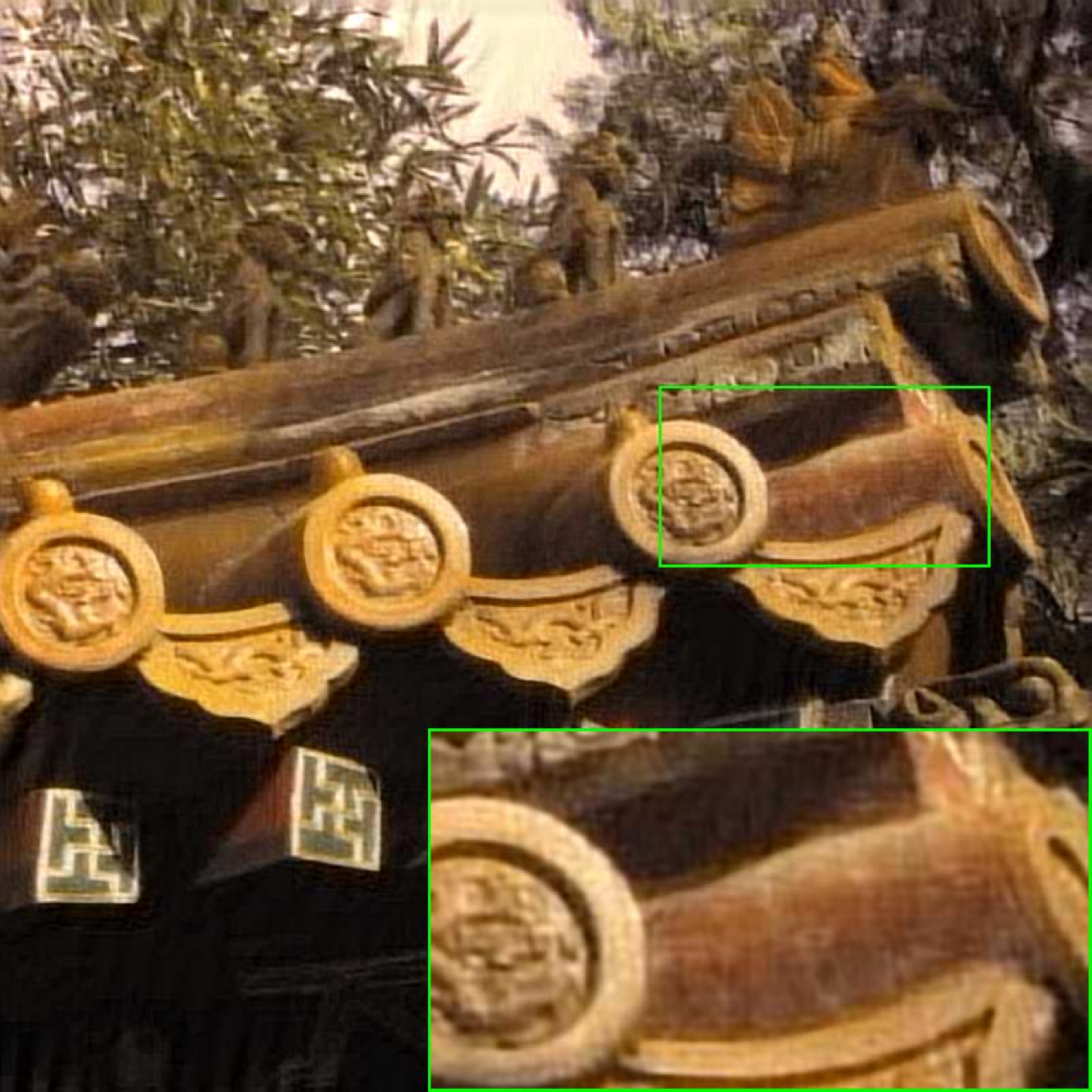}\\
		\footnotesize Blurry &\footnotesize IDDBM3D &\footnotesize MLP  &\footnotesize FDN &\footnotesize IRCNN & \footnotesize Ours
	\end{tabular}
	\caption{Comparisons of non-blind image deconvolution results with state-of-the-art methods on a challenging real-world blurry image.}
	\label{fig:realdeblurres}
\end{figure*}

\begin{figure*}[t]
	\begin{tabular}{c@{\extracolsep{0.2em}}c@{\extracolsep{0.2em}}c@{\extracolsep{0.2em}}c@{\extracolsep{0.2em}}c@{\extracolsep{0.2em}}c}
		\includegraphics[width=0.16\textwidth]{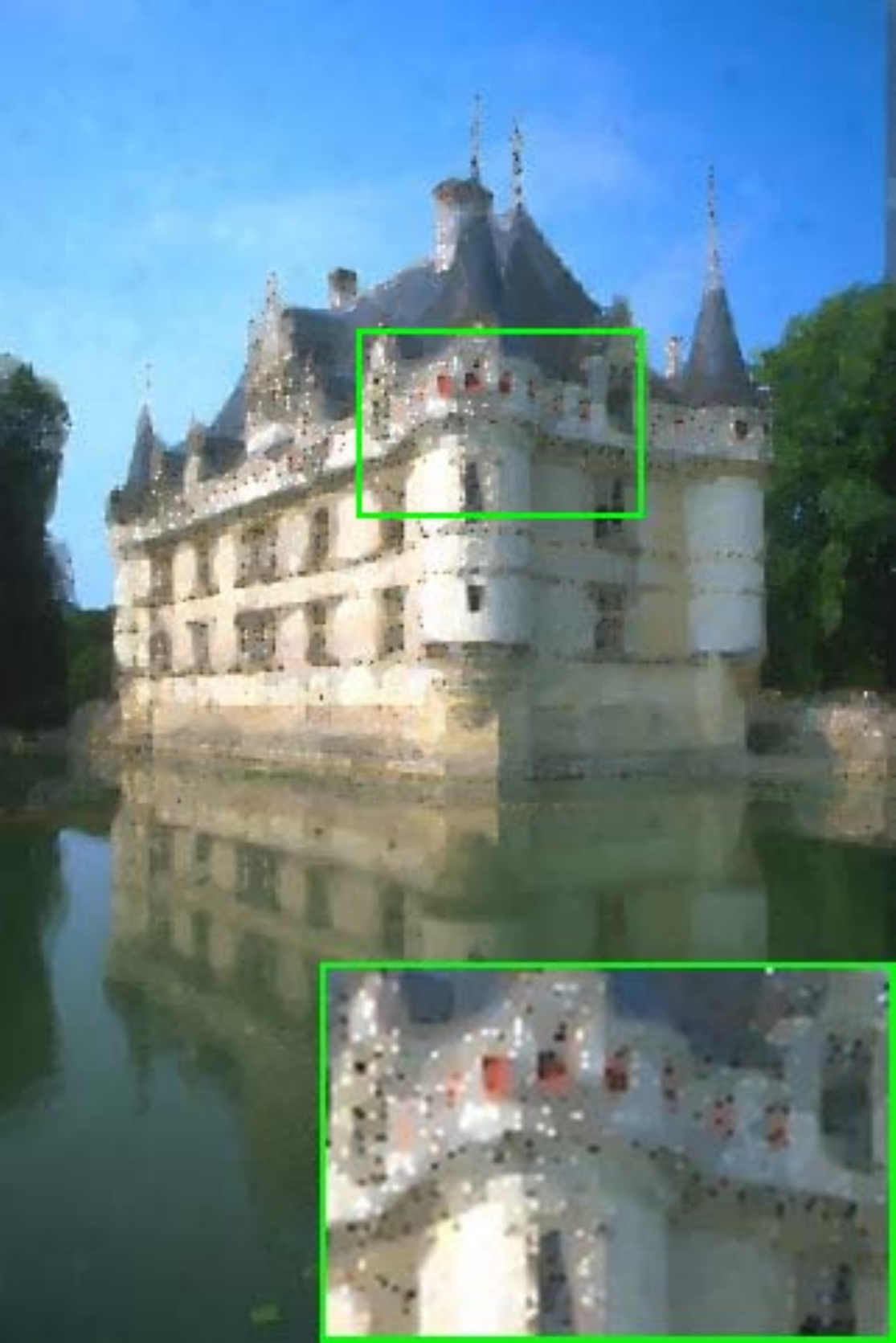}&
		\includegraphics[width=0.16\textwidth]{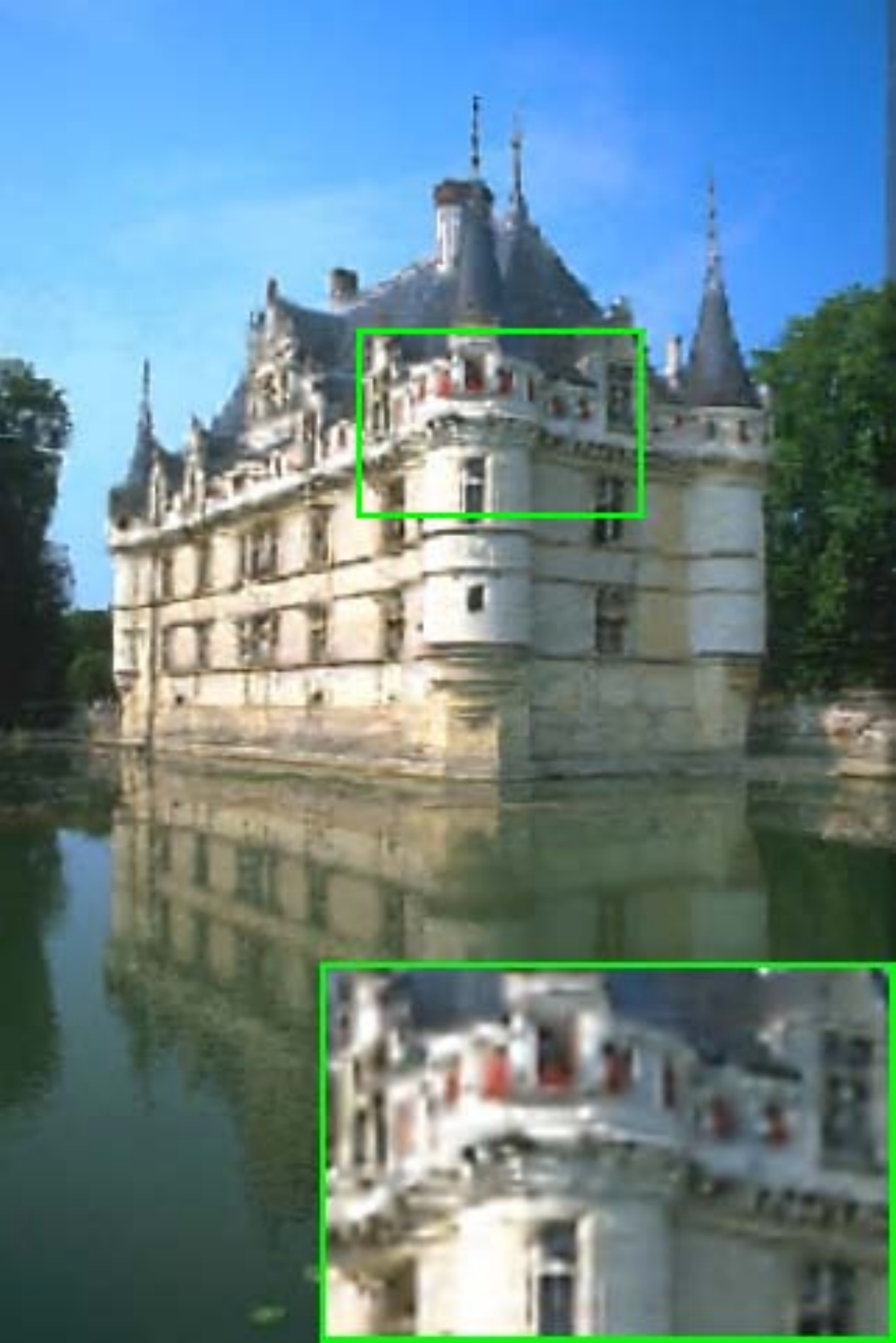}&
		\includegraphics[width=0.16\textwidth]{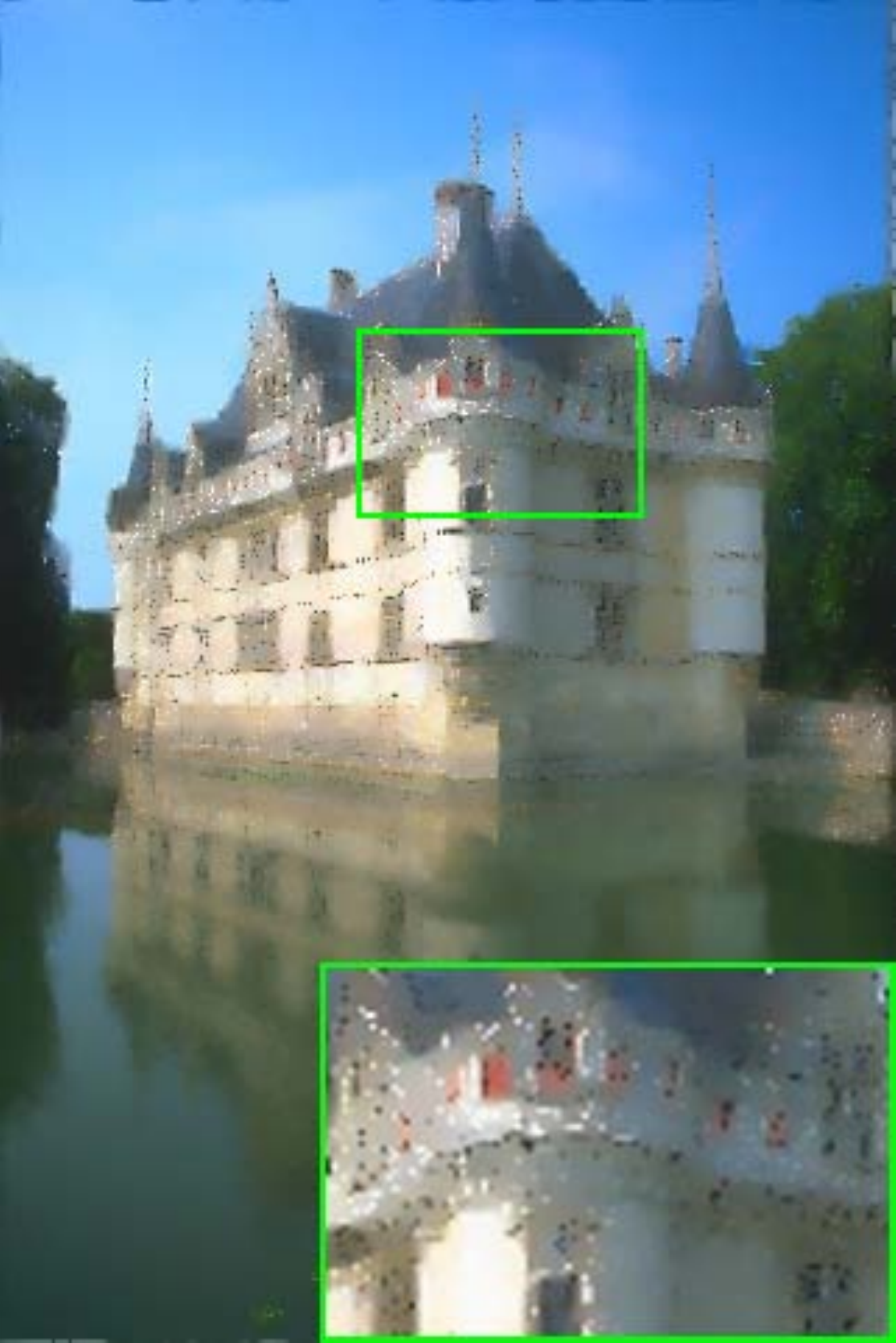}&
		\includegraphics[width=0.16\textwidth]{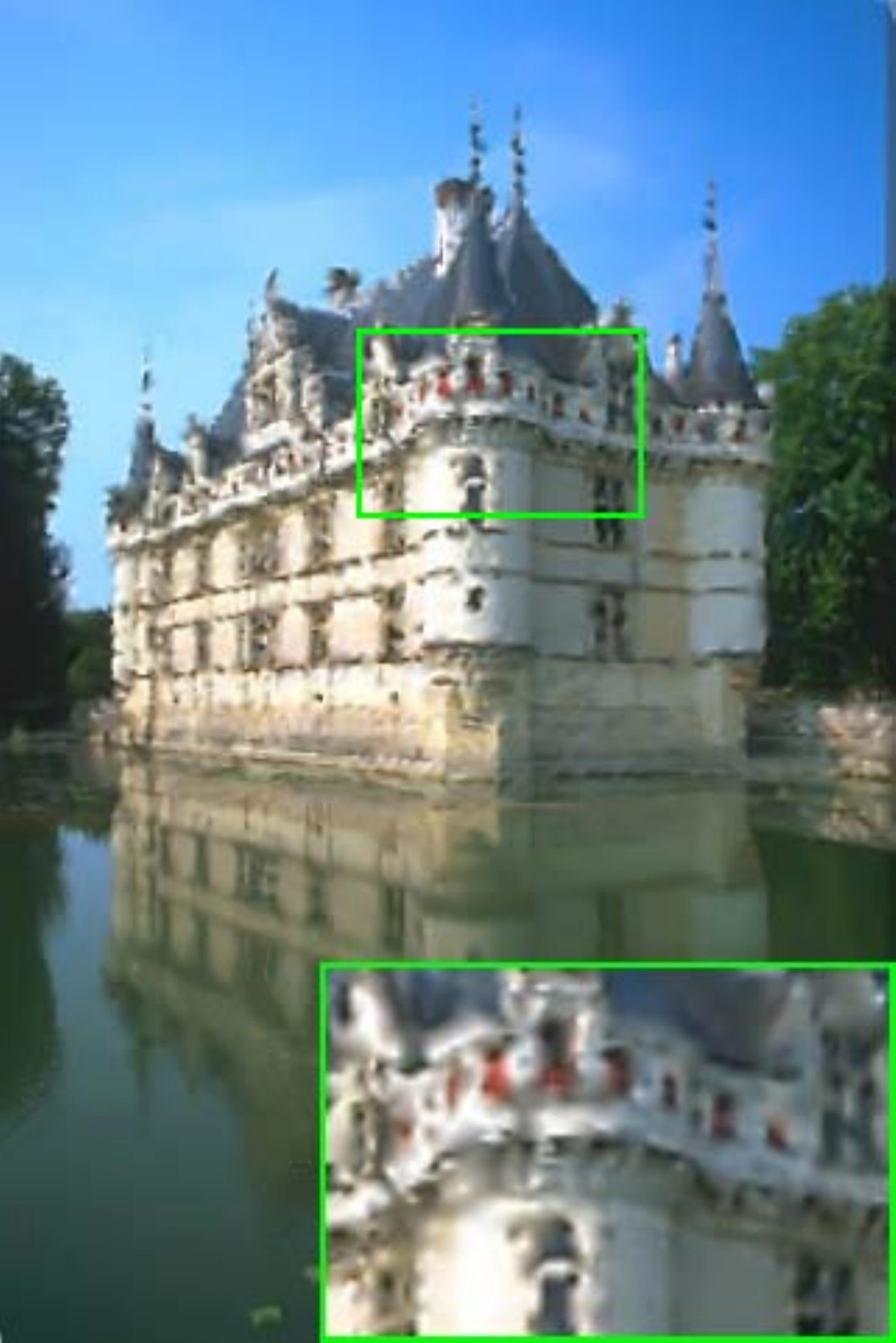}&
		\includegraphics[width=0.16\textwidth]{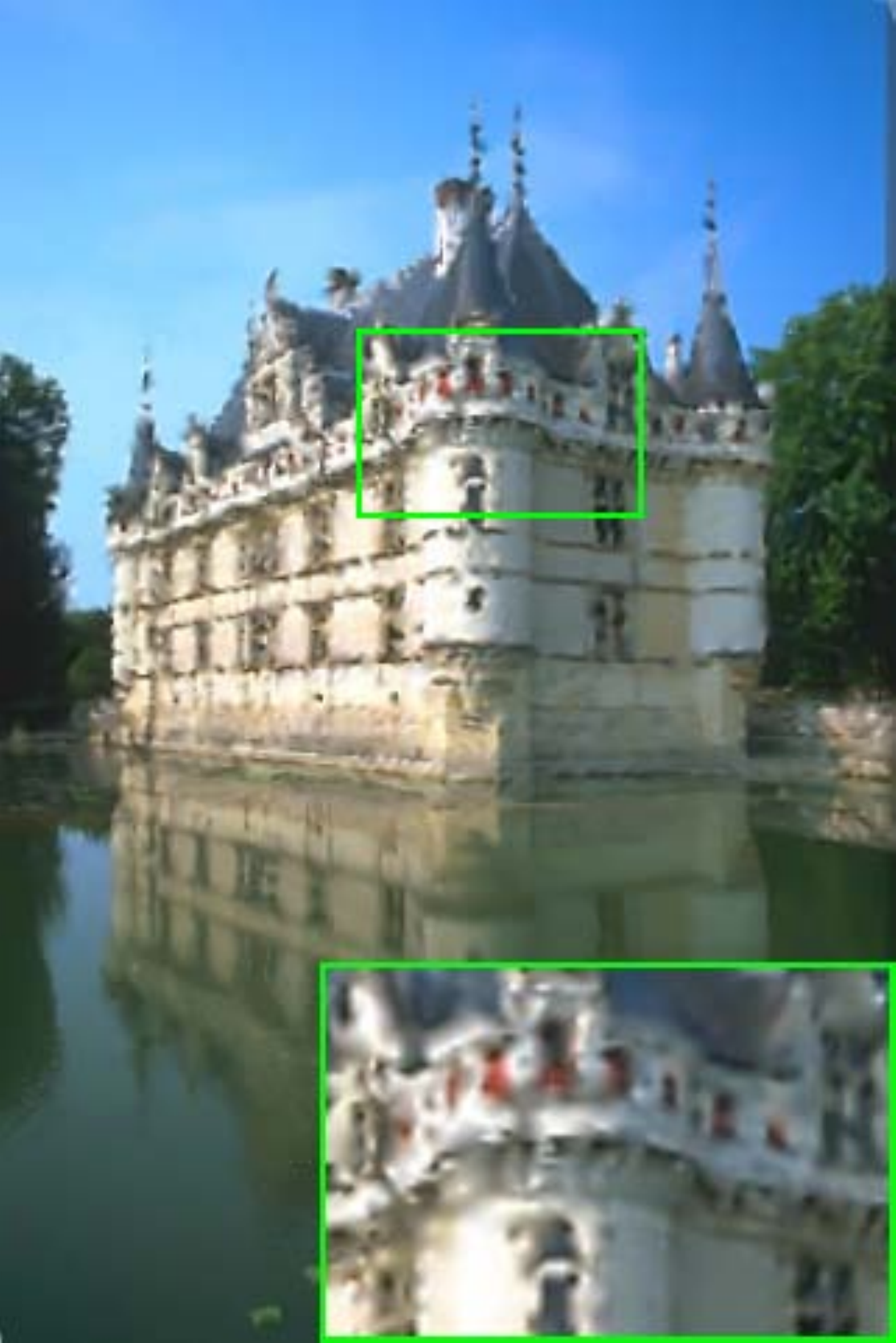}&
		\includegraphics[width=0.16\textwidth]{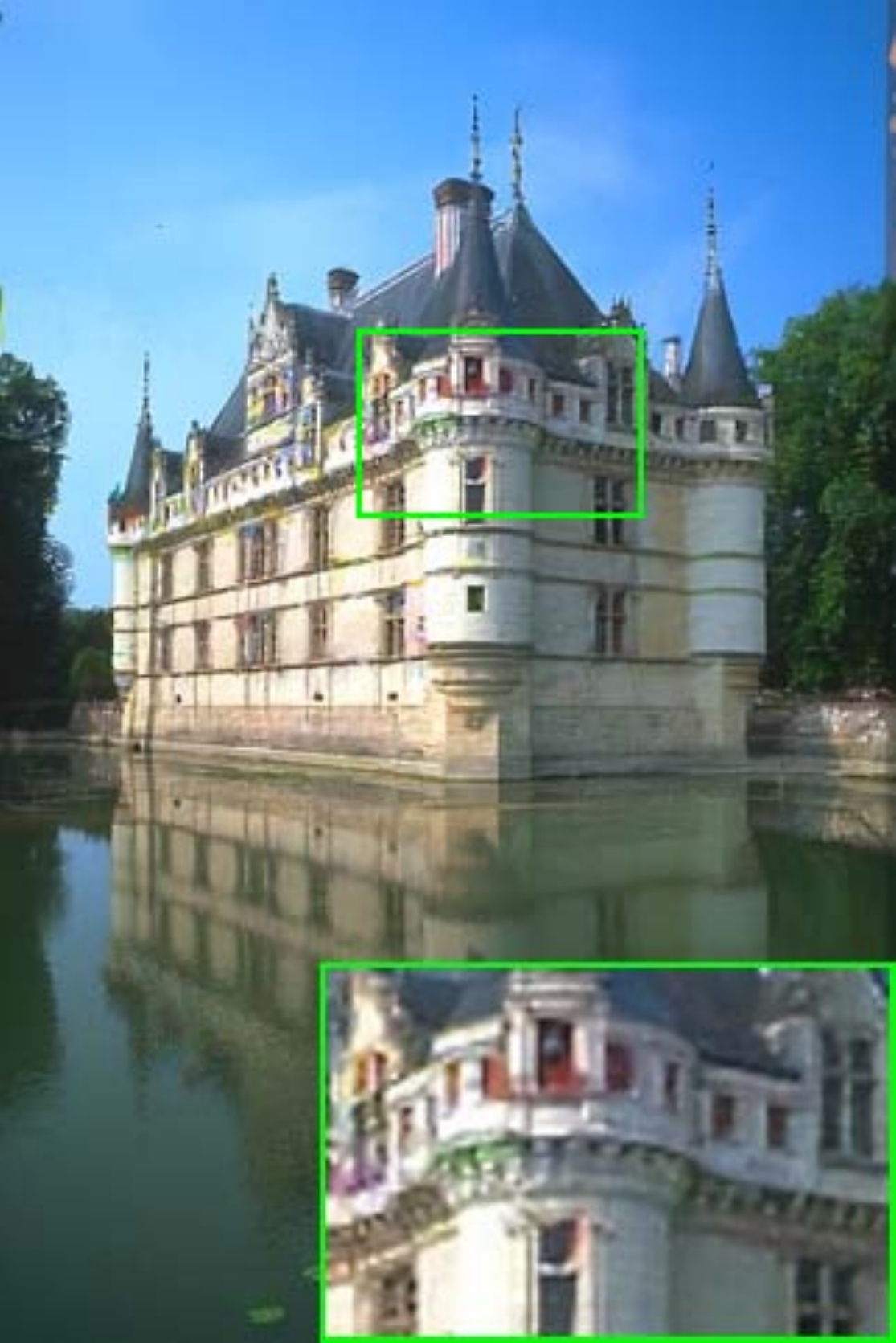}\\
		\footnotesize 24.02 / 0.79 &\footnotesize 25.57 / 0.84 &\footnotesize 23.21 / 0.76 & \footnotesize 26.00 / 0.89 & \footnotesize 25.95 / 0.85 &\footnotesize 29.02 / 0.92\\
		\footnotesize TV&\footnotesize FoE &\footnotesize ISDSB &\footnotesize WNNM &\footnotesize IRCNN  &\footnotesize Ours \\
	\end{tabular}
	\caption{Image inpainting results (with PSNR / SSIM scores) on a challenging example image with $80\%$ missing pixels.}
	\label{fig:inpainting}
\end{figure*}

\begin{table}[t]
	\renewcommand\arraystretch{1.15}
	\centering
	\caption{Averaged quantitative comparison of image deblurring on Sun \emph{et al.}'s and Levin \emph{et al.}'s benchmark.}
	\setlength{\tabcolsep}{1.8mm}{
		\begin{tabular}{c|c|c|c|c}
			\hline
			Methods & IDDBM3D     & TV    & EPLL & CSF\\
			\hline
			Levin   & 31.35/0.90   & 29.38/0.88  & 31.65/0.93& 31.55/0.87\\
			Sun  &  30.79/0.86 &30.67/0.85&	32.44/0.89&	31.55/0.88\\
			Times&	48.66 &	6.38&	721.98&	\textbf{0.50}\\
			\hline
			\hline
			Methods & MLP     & IRCNN    & FDN & Ours\\
			\hline
			Levin   & 31.32/0.90   & 32.28/0.92  & 32.04/0.93& \textbf{32.98}/\textbf{0.94}\\
			Sun  &  31.47/0.88 &32.61/0.89&	32.65/0.89&	\textbf{32.90}/\textbf{0.90}\\
			Times&	4.59 &	16.67&	2.70&	2.41\\
			\hline
	\end{tabular}}
	\label{tab:deblur}
\end{table}

\subsubsection{Convergence of TLF}
Next, we illustrated the convergent behaviors of TLF. In Fig.~\ref{fig:DiffError} (a) and (b), we plotted the iterative behaviors of variable ($\mathbf{x}^k$ and $\mathbf{x}^{k+1}$) and intermediate variables ($\mathbf{x}^k_{\mathcal{F}}$, $\mathbf{x}^k_{\mathcal{G}}$ and $\mathbf{x}^k_{\mathcal{G}_{\mu}}$). In Fig.~\ref{fig:DiffError} (a), the legends $\mathbf{x}^k\rightarrow\mathbf{x}^k_{\mathcal{G}}$,  $\mathbf{x}^k_{\mathcal{G}}\rightarrow\mathbf{x}^{k+1}$,  $\mathbf{x}^k\rightarrow\mathbf{x}^k_{\mathcal{F}}$ and $\mathbf{x}^k_{\mathcal{F}}\rightarrow\mathbf{x}^{k+1}$ prove the boundness of $\|\mathbf{x}^{k+1}-\mathbf{x}^k\|$, $\|\mathbf{x}^{k}_{\mathcal{G}}-\mathbf{x}^k\|$ and $\|\mathbf{x}^{k}_{\mathcal{F}}-\mathbf{x}^k\|$. Similarly, we plotted the corresponding curves of DTLF in Fig.~\ref{fig:DiffError} (b). To further illustrate the bounded condition (i.e., BUS) used in DTLF, Fig.~\ref{fig:DiffError} (c) showed the relationship between $\|\mathbf{x}^k_{\mathcal{G}_{\mu}}-\mathbf{x}^k\|$ and $C\|\mathbf{x}^k_{\mathcal{G}}-\mathbf{x}^k\|$. Obviously, Fig.~\ref{fig:DiffError} (c) implies that the boundedness of $\|\mathbf{x}^k_{\mathcal{G}_{\mu}}-\mathbf{x}^k\|$ is satisfied.

\begin{table}[t]
	\renewcommand\arraystretch{1.15}
	\centering
	\caption{Averaged quantitative comparison of image inpainting on CBSD68 dataset~\protect\cite{zhang2017beyond}. The first row is the proportion of masks. The first column is the Comparison methods on inpainting.}
	\setlength{\tabcolsep}{2mm}{
		\begin{tabular}{c|c|c|c|c}
			\hline
			Mask &$40\%$ &$60\%$ &$80\%$&Text\\
			\hline
			TV & 32.22/0.93&	29.20/0.86&	26.07/0.74&	35.29/0.97\\
			FoE  &  34.01/0.90 &30.81/0.81&	27.64/0.65&	37.05/0.95\\
			VNL&	27.55/0.91&	26.13/0.85&	24.23/0.75&	28.58/0.95\\
			ISDSB &	31.32/0.91&	28.23/0.83&	24.92/0.70&	34.91/0.96\\
			WNNM &	31.75/0.94&	28.71/0.89&	25.63/0.78&	34.89/0.97\\
			IRCNN	&34.92/0.95	&31.45/\textbf{0.91}	&26.44/0.79	&37.26/0.97\\
			Ours &\textbf{34.94}/\textbf{0.96}&\textbf{31.61}/\textbf{0.91}&	\textbf{27.88}/\textbf{0.81}& \textbf{37.38}/\textbf{0.98}\\
			\hline
	\end{tabular}}
	\label{tab:inpainting_comp}
\end{table}

\begin{figure*}[t]
	\centering
	\begin{tabular}{c@{\extracolsep{0.2em}}c@{\extracolsep{0.2em}}c@{\extracolsep{0.2em}}c@{\extracolsep{0.2em}}c@{\extracolsep{0.2em}}c}
		\includegraphics[width=0.16\textwidth]{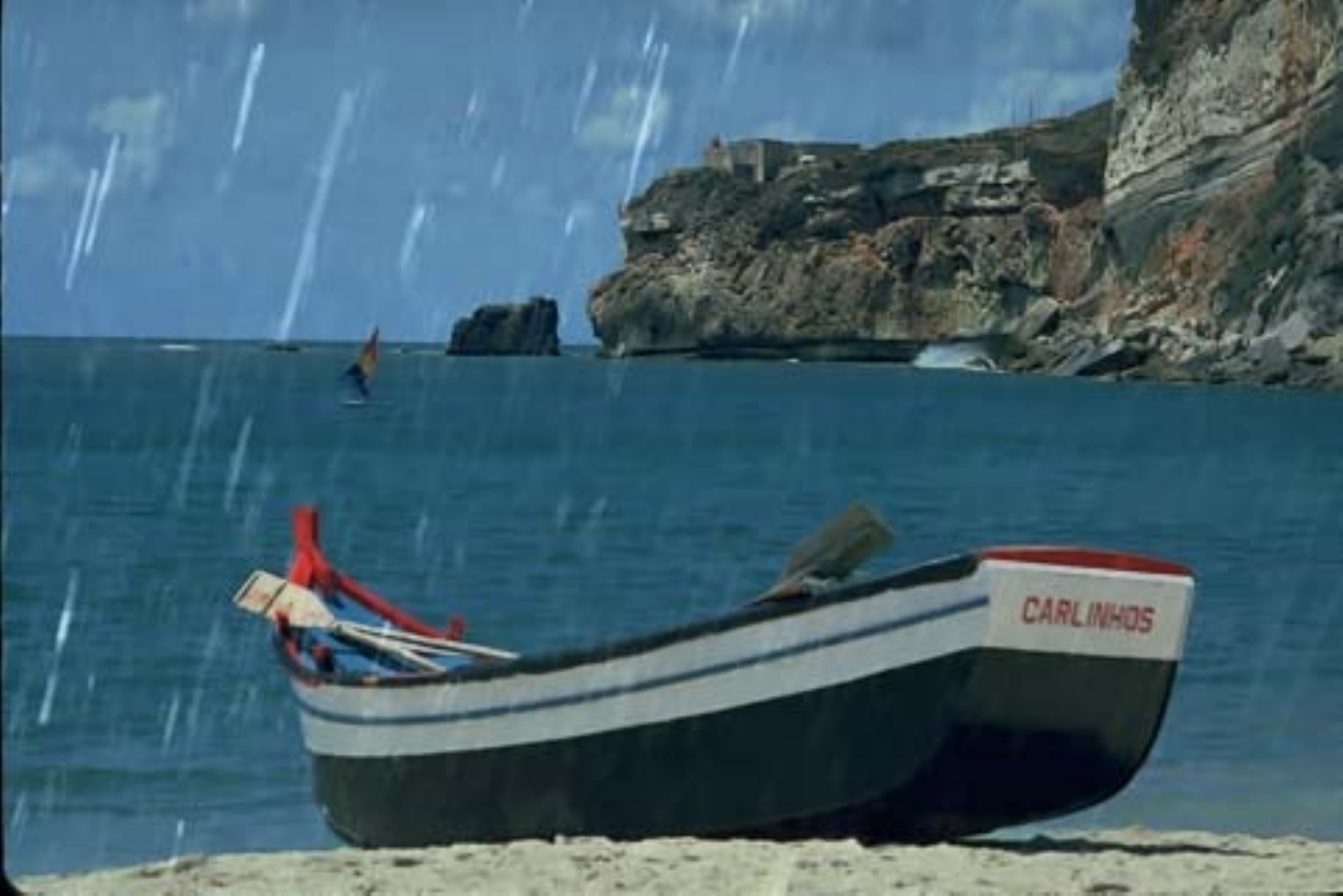}&
		\includegraphics[width=0.16\textwidth]{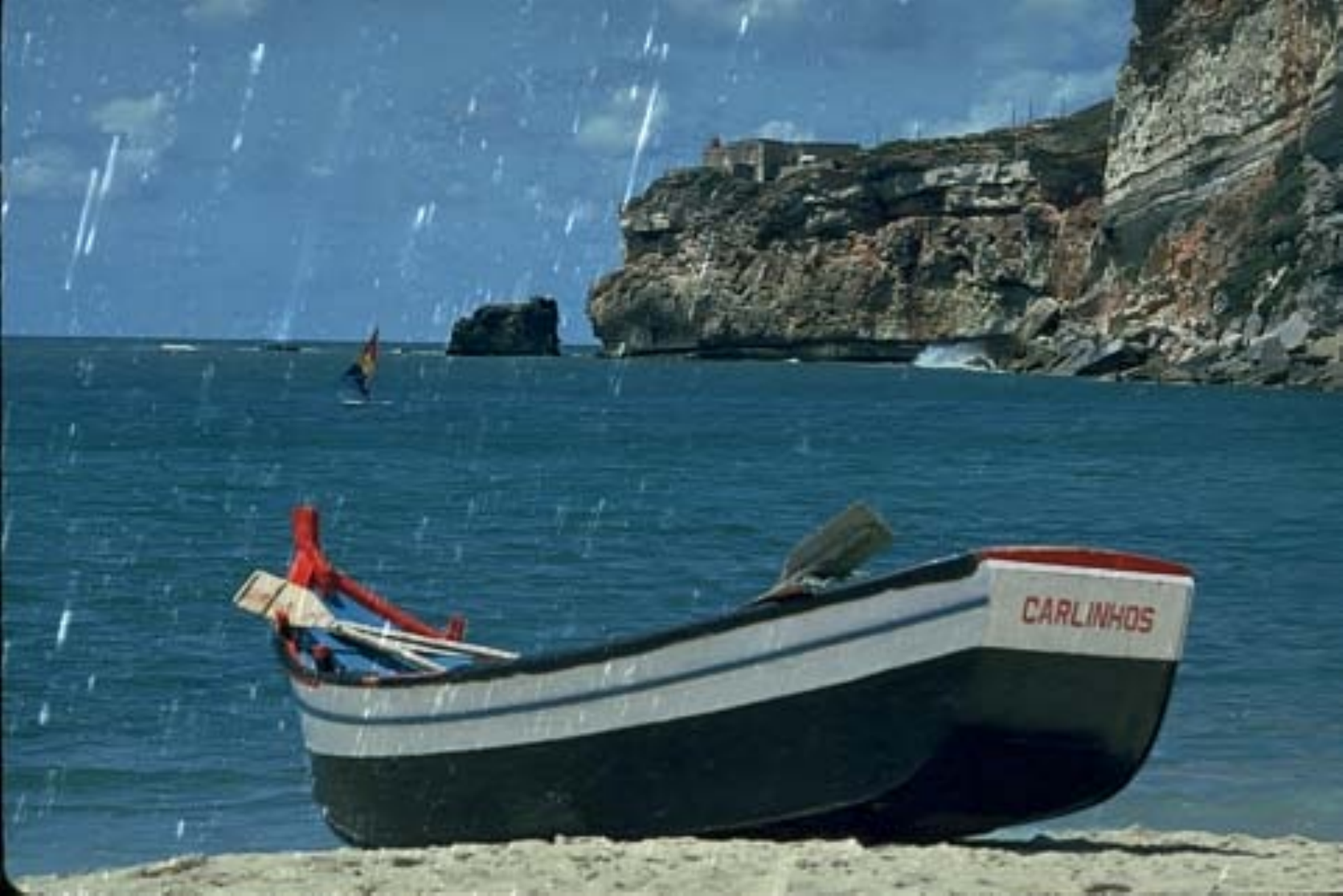}&
		\includegraphics[width=0.16\textwidth]{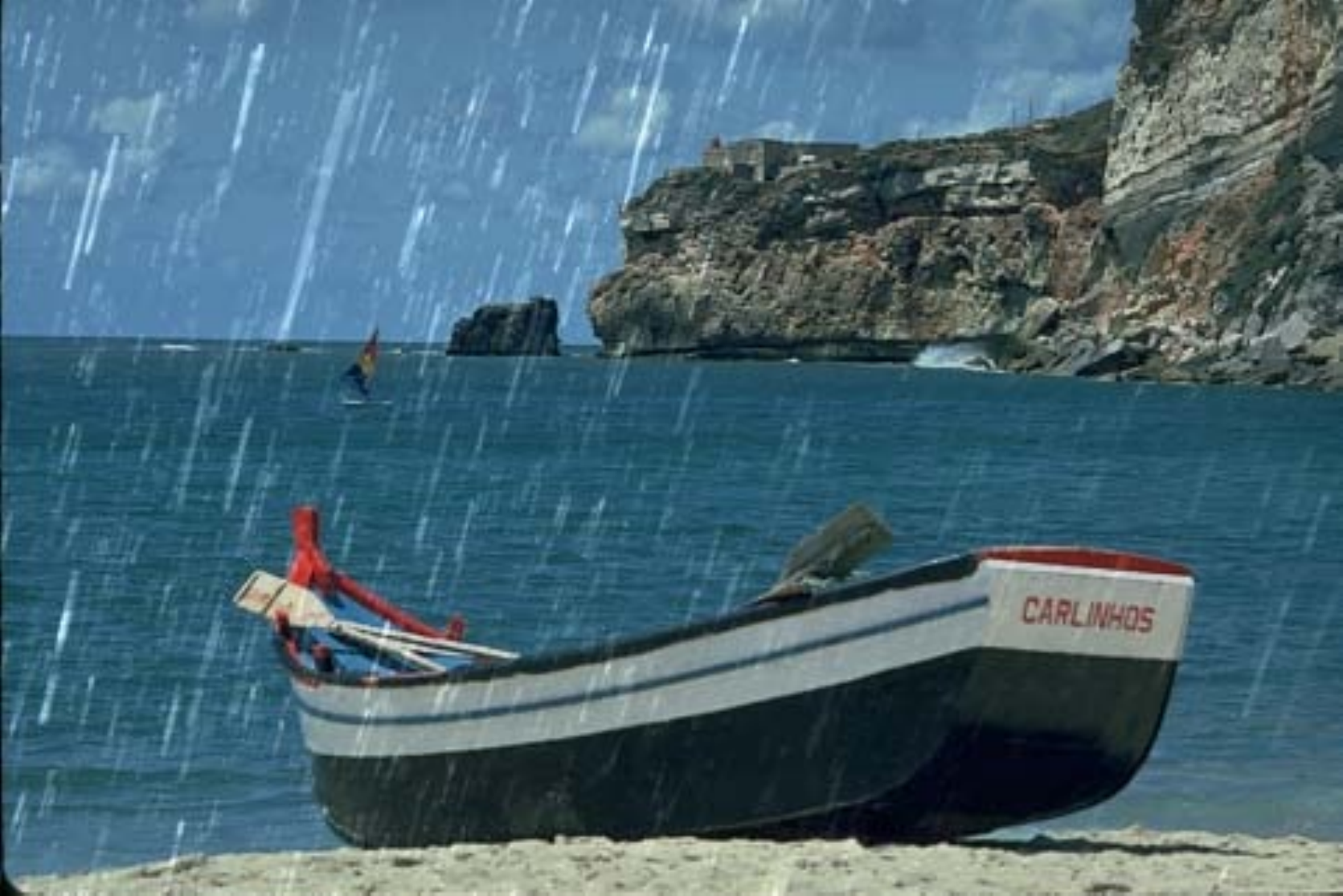}&
		\includegraphics[width=0.16\textwidth]{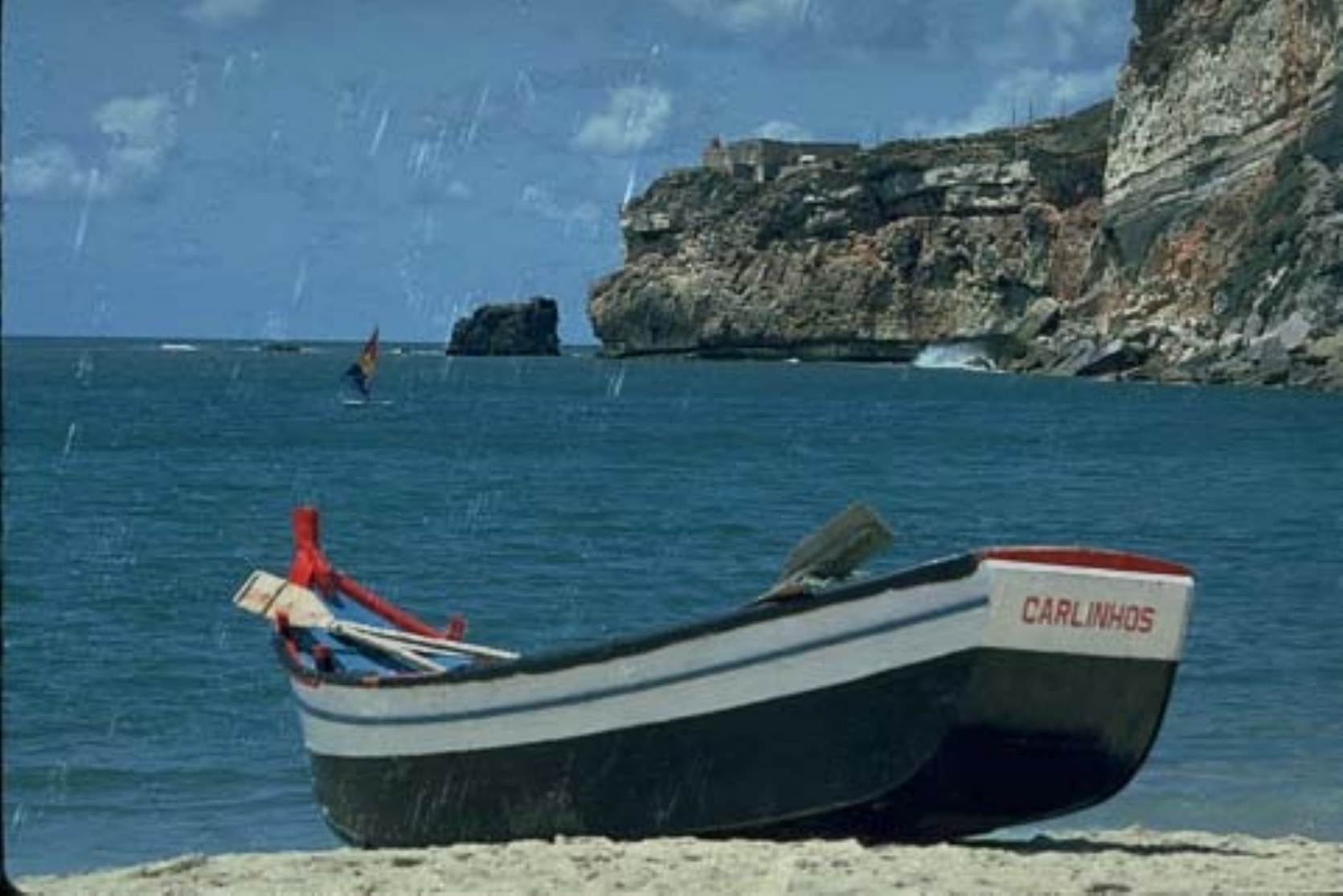}&
		\includegraphics[width=0.16\textwidth]{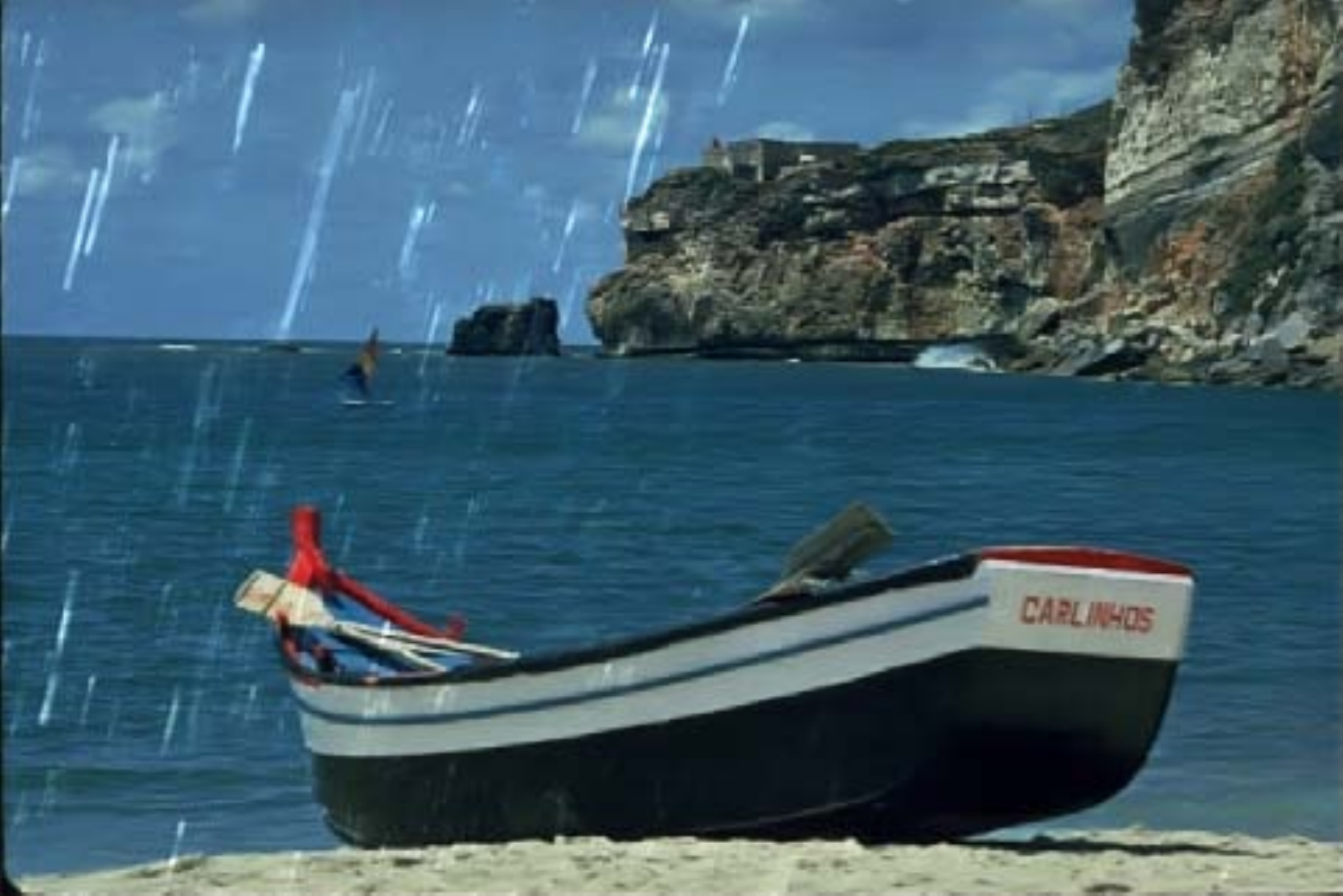}&
		\includegraphics[width=0.16\textwidth]{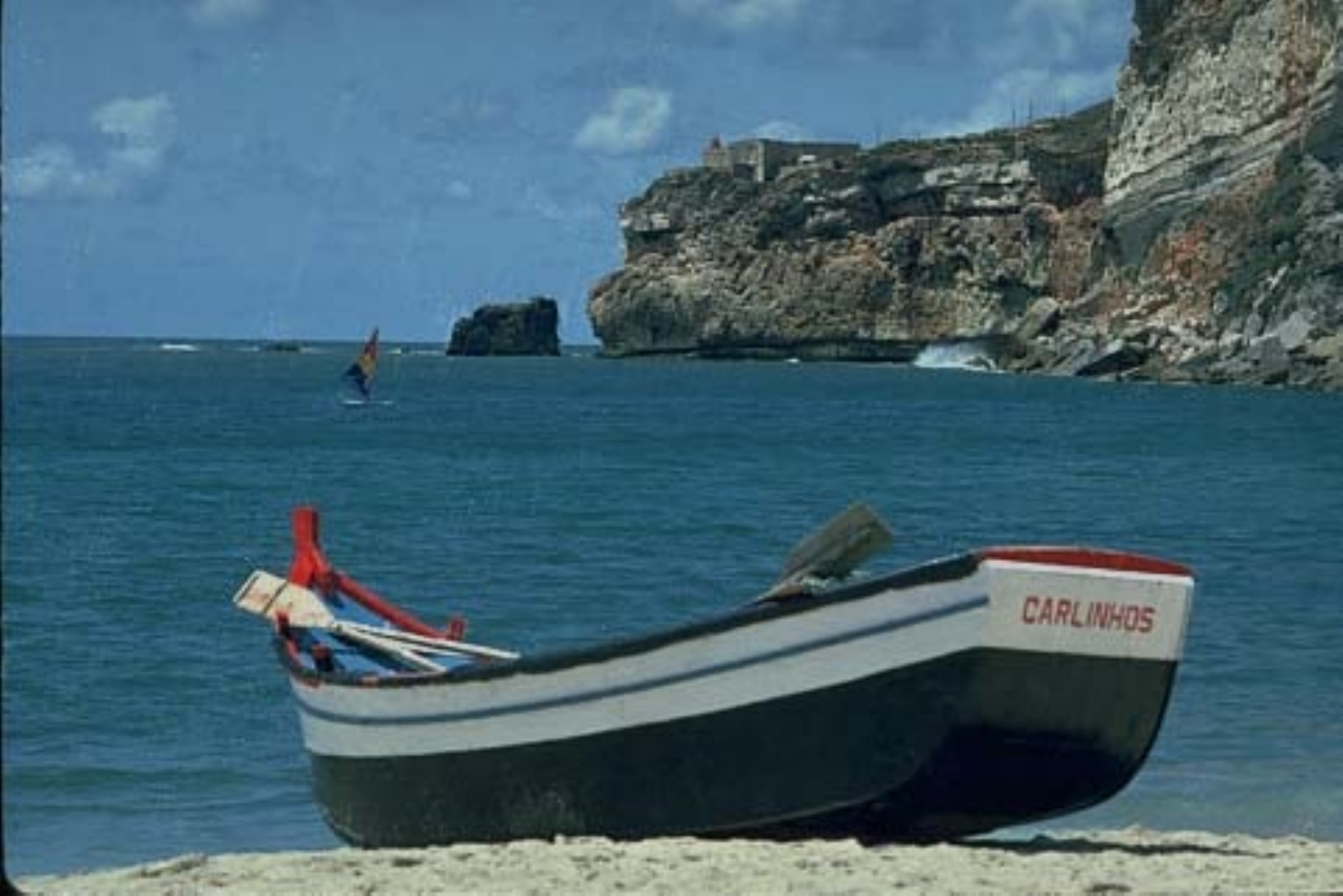}\\
		\footnotesize 32.87 / 0.91 &\footnotesize 32.12 / 0.92 &\footnotesize 29.69 / 0.86 &\footnotesize 33.40 / 0.96 &\footnotesize 28.18 / 0.89  &\footnotesize 37.10 / 0.97\\
		\includegraphics[width=0.16\textwidth]{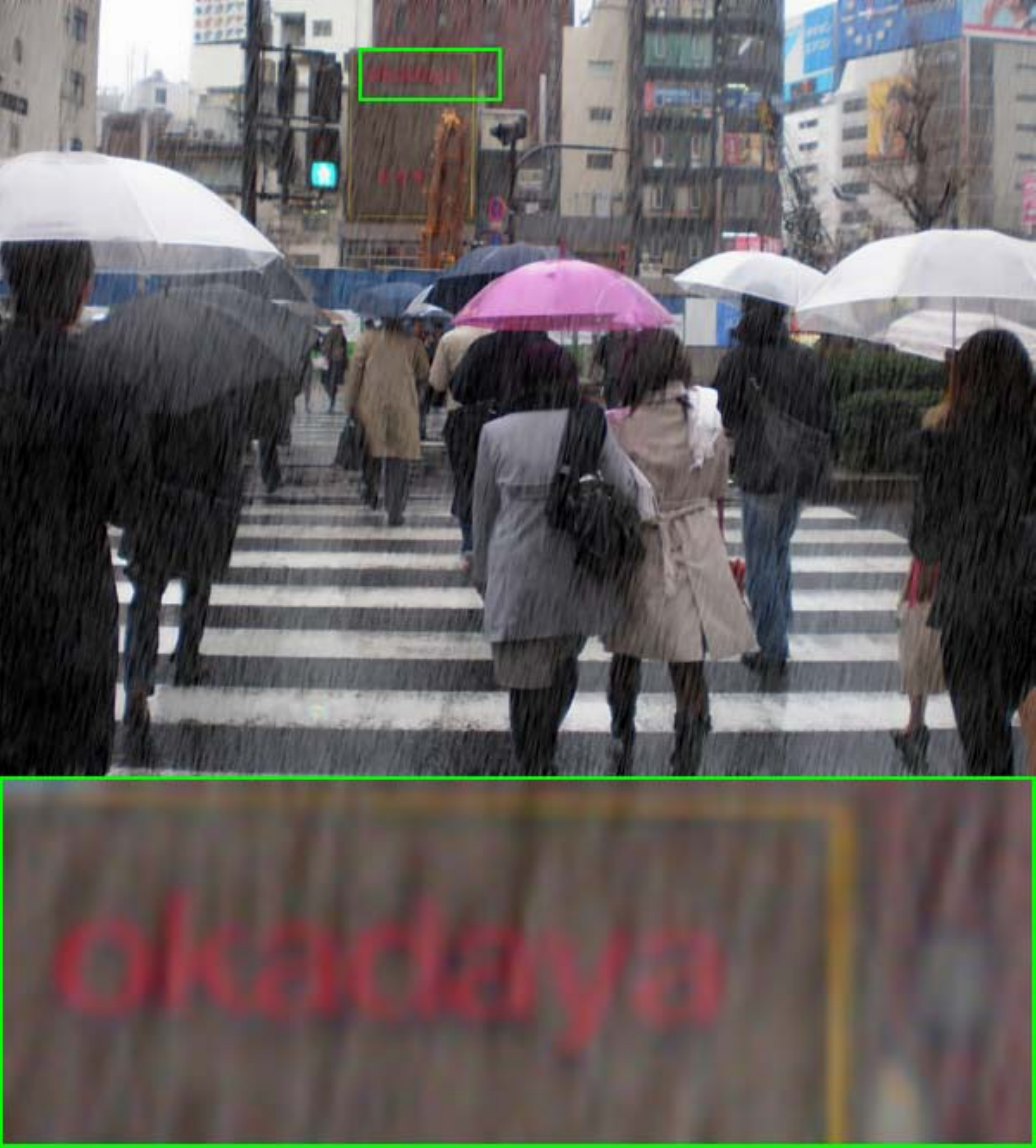}&
		\includegraphics[width=0.16\textwidth]{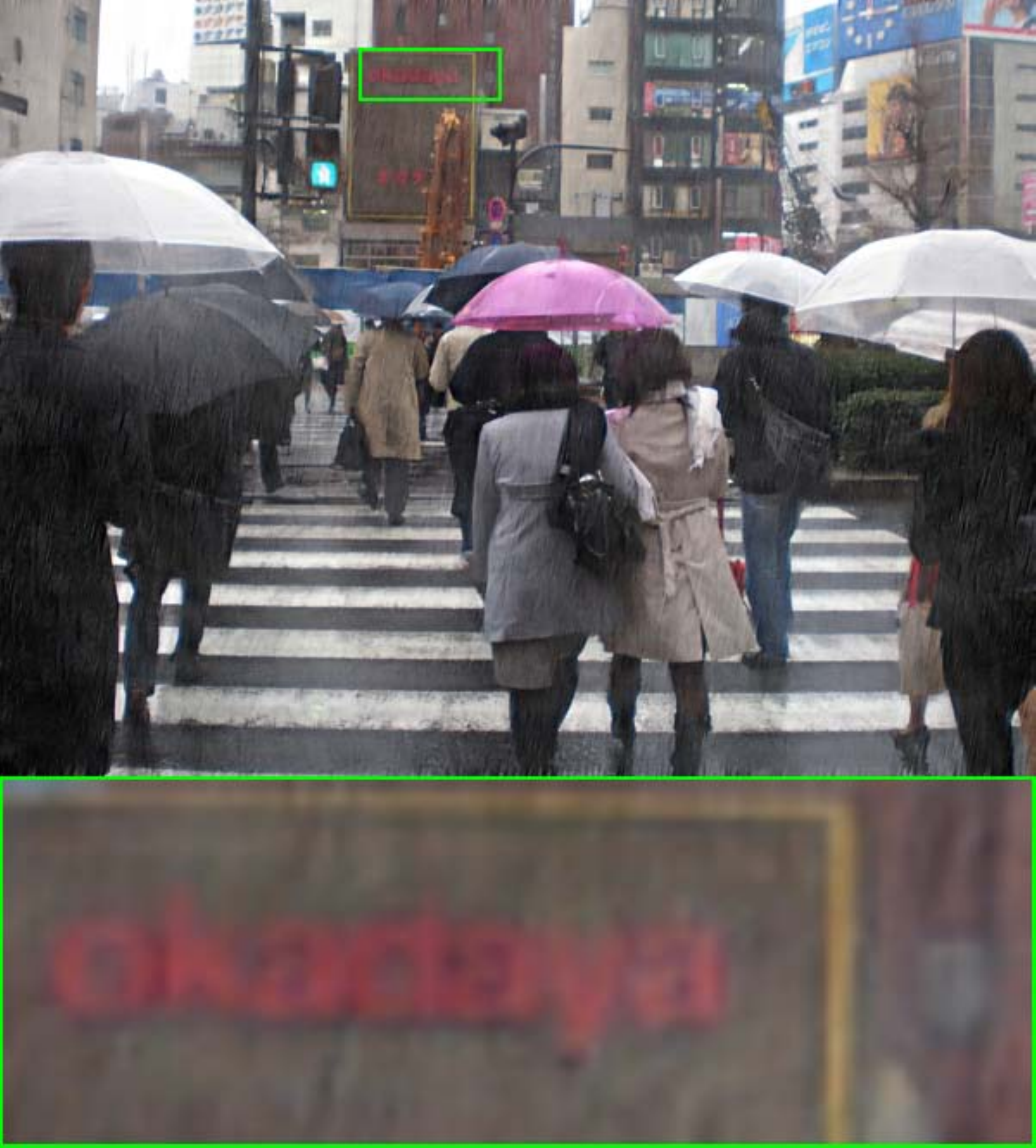}&
		\includegraphics[width=0.16\textwidth]{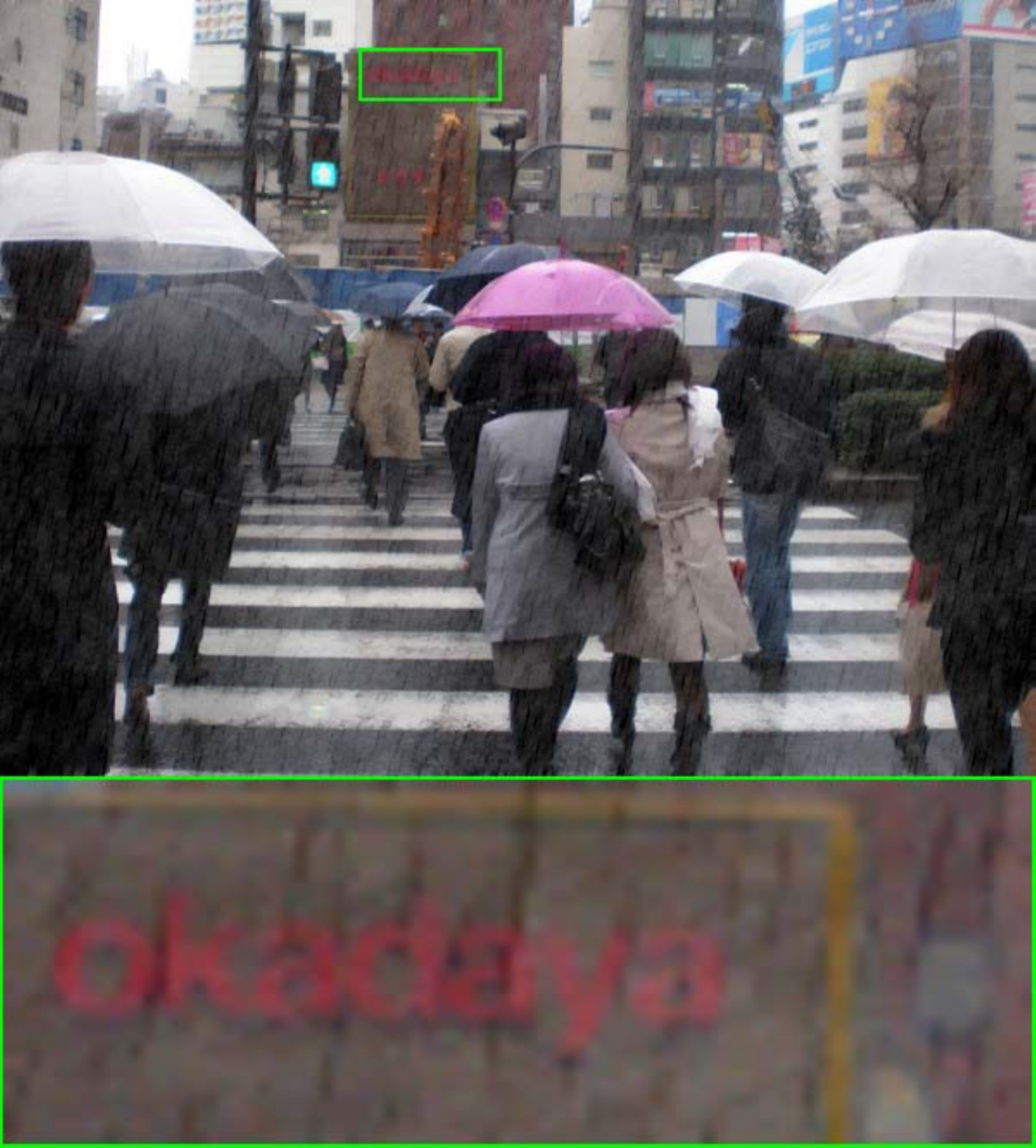}&
		\includegraphics[width=0.16\textwidth]{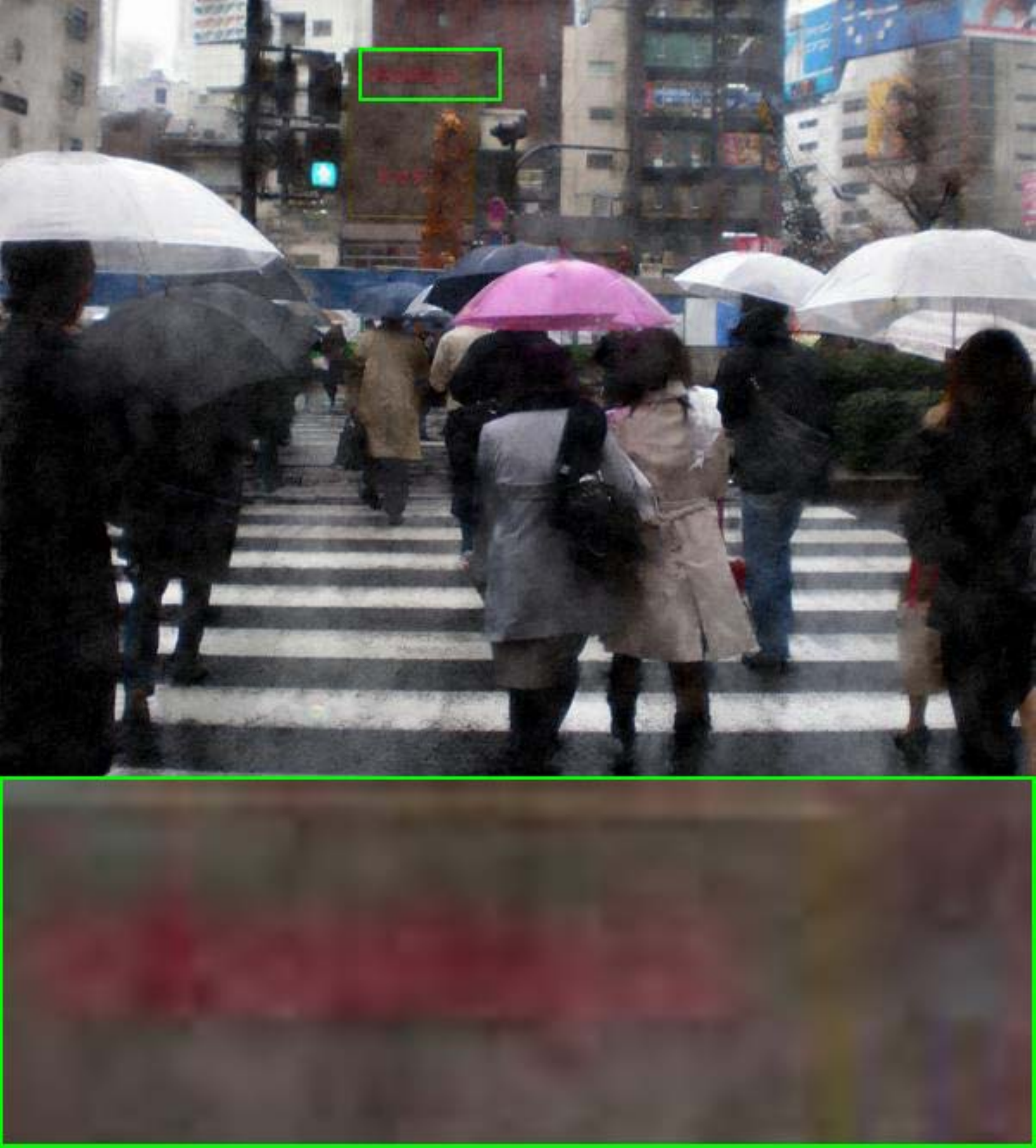}&
		\includegraphics[width=0.16\textwidth]{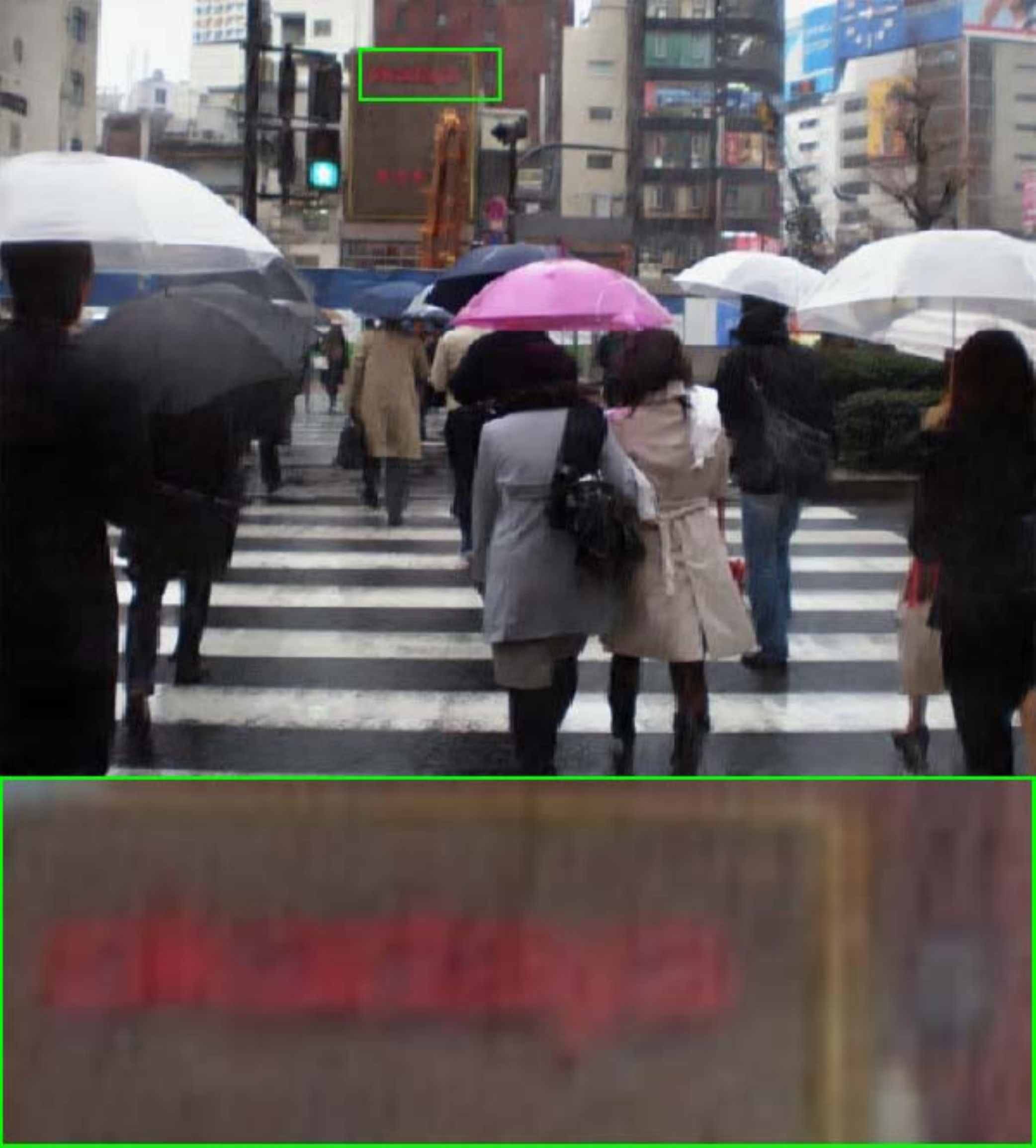}&
		\includegraphics[width=0.16\textwidth]{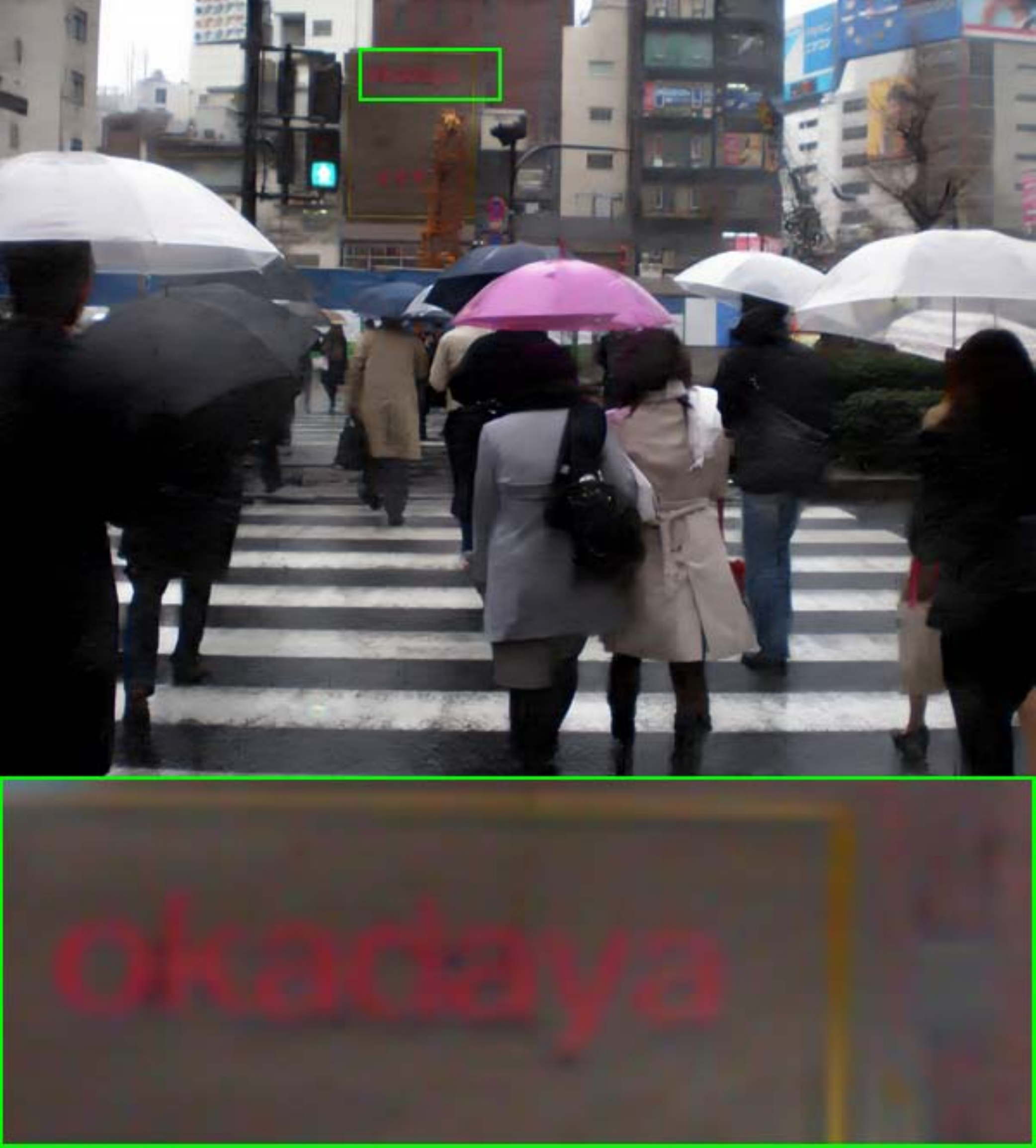}\\
		\includegraphics[width=0.16\textwidth]{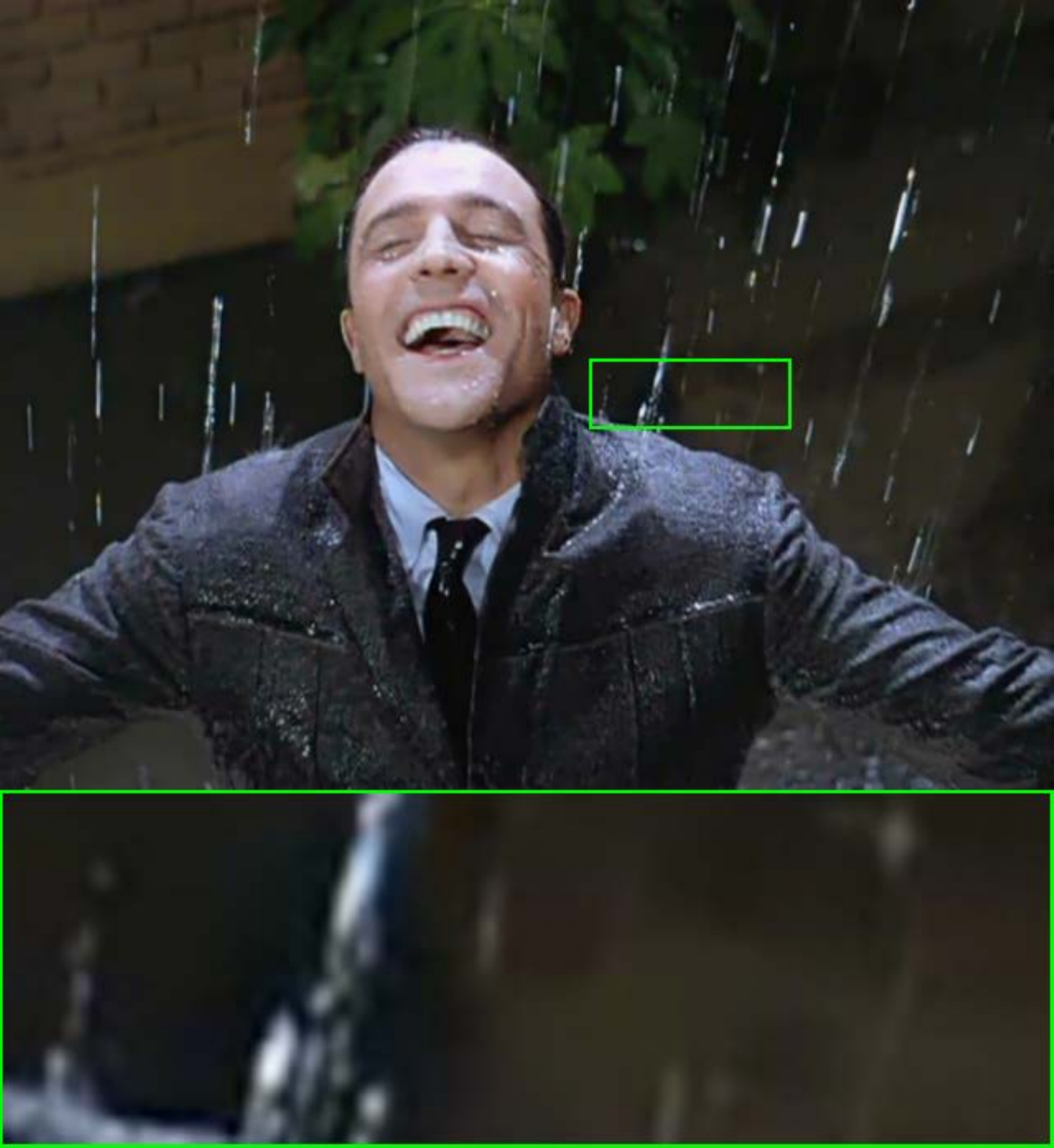}&
		\includegraphics[width=0.16\textwidth]{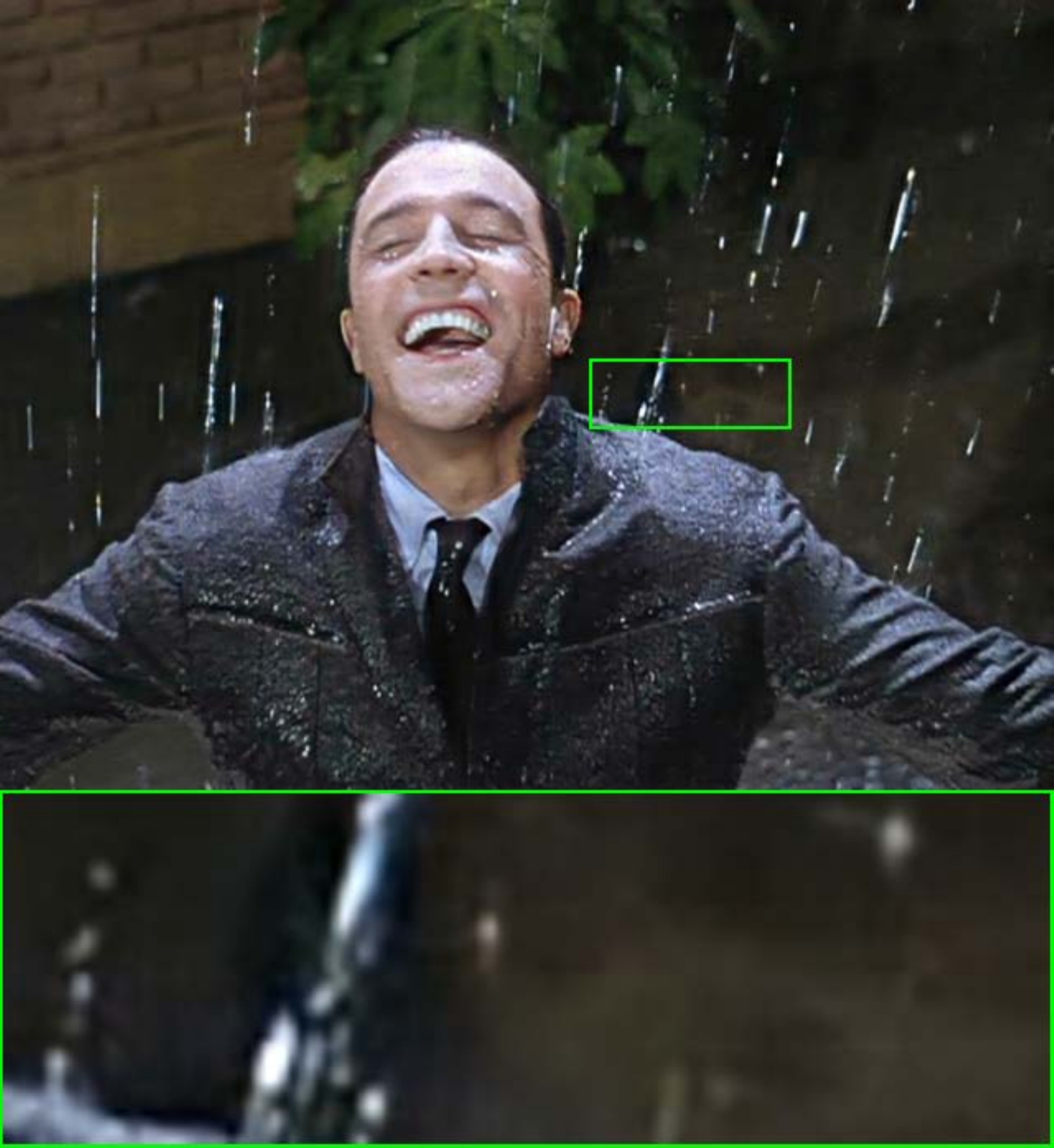}&
		\includegraphics[width=0.16\textwidth]{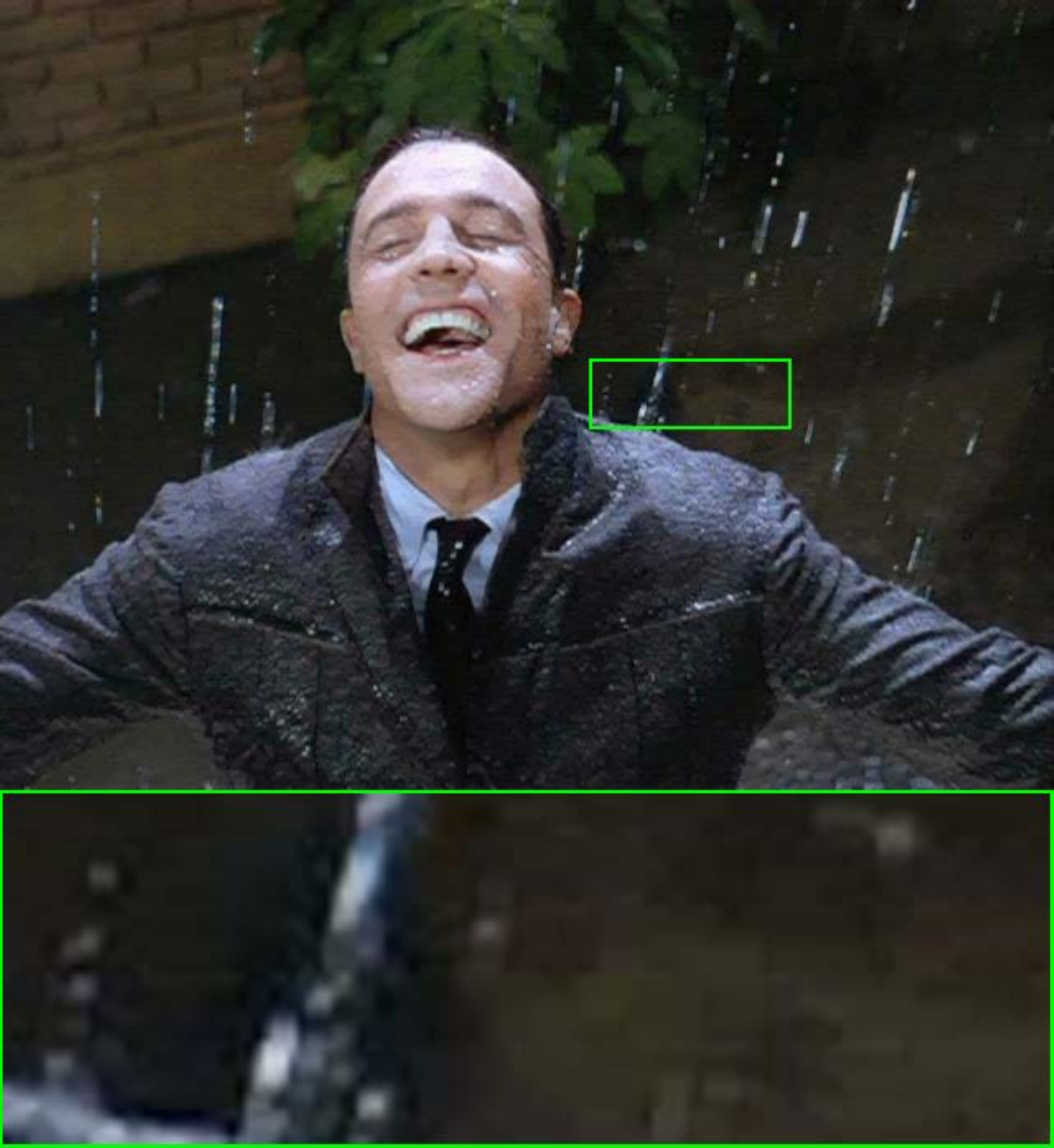}&
		\includegraphics[width=0.16\textwidth]{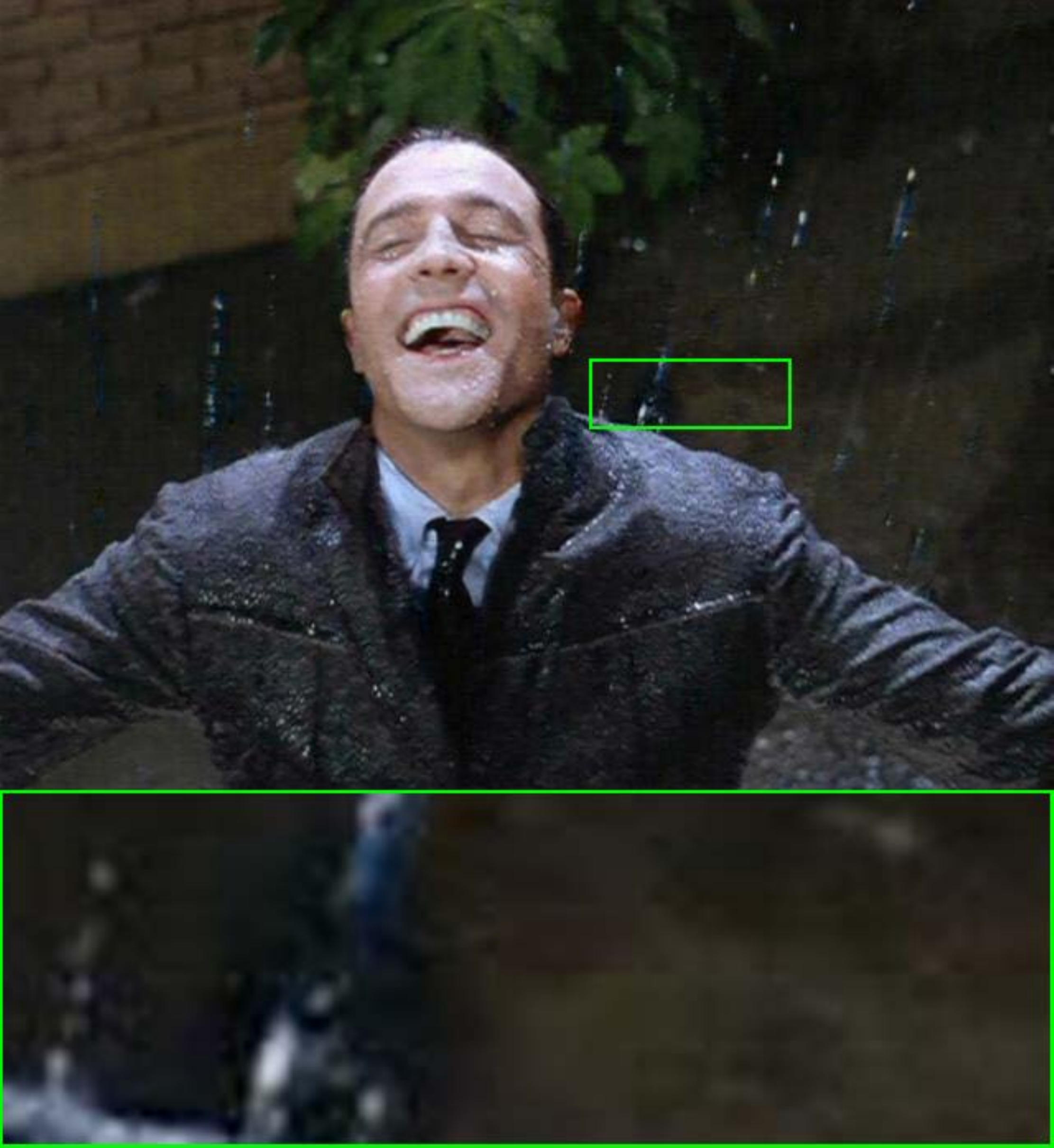}&
		\includegraphics[width=0.16\textwidth]{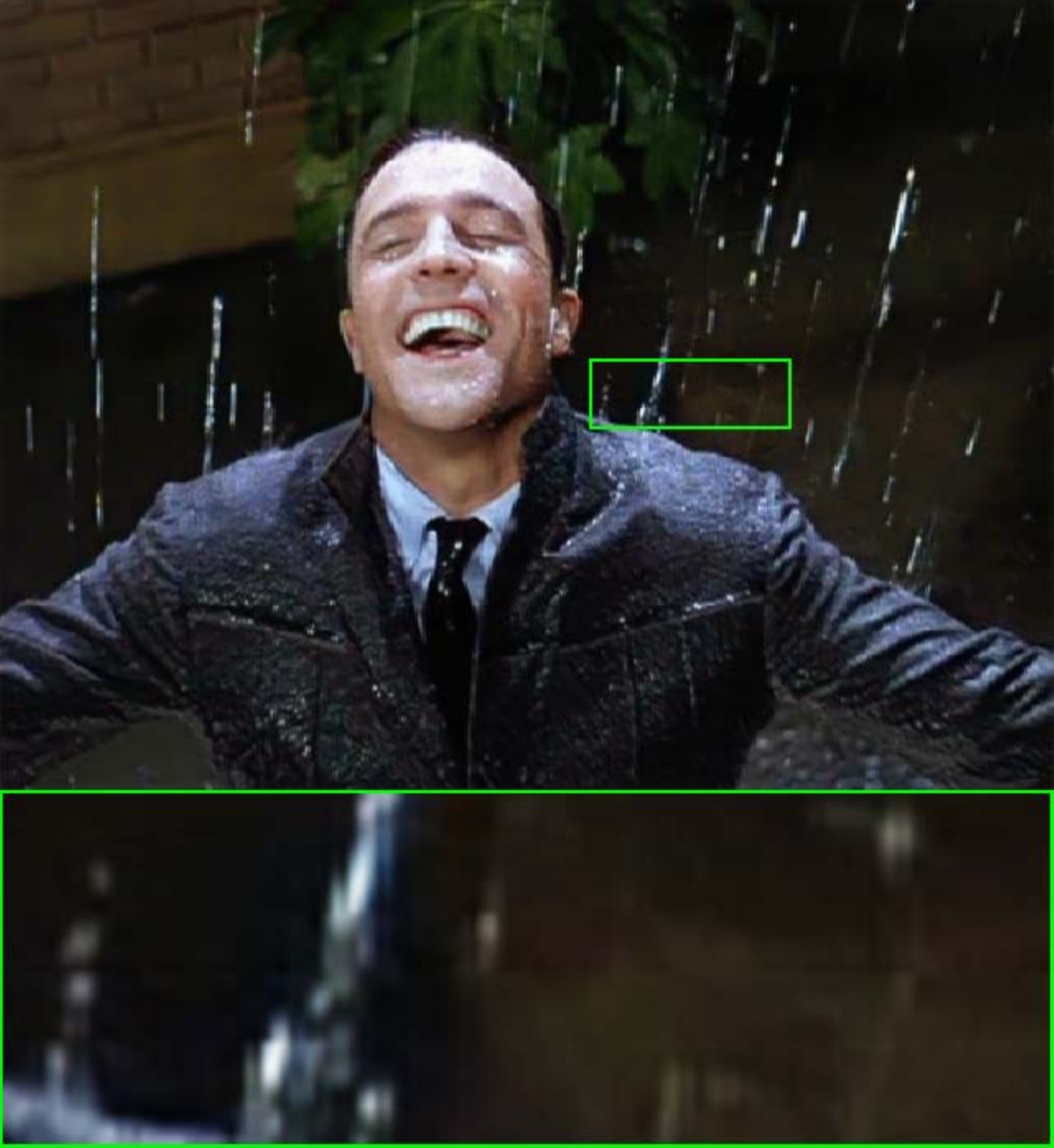}&
		\includegraphics[width=0.16\textwidth]{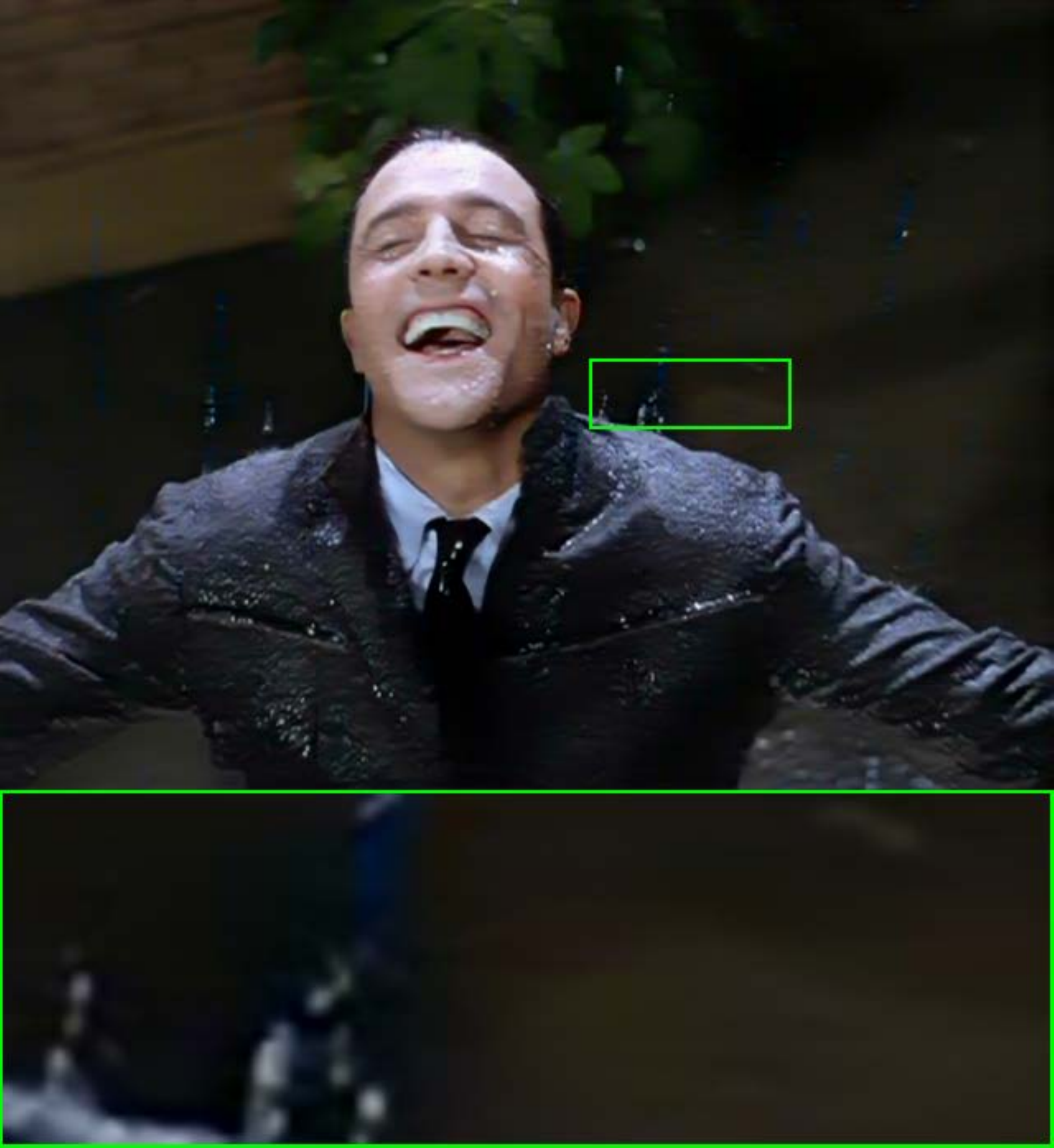}\\
		\footnotesize GMM & \footnotesize DDN & \footnotesize UGSM &\footnotesize JORDER  & \footnotesize DID-MDN &\footnotesize Ours \\
	\end{tabular}
	\caption{Rain streaks removal comparison results of synthesized image (top row) from Test1 and real-world rainy image (middle and bottom row).}
	\label{fig:derain_real}
\end{figure*}

As the proposed TLF is a proximal-based scheme, it is necessary to compare our method with the existing proximal-based first-order approaches (such as the classical PG, mAPG~\cite{li2015accelerated}, niAPG~\cite{yao2016efficient} and momentum APG for nonconvex problem (APGnc)~\cite{li2017convergence}) with additional 1\textperthousand~noise level and $27\times 27$ kernel size. The comparative results are shown in Fig.~\ref{fig:CurvesCompPGs} with relative error after log transformation ($\log(\|\mathbf{x}^{k+1}-\mathbf{x}^k\|/\|\mathbf{x}^{k+1}\|)$), reconstruction error ($\|\mathbf{x}^{k+1}-\mathbf{x}_{gt}\|/\|\mathbf{x}_{gt}\|$), functional value $F(\mathbf{x})$ and PSNR, where $\mathbf{x}_{gt}$ denotes ground truth. Here, we set stop criterion as $5e-4$. Observed that, our TLF converges faster than other PGs under the same stop condition. Moreover, DTLF performs the best both in PSNR scores and iteration steps. The corresponding visual results are shown in Fig.~\ref{fig:VisualDeblurPGs} with PSNR and SSIM (i.e., structural similarity) scores. Observed that the proposed DTLF remove more noise while keeping details.

\subsection{State-of-the-art Comparisons}\label{sec:app-LBPP}

We then evaluated our DTLF on a variety of low-level vision applications including image deblurring, image inpainting and rain streaks removal. 

\subsubsection{Image Deblurring} 
In this task, matrix $\mathbf{K}$ stated in the application part is blur kernel and $\mathbf{b}$ is a blurry image. As usual, blurry images are synthesized by applying a blur kernel and adding additive Gaussian noise. We considered the circular boundary conditions when performing the convolution. We reported the results of our DTLF on Sun \emph{et al}' challenging benchmark~\cite{Sun2013Edge} and Levin \emph{et al}' dataset~\cite{levin2009understanding}, together with other state-of-the-art methods including traditional methods (e.g., IDDBM3D~\cite{danielyan2012bm3d}, TV~\cite{wang2008new}, parameters learning-based methods (e.g.,  EPLL~\cite{zoran2011learning}, CSF~\cite{schmidt2014shrinkage}) and network-based methods (e.g., MLP~\cite{schuler2013machine}, IRCNN~\cite{zhang2017learning}, FDN~\cite{kruse2017learning}). It can be seen that in Tab.~\ref{tab:deblur}, our method obtained the best quantitative performance (i.e., PSNR and SSIM metrics) on Sun \emph{et al}' and Levin \emph{et al}' dataset. Also, it is faster than most of the compared methods. Moreover, we illustrated the visual comparisons on real image deblurring~\cite{kohler2012recording} with unknown blur kernel which is estimated roughly by Pan \emph{et al.}' method~\cite{pan2016robust}. As shown in Fig.~\ref{fig:realdeblurres}, our method reserve more details.

\subsubsection{Image Inpainting} 
In image inpainting task, matrix $\mathbf{K}$ and $\mathbf{b}$ denote mask and the missing pixels image, respectively. This task aims to recover missing pixels of observation. We compared our DTLF with TV~\cite{osher2005iterative}, FOE~\cite{Roth2009Fields}, VNL~\cite{arias2011variational},  ISDSB~\cite{He2014Iterative}, WNNM~\cite{gu2017weighted} and IRCNN~\cite{zhang2017learning} on this task. We normalized the pixel values to $[0,1]$ and then generated random masks of different levels including $40\%$, $60\%$, $80\%$ missing pixels on CBSD68 dataset~\cite{zhang2017beyond}. Moreover, we collected 12 different text masks to further evaluate the proposed methods. Tab.~\ref{tab:inpainting_comp} presents the PSNR and SSIM comparison results with different masks. Observed that our method perform better than the state-of-the-art approaches regardless the proportion of masks. Furthermore, in comparison with the visual performance of DTLF with other methods, we presented the $80\%$ missing pixels comparisons in Fig.~\ref{fig:inpainting} with top five scores (TV, FoE, ISDSB, WNNM and IRCNN). It can be observed that our approach successfully recovered the image with better visual effect, especially in the zoomed-in regions with rich details.

\subsubsection{Single-image Rain Streaks Removal} 
In this part, we evaluated our method on rain streaks removal task, in comparison with the state-of-the-art including GMM~\cite{li2016rain}, DN~\cite{fu2017clearing}, DDN~\cite{fu2017removing}, JCAS~\cite{gu2017joint}, JORDER~\cite{yang2017deep}, UGSM~\cite{Jiang2018FastDeRain}, and DID-MDN~\cite{zhang2018density}. For measuring the performance quantitatively, we employ PSNR and SSIM as the metrics. 

We first illustrated how to obtain temporary variable $\mathcal{N}_b(\mathbf{x}_b^k;\bm{\mathcal{W}}_{T,b}^k)$ and $\mathcal{N}_r(\mathbf{x}_r^k;\bm{\mathcal{W}}_{T,r}^k)$. As for $\mathcal{N}_b(\mathbf{x}_b^k;\bm{\mathcal{W}}_{T,b}^k)$, we adopt the same CNNs architecture and training strategy as in Section V-A for different noise level. As for $\mathcal{N}_r(\mathbf{x}_r^k;\bm{\mathcal{W}}_{T,r}^k)$, we train a series of $\mathcal{N}_r$ to adapt different rain streaks in each iteration. In other words, the architecture $\mathcal{N}_r(\mathbf{x}_r^k;\bm{\mathcal{W}}_{T,r}^k)$ is to estimate iteration behavior (i.e., the remaining rain streaks will be decreased with the iteration increase) of variable $\mathbf{x}_r$. The training dataset consists of 900 clean images from~\cite{yang2017deep}. We randomly selected $100 \times 100$ clean/rainy patch pairs from the synthesized rainy data as training samples. In fact, the learning rate, loss function and update strategy are setting the same with above.

We then compared the developed approach with state-of-the-art rain streaks removal methods. Tab.~\ref{tab:derain_set} reported the quantitative scores (i.e., PSNR and SSIM) on three different datasets: (1) Test1 is obtained by~\cite{li2016rain}, including 12 synthesized rain images with only one type of rain streaks rendering technique; (2) Rain100H is collected from BSD200~\cite{martin2001database} and synthesized with five streak directions; (3) Test2 consists of 7 images, using photorealistic rendering of rain streaks~\cite{tariq2007rain}. For $\{\mathcal{N}_r(\cdot,\bm{\mathcal{W}}_{T,r}^k)\}_{k=1,\cdots,N}$, two different trained results (i.e., trained by heavy and light rain streaks) are adopted during iteration. According to the quantitative results reported in Tab.~\ref{tab:derain_set}, we provided visual comparisons for five methods with relative high PSNR and SSIM scores (i.e., GMM, DDN, UGSM, JORDER, DID-MDN) in Fig.~\ref{fig:derain_real}. It can be observed that the proposed DTLF scheme can reserves more details with very few rain streaks left no matter in synthesized or real-world rainy images.

\begin{table}[t]
	\renewcommand\arraystretch{1.15}
	\centering
	\caption{Averaged PSNR and SSIM results among different rain streaks removal methods on three different rain streaks synthesized form: Test1 (Rain12)~\protect\cite{li2016rain}, Test2 (Rain7)~\protect\cite{tariq2007rain} and Rain100H~\protect\cite{martin2001database}.}
	\setlength{\tabcolsep}{1.8mm}{
		\begin{tabular}{c|cc|cc|cc}
			\hline
			\multirow{2}*{Methods} &\multicolumn{2}{c|}{Test1} &\multicolumn{2}{c|}{Test2} &\multicolumn{2}{c}{Rain100H}\\
			\cline{2-7}
			&PSNR &SSIM & PSNR & SSIM & PSNR & SSIM\\
			\hline
			JCAS  & 31.61 & 0.9183 & 28.37 & 0.9050 & 15.23 & 0.5150 \\	
			GMM & 32.33 & 0.9042 & 29.57 & 0.8878 & 14.26 & 0.4225\\
			DN    & 30.30 & 0.9151 & 27.34 & 0.9009 & 13.72 & 0.4417\\	
			DDN   & 33.41 & 0.9442 & 29.91 & 0.9433 & 17.93 & 0.5655\\
			UGSM  & 33.30 & 0.9253 & 27.07 & 0.9220 & 14.90 & 0.4674\\
			JORDER &35.93 & 0.9530 & \textbf{35.11} & 0.9732 & 23.45 & 0.7490\\
			DID-MDN & 29.08 & 0.9015 & 27.92 & 0.8695 & 17.28 & 0.6035\\
			Ours   & \textbf{36.55} & \textbf{0.9652} & 34.88 & \textbf{0.9737} & \textbf{24.51} & \textbf{0.8053} \\
			\hline	
	\end{tabular}}
	\label{tab:derain_set}
\end{table}

\begin{table}[t]
	\renewcommand\arraystretch{1.15}
	\centering
	\caption{Averaged PSNR and SSIM on Fu~\protect\emph{et al.}'~\protect\cite{fu2017removing} test set.}
	\setlength{\tabcolsep}{1.8mm}{
		\begin{tabular}{c|cc||c|cc}
			\hline
			Methods &PSNR & SSIM & Methods &PSNR & SSIM\\
			\hline
			DN  & 25.51 & 0.8885 & DID-MDN & 27.94 & 0.8696 \\	
			UGSM    & 26.38 & 0.8261 &  JORDER & 27.50  & 0.8515 \\
			DDN & 29.90 & 0.8999 & Ours & \textbf{31.18} & \textbf{0.9152}\\
			\hline
	\end{tabular}}
	\label{tab:derain_ddn}
\end{table}

\begin{figure}[htb!]
	\begin{tabular}{c@{\extracolsep{0.2em}}c}
		\includegraphics[width=0.232\textwidth]{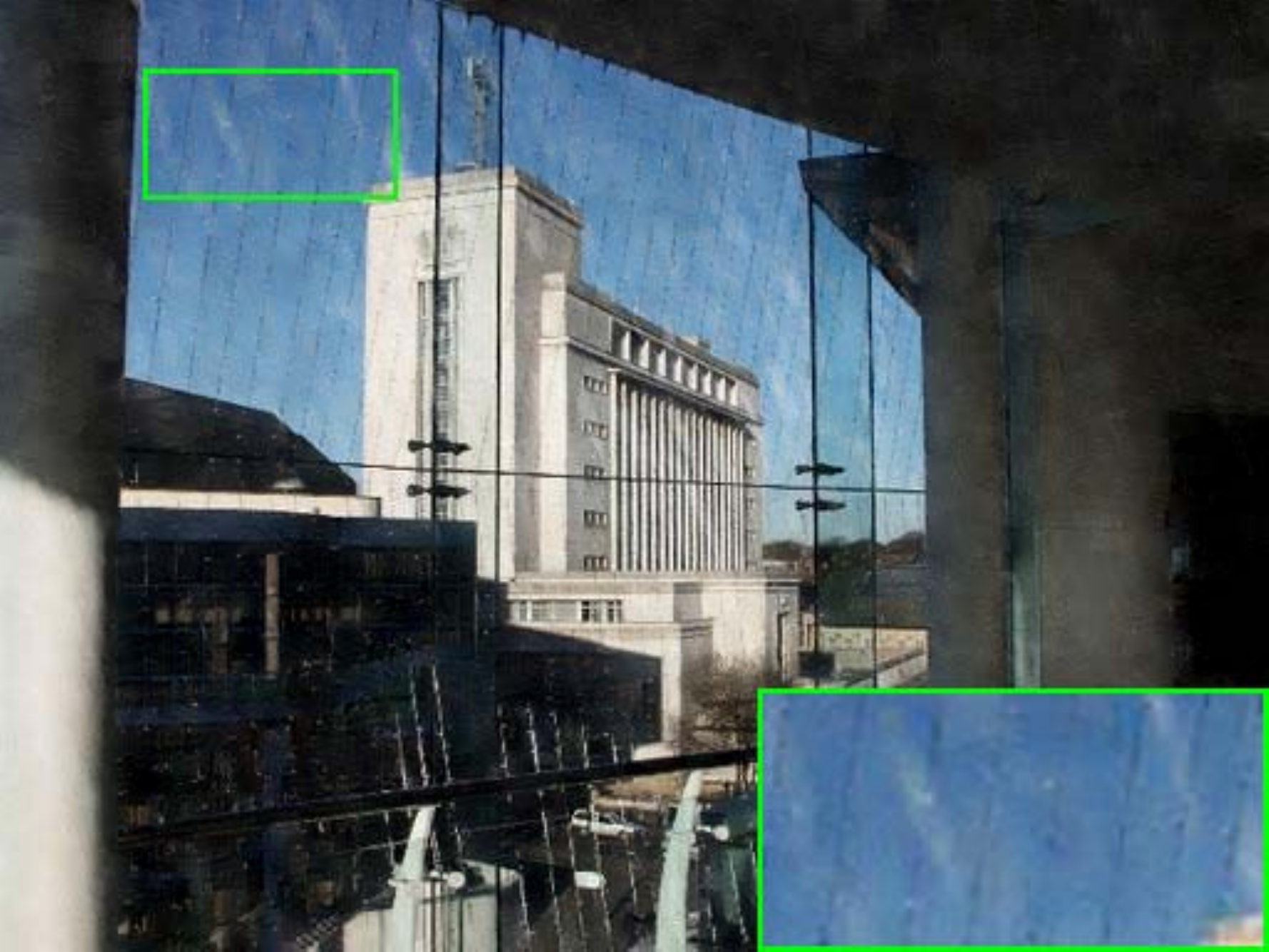}&
		\includegraphics[width=0.232\textwidth]{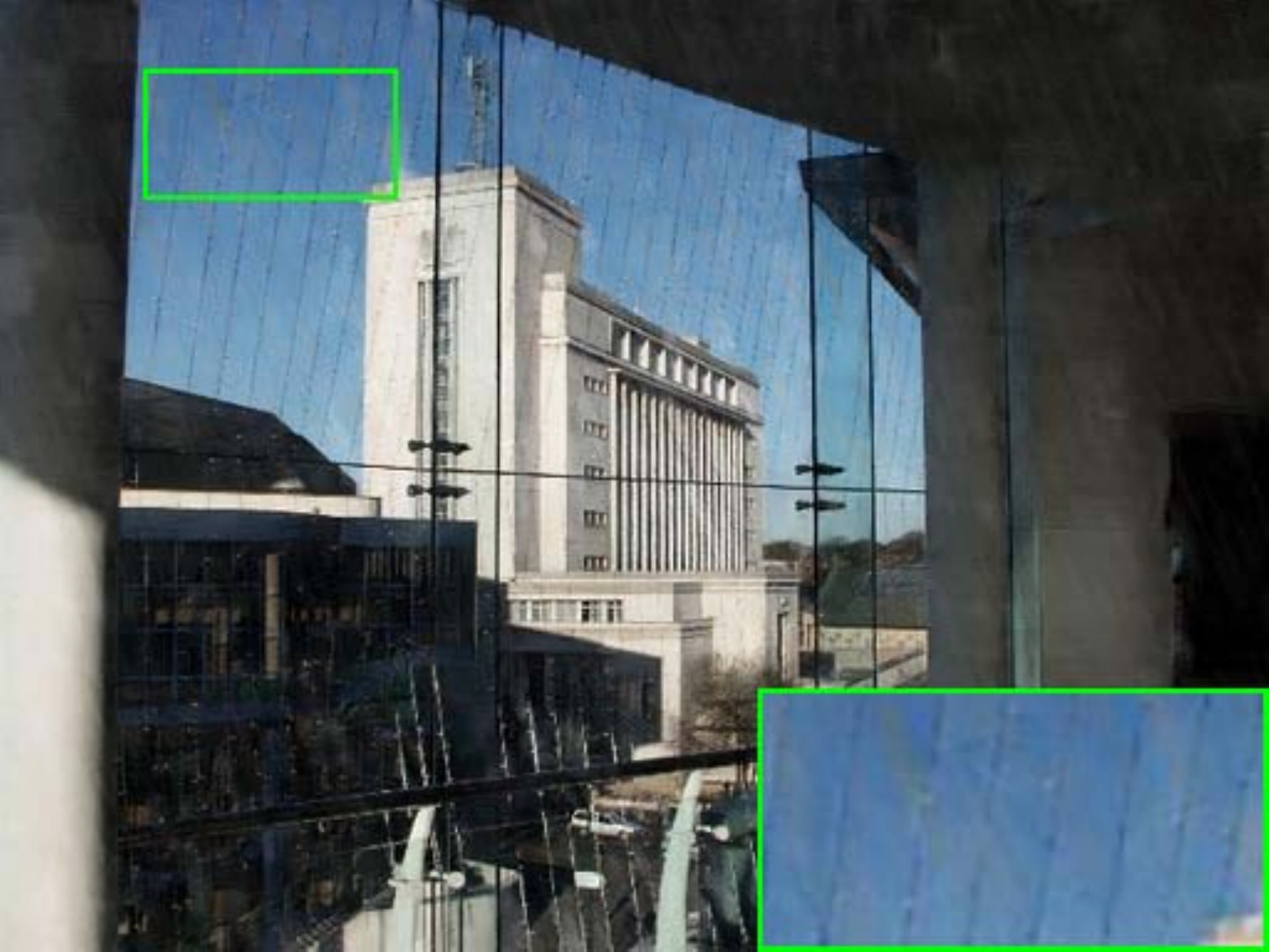}\\
		\footnotesize DDN (27.79 / 0.8371) &\footnotesize Ours (28.48 / 0.8531)
	\end{tabular}
	\caption{Image rain streaks removal results (PSNR / SSIM scores) on Fu \emph{et al.}' test set.}
	\label{fig:derainddn}
\end{figure}

Furthermore, we conducted experiments on a large scale dataset with 1,400 test images (collected by~\cite{fu2017removing}). The quantitative and qualitative results are demonstrated in Tab.~\ref{tab:derain_ddn} and Fig.~\ref{fig:derainddn}, respectively. Obviously, our method performed much better than the compared ones.

\section{Conclusions}
In this paper, we proposed a new perspective to understand and formulate constraints for MAP-based image models. By introducing an energy model as the constraint for MAP-type objective, we first established Task-driven Latent Feasibility (TLF), a simple bilevel scheme to solve the nonconvex image model formulated in Eq.~\eqref{eq:model_orginal}. We also designed DTLF to incorporate data-driven feasibility to further improve the performance of image modeling. Thanks to our specifically designed iteration control mechanisms, the convergence of TLF and DTLF can be strictly proved in theory. Extensive experiments on challenging image processing tasks also demonstrated the superiority of our method against other state-of-the-art approaches.

\ifCLASSOPTIONcaptionsoff
  \newpage
\fi




\bibliographystyle{IEEEtran}
\bibliography{reference1}

\begin{IEEEbiography}[{\includegraphics[width=1in,height=1.25in,clip,keepaspectratio]{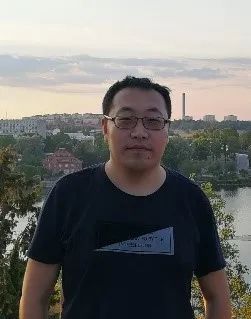}}]{Risheng Liu} received the B.S. and Ph.D. degrees both in mathematics from the Dalian University of Technology in 2007 and 2012, respectively. He was a visiting scholar in the Robotic Institute of Carnegie Mellon University from 2010 to 2012. He served as Hong Kong Scholar Research Fellow at the Hong Kong Polytechnic University from 2016 to 2017. He is currently a professor with DUT-RU International School of Information Science \& Engineering, Dalian University of Technology. He was awarded the ``Outstanding Youth Science Foundation'' of the National Natural Science Foundation of China. His research interests include machine learning, optimization, computer vision and multimedia. He was a co-recipient of the IEEE ICME Best Student Paper Award in both 2014 and 2015. His two papers were also selected as Finalist of the Best Paper Award in ICME 2017. He is a member of the IEEE and ACM. 
\end{IEEEbiography}
\vspace{-1.2cm}
\begin{IEEEbiography}[{\includegraphics[width=1in,height=1.25in,clip,keepaspectratio]{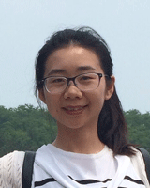}}]{Pan Mu} received the B.S. degree in Applied Mathematics from Henan University, China, in 2014, the M.S. degree in Operational Research and Cybernetics from Dalian University of Technology, China, in 2017. She is currently pursuing the PhD degree in Computational Mathematics at Dalian University of Technology, Dalian, China. She is with the Key Laboratory for Ubiquitous Network and Service Software of Liaoning Province, Dalian University of Technology, Dalian, China. Her research interests include computer vision, machine learning, optimization and control.
\end{IEEEbiography}
\vspace{-1.2cm}
\begin{IEEEbiography}[{\includegraphics[width=1in,height=1.25in,clip,keepaspectratio]{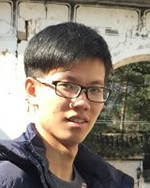}}]{Jian Chen} received the B.E. degree in Electronic Commerce from Dalian University of Technology, Dalian, China, in 2018. He is currently pursuing the master degree in software engineering at Dalian University of Technology, Dalian, China. He is with the Key Laboratory for Ubiquitous Network and Service Software of Liaoning Province, Dalian University of Technology, Dalian, China. His research interests include rain streaks removal, computer vision, deep learning.
\end{IEEEbiography}
\vspace{-1.2cm}
\begin{IEEEbiography}[{\includegraphics[width=1in,height=1.25in,clip,keepaspectratio]{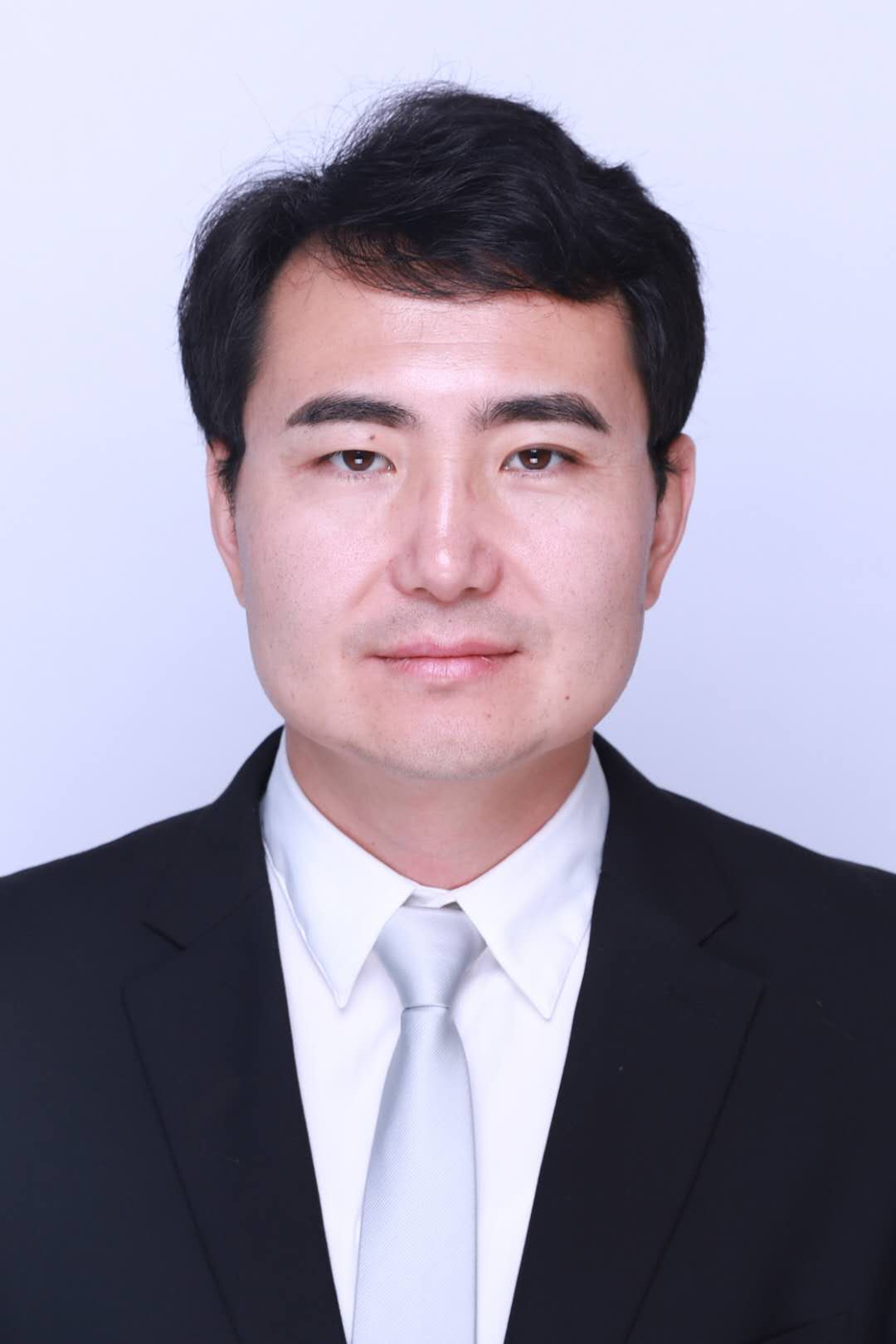}}]{Xin Fan} received the B.E. and Ph.D. degrees at Xi'an Jiaotong University, Xi'an, China, in 1998 and 2004, respectively. He was with Oklahoma State University, Stillwater, and the University of Texas Southwestern Medical Center, Dallas, from 2006 to 2009, as a post-doctoral research fellow. He joined Dalian University of Technology, Dalian, China, in 2009, where he is currently a full professor. He won the 2015 IEEE ICME Best Student Award as the corresponding author, and two papers were selected as the Finalist of the Best Paper Award at ICME 2017. His current research interests include image processing and machine vision.
\end{IEEEbiography}
\vspace{-1.2cm}
\begin{IEEEbiography}[{\includegraphics[width=1in,height=1.25in,clip,keepaspectratio]{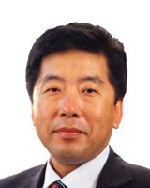}}]{Zhongxuan Luo} received the B.S. degree in Computational Mathematics from Jilin University, China, in 1985, the M.S. degree in Computational Mathematics from Jilin University in 1988, and the PhD degree in Computational Mathematics from Dalian University of Technology, China, in 1991. He has	been a full professor of the School of Mathematical Sciences at Dalian University of Technology since 1997. He is also with the Peng Cheng Laboratory, Shenzhen, China. His research interests include computational geometry and computer vision. 
\end{IEEEbiography}

\end{document}